\documentclass[10pt,letterpaper,twocolumn]{article}
\usepackage{times}
\usepackage{graphicx,booktabs}
\usepackage[font=small]{caption}
\usepackage[numbers,square,sort&compress]{natbib}

\usepackage{url}
\usepackage{amsmath,amssymb,amsthm,mathtools}
\usepackage{bm}
\usepackage{microtype}
\usepackage{array}
\usepackage{algorithm}
\usepackage{algpseudocode}
\usepackage{balance}
\usepackage{enumerate}
\makeatletter

\renewcommand{\normalsize}{%
    \@setfontsize\normalsize{10}{11}%
    \abovedisplayskip=7pt plus 2pt minus 5pt
    \belowdisplayskip=\abovedisplayskip
    \abovedisplayshortskip=0pt plus 3pt
    \belowdisplayshortskip=4pt plus 3pt minus 3pt
    \let\@listi\@listI
}

\renewcommand{\small}{\@setfontsize\small{9}{10}}
\renewcommand{\footnotesize}{\@setfontsize\footnotesize{8}{9.5}}
\renewcommand{\scriptsize}{\@setfontsize\scriptsize{7}{8}}
\renewcommand{\tiny}{\@setfontsize\tiny{6}{7}}
\renewcommand{\large}{\@setfontsize\large{12}{14}}
\renewcommand{\Large}{\@setfontsize\Large{14}{16}}
\renewcommand{\LARGE}{\@setfontsize\LARGE{16}{20}}
\renewcommand{\huge}{\@setfontsize\huge{20}{23}}
\renewcommand{\Huge}{\@setfontsize\Huge{25}{28}}

\renewcommand{\@seccntformat}[1]{%
    \csname the#1\endcsname.\nobreakspace
}

\renewcommand{\section}{%
    \@startsection{section}{1}{\z@}
    {-0.12in}{0.02in}
    {\normalfont\large\bfseries\raggedright}%
}

\renewcommand{\subsection}{%
    \@startsection{subsection}{2}{\z@}
    {-0.10in}{0.01in}
    {\normalfont\normalsize\bfseries\raggedright}%
}

\renewcommand{\subsubsection}{%
    \@startsection{subsubsection}{3}{\z@}
    {-0.08in}{0.01in}
    {\normalfont\normalsize\scshape\raggedright}%
}

\renewcommand{\paragraph}{%
    \@startsection{paragraph}{4}{\z@}
    {1.5ex plus 0.5ex minus 0.2ex}{-1em}
    {\normalfont\normalsize\bfseries}%
}

\renewcommand{\subparagraph}{%
    \@startsection{subparagraph}{5}{\z@}
    {1.5ex plus 0.5ex minus 0.2ex}{-1em}
    {\normalfont\normalsize\bfseries}%
}

\def\@listI{%
    \leftmargin\leftmargini
    \topsep 4pt plus 1pt minus 2pt
    \partopsep 1pt plus 0.5pt minus 0.5pt
    \parsep 2pt plus 1pt minus 0.5pt
    \itemsep\parsep
}

\def\@listii{%
    \leftmargin\leftmarginii
    \labelwidth\leftmarginii
    \advance\labelwidth-\labelsep
    \topsep 2pt plus 1pt minus 0.5pt
    \parsep 1pt plus 0.5pt minus 0.5pt
    \itemsep\parsep
}

\def\@listiii{%
    \leftmargin\leftmarginiii
    \labelwidth\leftmarginiii
    \advance\labelwidth-\labelsep
    \topsep 1pt plus 0.5pt minus 0.5pt
    \parsep\z@
    \partopsep 0.5pt plus 0pt minus 0.5pt
    \itemsep\topsep
}

\normalsize
\@listi

\renewcommand{\@maketitle}{%
    \begingroup
    \centering
    \setlength{\parskip}{0pt}

    \hrule height 1pt
    \vskip 0.25in

    {\fontsize{14.4}{18}\selectfont\bfseries\@title\par}

    \vskip 0.22in
    \hrule height 1pt
    \vskip 0.3in

    {\normalsize\bfseries
    \@author\par}

    \vskip 0.3in
    \endgroup
}

\makeatother

\renewcommand{\footnoterule}{%
    \kern-3pt
    \hrule width 0.8in
    \kern2.6pt
}

\DeclareCaptionFont{originalsmall}{\fontsize{9}{11}\selectfont}

\newtheorem{theorem}{Theorem}[section]
\newtheorem{proposition}[theorem]{Proposition}

\newtheorem{corollary}[theorem]{Corollary}

\renewenvironment{abstract}
{\fontsize{9}{10}\selectfont\bfseries
\leavevmode\kern-1pt
\textit{Abstract}---\ignorespaces}
{\par\normalfont\normalsize\vskip 0.12in}

\usepackage[
colorlinks=false,
citebordercolor={0 1 1},
linkbordercolor={0 1 0}
]{hyperref}

\newcommand{\plainhref}[2]{%
\begingroup
\hypersetup{pdfborder={0 0 0}}%
\href{#1}{#2}%
\endgroup}

\AtBeginDocument{\let\citet\citep}

\let\originalappendix\appendix

\renewcommand{\appendix}{%
    \onecolumn
    \originalappendix
}
\title{
Multi-Agent Flow Matching with Decoupled Generative Guidance
}
\author{
Ruoyu Lin
\quad\,
Magnus Egerstedt
\quad\,
Fabio Pasqualetti
}
\date{}
\begin{document}
\pagestyle{plain}
\maketitle
\thispagestyle{plain}
\begingroup
\renewcommand{\thefootnote}{}
\footnotetext[1]{
Ruoyu Lin and Fabio Pasqualetti are with the Department of Electrical Engineering and Computer Science, University of California, Irvine, Irvine, CA 92697 USA.
Email: \mbox{\ttfamily\{\plainhref{mailto:rlin10@uci.edu}{rlin10},%
\plainhref{mailto:fabiopas@uci.edu}{fabiopas}\}@uci.edu}
}
\footnotetext[2]{
Magnus Egerstedt is with the University of North Carolina at Chapel Hill, Chapel Hill, NC 27599 USA.
Email: \mbox{\ttfamily
\plainhref{mailto:magnus@unc.edu}{magnus@unc.edu}}
}
\endgroup
\begin{abstract}
Generative modeling is widely used to produce diverse objects from complex, multimodal distributions. However, its expressivity does not inherently guarantee that generated objects satisfy hard requirements. Guiding the generative process to satisfy such requirements is particularly challenging in multi-agent systems because the requirements may couple multiple agents but each agent may need to determine its own guidance input without relying on the simultaneously computed guidance input of any other agent. To this end, we introduce \mbox{DeGG-Flow}, a general framework for multi-agent flow matching with decoupled generative guidance. Representing the generative process as a control-affine system, we develop Lyapunov-based guidance for general multi-agent requirements by overcoming the key obstacles to decentralization, i.e., shared requirements whose satisfaction depends on a group of agents collectively, and private requirements associated with each individual agent but dependent on its neighbors. We establish feasibility conditions and finite-horizon convergence guarantees for \mbox{DeGG-Flow}, and derive a Wasserstein bound on the distributional deviation it induces. We demonstrate \mbox{DeGG-Flow} across three applications, namely multi-robot collaborative policy generation for spatial gap crossing, multi-object scene generation with affordance requirements, and multi-agent panoramic image generation. \mbox{DeGG-Flow} directly generates objects that satisfy all corresponding hard requirements, without retraining the generative models.
\end{abstract}

\section{Introduction} \label{sec:introduction}
Generative models have achieved remarkable success across a wide range of domains, from language and vision to robotics and scientific discovery  \citep{brown2020language,rombach2022high,chi2025diffusion,lipman2022flow,albergo2025stochastic,janner2022planning,tang2024diffuscene,watson2023novo,zeni2025generative}. A core strength of these models is their ability to characterize complex, multimodal distributions and generate diverse objects from them. However, such expressivity, in general, does not imply that every generated object satisfies desired hard constraints or requirements. For example, a generated motion may violate safety requirements, a generated molecular structure may violate physical or chemical constraints, and new requirements introduced after training may not be represented in the training data. This gap is especially important when a learned generative model is already useful but new hard requirements arise after training. In such cases, it is desirable to impose constraints at generation time without collecting new data or retraining the model.

Guidance in the generative process provides a natural mechanism for enforcing constraints \citep{bansal2024universal,huang2024constrained,feng2025guidance}. Recent work has made this idea increasingly principled by incorporating control Lyapunov function (CLF) \citep{lyapunov1992general,sontag1983lyapunov,artstein1983stabilization} and control barrier function (CBF) \citep{prajna2006barrier,prajna2007framework,wieland2007constructive,ames2016control} into the generative process. SafeDiffuser \citep{xiao2025safediffuser} and \citet{botteghi2023trajectory} introduce CBFs into the denoising process of diffusion models \citep{ho2020denoising}; CoBL-Diffusion \citep{mizuta2024cobl} incorporates CLF and CBF guidance into the denoising process; \citet{gadginmath2026provably} incorporates CBFs into the score-based stochastic process \citep{hyvarinen2005estimation,song2020score}; SafeFlow \citep{dai2025safeflow} introduces CBFs into the generative process of flow matching models \citep{lipman2022flow}; and SAD-Flower \citep{huang2025sad} incorporates CLF and CBF guidance into the flow-based generative process.

Multi-agent generative modeling changes the structure of such guidance, as hard constraints or requirements may depend jointly on multiple agents. A direct approach is to treat all agents as a single-agent system and compute the guidance together. However, a key advantage of multi-agent systems is that the team need not depend on a single central decision maker, whose failure could affect the entire system. This motivates decentralized guidance, where each agent determines its own guidance without relying on any other agent's guidance. Existing multi-agent generative models address different aspects of this problem. MADiff \citep{zhu2024madiff} models multimodal multi-agent behavior and supports decentralized execution; \citet{shaoul2025multi} combines diffusion models with search-based planning to generate collision-free multi-robot trajectories; \citet{parimi2025diffusion} decomposes multi-arm planning into single-arm trajectory generation and pairwise collision resolution; \citet{liang2025simultaneous} integrates constrained projection into the denoising process to satisfy collision avoidance and kinematic constraints; and MAC-Flow \citep{lee2026multi} learns joint behaviors with flow matching and distills them into decentralized one-step policies. These works provide different mechanisms for generating, coordinating, and adapting multi-agent behavior. What remains unresolved is how to guarantee a hard requirement that depends on multiple agents when each agent determines only its own guidance.

\begin{figure*}[t]
\centering
\includegraphics[width=0.9\linewidth]{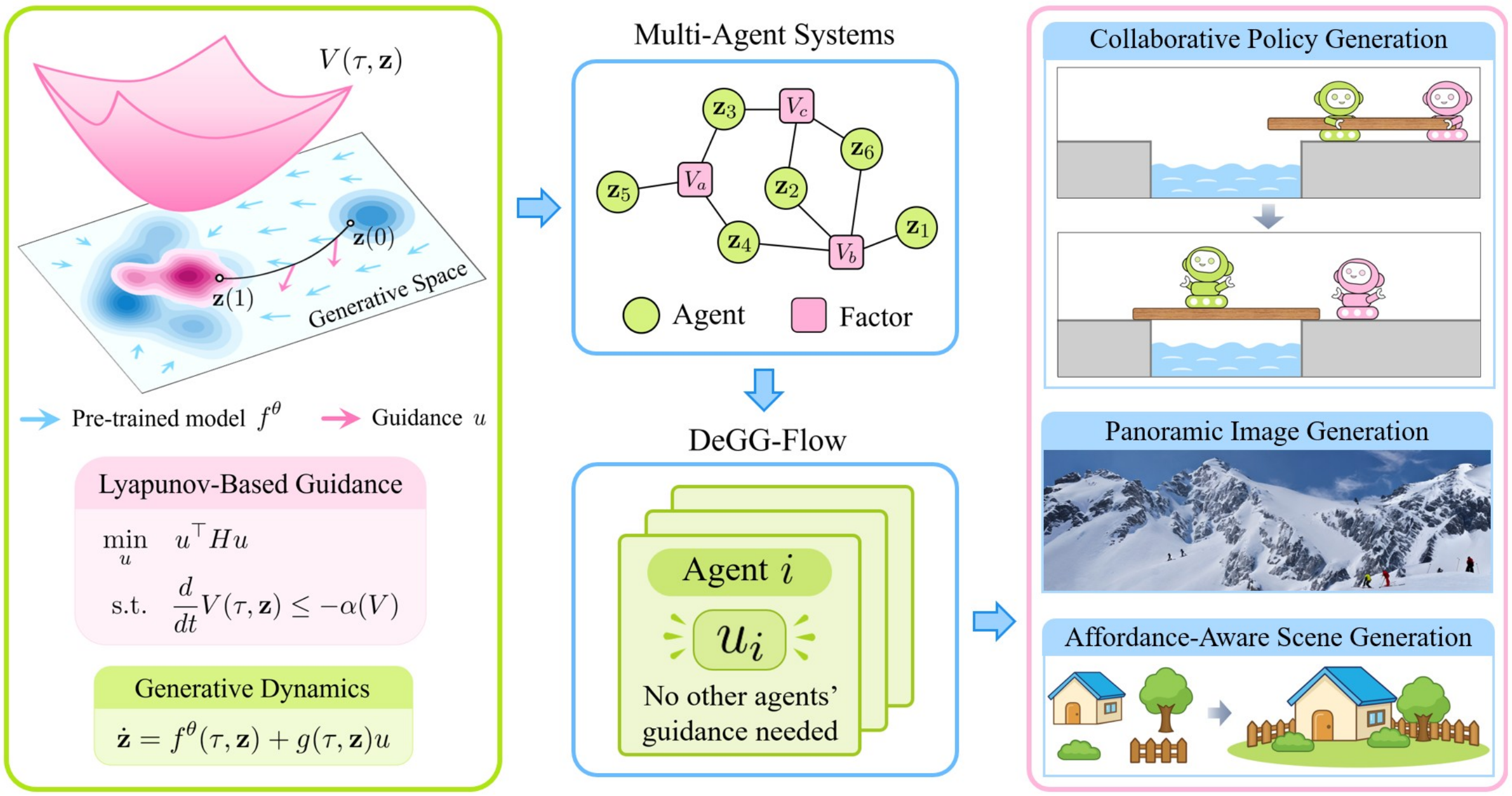}
\caption{Illustration of DeGG-Flow and its applications.}
\label{fig:illustration}
\end{figure*}

The generative guidance methods discussed earlier do not directly resolve this problem, as their application to the joint state of all agents naturally leads to a centralized guidance problem. To this end, we introduce \textbf{DeGG-Flow}, a general, principled framework for multi-agent flow matching with decoupled generative guidance, as illustrated in Figure~\ref{fig:illustration}. The main contributions of this paper are as follows.

\begin{enumerate}[1)]
\item 
We formulate the generative process of multi-agent flow matching as a control-affine system and extend disentangled control \citep{lin2026disentangled} to generative guidance. DeGG-Flow enables each agent to compute its own guidance input without relying on the simultaneously computed guidance input of any other agent, regardless of how the requirements couple the agents, and is theoretically guaranteed to produce objects that satisfy hard requirements without retraining the underlying generative models.
\item 
To make DeGG-Flow broadly applicable, we resolve the key obstacles to decentralized guidance, namely shared requirements whose satisfaction depends on an entire team of agents, and private requirements associated with each individual agent and dependent on its neighbors. For shared requirements, we introduce an adaptive constraint allocation approach to ensure feasibility of the guidance, and propose a time-varying bound that guarantees the generative state reaches a desired set within a finite horizon. For private requirements, we also establish corresponding conditions for the feasibility of the guidance and the finite-horizon convergence of the generative state.
\item 
We derive a Wasserstein bound that characterizes how far the guided joint generative distribution can deviate from the unguided one, and empirically evaluate the resulting distributional deviation.
\item 
We demonstrate DeGG-Flow on multi-robot collaborative policy generation with shared requirements for crossing a spatial gap by reconfiguring the environment, on multi-agent panoramic image generation with shared requirements, and on multi-object scene generation with private requirements in which each object's affordance requirements, i.e., its access and use requirements, depend on neighboring objects. In all tasks, DeGG-Flow directly generates objects that satisfy all hard requirements without retraining the generative models.
\end{enumerate}

\section{Multi-Agent Flow Matching Guided by Disentangled Control} \label{sec:formulation}
\subsection{Lyapunov-Based Generative Guidance}
\label{sec:lyapunov_guidance}

Flow matching \citep{lipman2022flow} learns a time-varying vector field $f:[0,1]\times\mathbb R^d\rightarrow\mathbb R^d$ that transports samples from a simple distribution toward generated objects from a complex, multimodal data distribution. Since $f$ is learned to reproduce the training data distribution, the generated objects may violate hard requirements, especially those introduced after training. To enforce such requirements, one can introduce a guidance input $u \in \mathbb{R}^m$ into the generative dynamics in the control-affine form,
\begin{equation}
\label{eqn:generative_dynamics_general}
\dot{x} = f(\tau,x) + g(\tau,x)u,
\end{equation}
to guide the generative process, where $x(\tau)\in\mathbb{R}^d$ is the generative state at generative time $\tau\in[0,1]$, and $g:[0,1]\times\mathbb{R}^d \rightarrow \mathbb{R}^{d\times m}$ determines how the input affects the generative dynamics. Note that $\tau$ is different from physical time $t \in \mathbb R_{\geq 0}$.

To guarantee that the generated objects satisfy desired hard requirements, we resort to the idea of CLF \citep{lyapunov1992general,sontag1983lyapunov,artstein1983stabilization}. Specifically, consider a $C^1$ (i.e., continuously differentiable) function $V:[0,1] \times \mathbb{R}^d \rightarrow \mathbb{R}$ characterizing a desired requirement. Along a trajectory of \eqref{eqn:generative_dynamics_general}, we impose the differential inequality constraint
\begin{equation*}
\frac{d}{d\tau} V(\tau,x(\tau))
= 
\partial_\tau V
+ \nabla_x V^\top(f+gu)
\leq 
-\alpha(V)
\end{equation*}
to regulate the behavior of $V$, where $\alpha: \mathbb{R} \rightarrow \mathbb{R}$ is an extended class-$\mathcal{K}_\infty$ function, and the arguments are omitted for notational simplicity. Following the idea of minimum-norm optimization-based control \citep{freeman1995optimal}, the smallest instantaneous guidance input can be computed through 
\begin{equation}
\begin{aligned}
&\underset{u\in\mathbb{R}^{m}}{\arg\min}
\quad u^{\top}H u\\
&\quad \mathrm{s.t.}\quad\;\;\;
\partial_\tau V + \nabla_x V^\top(f+gu) + \alpha(V) \leq 0
\end{aligned}
\label{eqn:QP_general_initial}
\end{equation}
where $H \succ 0$. In particular, if the learned flow matching model $f(\tau,x)$ satisfies the differential inequality in 
\eqref{eqn:QP_general_initial} with $u=0$, the guidance input is zero. Otherwise, \eqref{eqn:QP_general_initial} determines the minimum correction needed to enforce it.

This Lyapunov-based formulation provides a mechanism for guiding the generative process in a minimally invasive manner. In multi-agent systems, however, hard requirements may couple multiple agents, making decentralized guidance nontrivial, which we investigate in the following content.

\subsection{Multi-Agent Flow Matching}
Consider a team of $N\in\mathbb{Z}^+$ agents, indexed by $i\in\mathcal{N}\coloneqq\{1,\ldots,N\}$, whose joint generative process produces an output $y\in\mathcal{Y}_N$ for a given task, where $\mathcal{Y}_N$ denotes the corresponding output space. Let $\chi\in\mathcal{C}_N$ denote the generation condition, where $\mathcal{C}_N$ is the corresponding condition space. The condition contains information such as initial states, agent or object properties, environment information, or task specifications. The output $y$ depends on specific applications. For example, in Section~\ref{sec:task1}, $y$ represents a motion policy for $N$ robots. Instead of generating $y$ directly, each agent $i$ has a generative state $z_i(\tau)\in\mathcal{Z}\subseteq\mathbb{R}^{d_z}$. Denote the joint generative state as $\mathbf{z}(\tau)\coloneqq[z_1(\tau)^{\top},\ldots,z_N(\tau)^{\top}]^{\top}\in\mathcal{Z}^{N}\subseteq\mathbb{R}^{Nd_z}$. The task-specific decoder and encoder maps between the joint generative state space and the output space are denoted as
\begin{equation*}
\mathcal{D}_N:\mathcal{Z}^{N}\times\mathcal{C}_N\rightarrow\mathcal{Y}_N,
\quad
\mathcal{E}_N:\mathcal{Y}_N\times\mathcal{C}_N\rightarrow\mathcal{Z}^{N},
\end{equation*}
respectively. For a training data $y$, $\mathbf{z}=\mathcal{E}_N(y\,|\,\chi)$ is its target generative state for flow matching, and $\mathcal{D}_N(\mathbf{z}(1)\,|\,\chi)$ represents the final generated object.

Given $N$ and $\chi$, let $p_{\mathrm{data}}(\mathbf{z}\,|\,\chi)$ denote the joint conditional distribution of the encoded training data in $\mathcal{Z}^{N}$, and let $p_0(\mathbf{x})$ denote a simple initial distribution in $\mathcal{Z}^{N}$ from which samples can be drawn efficiently. Generation starts from $\mathbf{z}(0)\sim p_0$ and evolves over $\tau\in[0,1]$ governed by the dynamical system $\dot{\mathbf{z}} = F(\tau,\mathbf{z}\,|\,\chi)$. According to \citet{lipman2022flow}, to learn the vector field whose flow transports $p_0$ to $p_{\mathrm{data}}(\mathbf{z}\,|\,\chi)$, for each target $\mathbf{z}\in\mathcal{Z}^{N}$, consider a conditional probability path $p(\tau,\mathbf{x}\,|\,\mathbf{z},\chi)$ satisfying $p(0,\mathbf{x}\,|\,\mathbf{z},\chi)=p_0(\mathbf{x})$ and $p(1,\mathbf{x}\,|\,\mathbf{z},\chi)=\delta(\mathbf{x}-\mathbf{z})$, where $\mathbf{x}\in\mathcal{Z}^{N}$ is a realization of the intermediate random variable $\mathbf{X}(\tau)$ and $\delta(\mathbf{x}-\mathbf{z})$ is the Dirac delta concentrated at $\mathbf{z}$. Let $v(\tau,\mathbf{x}\,|\,\mathbf{z},\chi)\in\mathbb{R}^{Nd_z}$ denote a conditional time-varying vector field that generates this path, then the marginal probability path is $p(\tau,\mathbf{x}\,|\,\chi)=\int_{\mathcal{Z}^{N}}p(\tau,\mathbf{x}\,|\,\mathbf{z},\chi) p_{\mathrm{data}}(\mathbf{z}\,|\,\chi) \,\mathrm{d}\mathbf{z}$, and the corresponding marginal vector field is $F(\tau,\mathbf{x}\,|\,\chi)
=
\int_{\mathcal{Z}^{N}}\!
v(\tau,\mathbf{x}\,|\,\mathbf{z},\chi)
\,
(p(\tau,\mathbf{x}\,|\,\mathbf{z},\chi)
p_{\mathrm{data}}(\mathbf{z}\,|\,\chi))
/p(\tau,\mathbf{x}\,|\,\chi)
\,\mathrm{d}\mathbf{z}$.

Since $F(\tau,\mathbf{x}\,|\,\chi)$ depends on the unknown data distribution, we approximate it using a learned vector field
\begin{equation*}
f^{\theta}:
[0,1]\times\mathcal{Z}^{N}\times\mathcal{C}_N
\rightarrow
\mathbb{R}^{Nd_z},
\end{equation*}
which we refer to as the nominal vector field. (For example, $f^{\theta}$ can be represented as a message passing graph neural network (GNN)~\citep{gilmer2017neural}, with the same parameters $\theta$ shared across all agents, each of which computes its output from its own state and messages aggregated from neighboring agents so that $f^{\theta}$ is permutation equivariant.) Denote $f^{\theta}=[(f_1^{\theta})^{\top},\ldots,(f_N^{\theta})^{\top}]^{\top}$ and $v=[v_1^{\top},\ldots,v_N^{\top}]^{\top}$, where $f_i^{\theta}, v_i \in \mathbb{R}^{d_z}$ are the nominal and conditional vector fields of agent $i$, respectively. Then, $f^{\theta}$ can be trained using the multi-agent conditional flow matching loss
\begin{equation}
\underbrace{
\mathbb{E}
\left[
\frac{1}{N}\sum_{i=1}^{N}
\left\|
f_i^{\theta}(\tau,\mathbf{X}(\tau)\,|\,\chi)
-
v_i(\tau,\mathbf{X}(\tau)\,|\,\mathbf{Z},\chi)
\right\|^2
\right]
}_{\mathcal{L}_{\mathrm{MAC}}(\theta\,|\,\chi)},
\label{eqn:macfm_loss}
\end{equation}
in which the expectation is taken over $\tau\sim\mathrm{Unif}[0,1]$, the encoded data random variable $\mathbf{Z}\sim p_{\mathrm{data}}(\mathbf{z}\,|\,\chi)$, and the intermediate state random variable $\mathbf{X}(\tau)\sim p(\tau,\mathbf{x}\,|\,\mathbf{Z},\chi)$ along the conditional probability path.

\subsection{Disentangled Control Guidance}
\label{sec:disentangled_control_guidance}
A major advantage of multi-agent systems is that the team need not depend on a single central decision maker, whose failure could affect all agents at once. This motivates decentralized generative guidance, where each agent determines its own guidance without relying on the simultaneously computed guidance input of any other agent. As in \eqref{eqn:generative_dynamics_general}, we consider the control-affine generative dynamics
\begin{equation}
\dot{z}_i = f_i^{\theta}(\tau,\mathbf{z}\,|\,\chi) + g_i(\tau,\mathbf{z}\,|\,\chi)u_i,
\label{eqn:control_affine}
\end{equation}
$\forall i\in\mathcal{N}$, where $u_i \in \mathbb{R}^{m_i}$ and $g_i: [0,1]\times\mathcal{Z}^{N}\times\mathcal{C}_N\rightarrow \mathbb{R}^{d_z\times m_i}$. However, directly applying the Lyapunov-based guidance discussed in Section~\ref{sec:lyapunov_guidance} to coupled hard requirements generally involves the guidance inputs of multiple agents simultaneously, leading to a centralized guidance problem. The key obstacle arises from two classes of coupled requirements: shared requirements whose satisfaction depends on a group of agents as a whole, and private requirements associated with individual agents but dependent on their neighbors. To resolve this, we extend disentangled control \citep{lin2026disentangled} into the generative process as a form of decoupled guidance, i.e., guidance for shared-entangled requirements (SE guidance) and guidance for private-entangled requirements (PE guidance), to enable each agent~$i$ to compute its guidance input $u_i$ without access to the simultaneously computed guidance input $u_j$, $\forall j \in \mathcal{N}\setminus \! \{i\}$.

\paragraph{SE guidance.} Let $\mathcal{G}_{\mathrm{F}}^{\mathrm{SE}} = (\mathcal{N},\mathcal{A}^{\mathrm{SE}},\mathcal{I}^{\mathrm{SE}})$ be a factor graph~\citep{kschischang2001factor}, where $\mathcal{A}^{\mathrm{SE}}$ is a finite set of factors and $\mathcal{I}^{\mathrm{SE}} \subseteq \mathcal{N}\times\mathcal{A}^{\mathrm{SE}}$ is the incidence relation. For each factor $a\in\mathcal{A}^{\mathrm{SE}}$, define the incident agent set as $\mathcal{S}_{a}^{\mathrm{SE}} \coloneqq \{i\in\mathcal{N}:(i,a)\in\mathcal{I}^{\mathrm{SE}}\} \neq\varnothing$, and let $\mathbf{z}_{\mathcal{S}_{a}^{\mathrm{SE}}}$ denote the corresponding subvector of the joint generative state. Let $V_a^{\mathrm{SE}}:[0,1]\times\prod_{i\in\mathcal{S}_a^{\mathrm{SE}}}\mathcal{Z}\times\mathcal{C}_N\rightarrow \mathbb{R}$ be $C^1$. The guidance input $u_i$ of agent~$i \in\mathcal{S}_a^{\mathrm{SE}}$ is such that
\begin{equation}
\nabla_{z_i}V_a^{\mathrm{SE}}
(\tau,\mathbf{z}_{\mathcal{S}_a^{\mathrm{SE}}}\,|\,\chi)^{\top} \dot{z}_i
+ 
\zeta_{a,i}^{\mathrm{SE}}(\tau,\mathbf{z}\,|\,\chi)
\leq 0,
\label{eqn:SE_local}
\end{equation}
in which $\dot{z}_i$ is given by \eqref{eqn:control_affine}, and
\begin{equation}
\begin{aligned}
\sum_{i\in\mathcal{S}_a^{\mathrm{SE}}}
\zeta_{a,i}^{\mathrm{SE}}(\tau,\mathbf{z}\,|\,\chi)
&=
\partial_{\tau}V_a^{\mathrm{SE}}
(\tau,\mathbf{z}_{\mathcal{S}_a^{\mathrm{SE}}}\,|\,\chi)
\\
&\quad+
\alpha_a^{\mathrm{SE}}\!
\left(
V_a^{\mathrm{SE}}
(\tau,\mathbf{z}_{\mathcal{S}_a^{\mathrm{SE}}}\,|\,\chi)
\right),
\end{aligned}
\label{eqn:SE_allocation}
\end{equation}
where $\alpha_a^{\mathrm{SE}}:\mathbb{R}\rightarrow\mathbb{R}$ is an extended class $\mathcal{K}_{\infty}$ function.

\paragraph{PE guidance.}
For each agent $i\in\mathcal{N}$, let $V_i^{\mathrm{PE}}:[0,1]\times\mathcal{Z}^{N}\times\mathcal{C}_N\rightarrow\mathbb{R}_{\geq0}$ be $C^1$. The guidance input $u_i$ of agent $i$ is such that
\begin{equation}
\nabla_{z_i}V^{\mathrm{PE}}(\tau,\mathbf{z}\,|\,\chi)^{\top} \dot{z}_i
+
\zeta_i^{\mathrm{PE}}(\tau,\mathbf{z}\,|\,\chi)
\leq 0,
\label{eqn:PE_local}
\end{equation}
$\forall i\in\mathcal{N}$, in which
\begin{equation*}
V^{\mathrm{PE}}(\tau,\mathbf{z}\,|\,\chi)
=
\sum_{i\in\mathcal{N}}V_i^{\mathrm{PE}}(\tau,\mathbf{z}\,|\,\chi),
\end{equation*}
$\dot{z}_i$ is given by \eqref{eqn:control_affine}, and
\begin{equation}
\begin{aligned}
\zeta_i^{\mathrm{PE}}(\tau,\mathbf{z}\,|\,\chi)
&=
\partial_{\tau}V_i^{\mathrm{PE}}(\tau,\mathbf{z}\,|\,\chi) \\
&\quad+
\omega^{\mathrm{PE}}(\tau)\,
\alpha^{\mathrm{PE}}\!
\left(
V_i^{\mathrm{PE}}(\tau,\mathbf{z}\,|\,\chi)
\right),
\end{aligned}
\label{eqn:PE_rhs}
\end{equation}
where $\omega^{\mathrm{PE}}:[0,1]\rightarrow\mathbb{R}_{\geq 0}$, and $\alpha^{\mathrm{PE}}:\mathbb{R}_{\geq 0}\rightarrow\mathbb{R}_{\geq 0}$ is a subadditive class $\mathcal{K}_{\infty}$ function.

\section{Adaptive Guidance and Formal Guarantees} \label{sec:theoretical}
In this section, we develop general methods for decentralized guidance under coupled multi-agent requirements, establish feasibility conditions and finite-horizon convergence guarantees, and characterize the distributional deviation induced by guidance through a Wasserstein bound on the guided and nominal joint generative distributions.

\subsection{SE Guidance with Finite-Horizon Convergence}
\label{sec:SEguidance_properties}
Let the $C^1$ function $q_a^{\mathrm{SE}}(\tau,\mathbf{z}_{\mathcal{S}_a^{\mathrm{SE}}}\,|\,\chi)$ characterize a shared requirement, where
$q_a^{\mathrm{SE}}\leq 0$ means that the requirement is satisfied. An initial sample $\mathbf{z}(0)\sim p_0$ need not satisfy this requirement. Thus, we introduce a time-varying upper bound $\beta_a^{\mathrm{SE}}(\tau\,|\,\chi,\mathbf{z}(0))$ that is initialized from
$\mathbf{z}(0)$. The bound is chosen to satisfy
\begin{equation}
\beta_a^{\mathrm{SE}}(0\,|\,\chi,\mathbf{z}(0))
\geq
q_a^{\mathrm{SE}}
(0,\mathbf{z}_{\mathcal{S}_a^{\mathrm{SE}}}(0)\,|\,\chi),
\label{eqn:SE_beta_start}
\end{equation}
and is tightened toward a prescribed terminal value. In particular, choosing $\beta_a^{\mathrm{SE}}(1\,|\,\chi,\mathbf{z}(0))=0$ ensures that the desired hard requirement is satisfied by the end of the generative process. Accordingly, define
\begin{equation}
V_a^{\mathrm{SE}}
(\tau,\mathbf{z}_{\mathcal{S}_a^{\mathrm{SE}}}\,|\,\chi)
\coloneqq
q_a^{\mathrm{SE}}
(\tau,\mathbf{z}_{\mathcal{S}_a^{\mathrm{SE}}}\,|\,\chi)
-
\beta_a^{\mathrm{SE}}
(\tau\,|\,\chi,\mathbf{z}(0)),
\label{eqn:SE_finite_horizon_V}
\end{equation}
so that $V_a^{\mathrm{SE}}\leq0$ is equivalent to $q_a^{\mathrm{SE}}\leq\beta_a^{\mathrm{SE}}$. Since factor $a$ involves multiple agents, $\partial_\tau V_a^{\mathrm{SE}}+\alpha_a^{\mathrm{SE}}(V_a^{\mathrm{SE}})$ can be allocated among the involved agents using $w_{a,i}^{\mathrm{SE}}\geq 0$ with $\sum_{i\in\mathcal{S}_a^{\mathrm{SE}}}w_{a,i}^{\mathrm{SE}}=1$, i.e.,
\begin{equation}
\begin{aligned}
\zeta_{a,i}^{\mathrm{SE}}
(\tau,\mathbf{z}\,|\,\chi)
&\coloneqq
w_{a,i}^{\mathrm{SE}}
\Big(
\partial_{\tau}q_a^{\mathrm{SE}}
(\tau,\mathbf{z}_{\mathcal{S}_a^{\mathrm{SE}}}\,|\,\chi)
\\
&\quad
-
\dot{\beta}_a^{\mathrm{SE}}
(\tau\,|\,\chi,\mathbf{z}(0))
\\
&\quad
+
\alpha_a^{\mathrm{SE}}\!\left(
V_a^{\mathrm{SE}}
(\tau,\mathbf{z}_{\mathcal{S}_a^{\mathrm{SE}}}\,|\,\chi)
\right)
\Big).
\end{aligned}
\label{eqn:SE_finite_horizon_zeta}
\end{equation}
Since
$\partial_\tau V_a^{\mathrm{SE}}
=
\partial_\tau q_a^{\mathrm{SE}}
-
\dot{\beta}_a^{\mathrm{SE}}$,
this construction satisfies \eqref{eqn:SE_allocation}. The following theorem shows that satisfying the agent-wise constraint of SE guidance keeps the shared requirement below its prescribed bound throughout the generative process.
\begin{theorem}[Finite-horizon convergence for SE guidance]
\label{thm:SE_finite_horizon}
Let $q_a^{\mathrm{SE}}$ and $\beta_a^{\mathrm{SE}}$ be $C^1$. Suppose \eqref{eqn:SE_beta_start} holds, and let $V_a^{\mathrm{SE}}$ and $\zeta_{a,i}^{\mathrm{SE}}$ be defined by \eqref{eqn:SE_finite_horizon_V} and \eqref{eqn:SE_finite_horizon_zeta}, with $w_{a,i}^{\mathrm{SE}}\geq0$ and $\sum_{i\in\mathcal{S}_a^{\mathrm{SE}}}w_{a,i}^{\mathrm{SE}}=1$. If every agent $i\in\mathcal{S}_a^{\mathrm{SE}}$, along a trajectory of \eqref{eqn:control_affine}, satisfies \eqref{eqn:SE_local}, $\forall \tau\in[0,1]$, then
\begin{equation*}
q_a^{\mathrm{SE}}
(\tau,\mathbf{z}_{\mathcal{S}_a^{\mathrm{SE}}}(\tau)\,|\,\chi)
\leq
\beta_a^{\mathrm{SE}}
(\tau\,|\,\chi,\mathbf{z}(0)),
\quad
\forall\tau\in[0,1].
\end{equation*}
In particular, if $\beta_a^{\mathrm{SE}}(1\,|\,\chi,\mathbf{z}(0))=0$, then we have $q_a^{\mathrm{SE}}(1,\mathbf{z}_{\mathcal{S}_a^{\mathrm{SE}}}(1)\,|\,\chi)\leq 0$, i.e., the final generative state satisfies the shared requirement.
\end{theorem}
\begin{proof}
See Appendix~\ref{appendix:SE_finite_horizon_proof}.
\end{proof}

To ensure that the constraint \eqref{eqn:SE_local} is feasible for each agent, we adapt the coefficients $w_{a,i}^{\mathrm{SE}}$ according to how each agent can affect the shared requirements. Define 
$$
\ell_{a,i}^{\mathrm{SE}} \coloneqq g_i(\tau,\mathbf{z}\,|\,\chi)^{\top} \nabla_{z_i}q_a^{\mathrm{SE}} (\tau,\mathbf{z}_{\mathcal{S}_a^{\mathrm{SE}}}\,|\,\chi),
$$
which determines how the guidance input $u_i$ affects the shared requirement $q_a^{\mathrm{SE}}$. Based on \eqref{eqn:SE_finite_horizon_zeta}, the
constraint \eqref{eqn:SE_local} can be written as $(\ell_{a,i}^{\mathrm{SE}})^{\top}u_i + b_{a,i}^{\mathrm{SE}} \leq 0$ with 
$$
b_{a,i}^{\mathrm{SE}} \coloneqq \nabla_{z_i}q_a^{\mathrm{SE}} (\tau,\mathbf{z}_{\mathcal{S}_a^{\mathrm{SE}}}\,|\,\chi)^{\top} f_i^{\theta}(\tau,\mathbf{z}\,|\,\chi)
+ \zeta_{a,i}^{\mathrm{SE}}(\tau,\mathbf{z}\,|\,\chi).
$$
When $b_{a,i}^{\mathrm{SE}}\leq0$, $u_i=0$ already satisfies this inequality. When $b_{a,i}^{\mathrm{SE}}>0$, agent $i$ must apply a guidance input that reduces the corresponding requirement. 

An agent may participate in multiple shared requirements, which can typically be classified into Liveness (achieving good things) and Safety (avoiding bad things), e.g., \cite{ames2019control}, with the requirements in each category composed into one factor \cite{lindemann2018control,lin2025hierarchy}. To address potential conflicts between the two factors, let $a\in\{\mathrm{L},\mathrm{S}\}$, and $\hat{\ell}_{a,i}^{\mathrm{SE}}\coloneqq {\ell_{a,i}^{\mathrm{SE}}}/{\|\ell_{a,i}^{\mathrm{SE}}\|}$ if $\ell_{a,i}^{\mathrm{SE}}\neq 0$ and $\hat{\ell}_{a,i}^{\mathrm{SE}}\coloneqq 0$ otherwise. Then, $\eta_i^{\mathrm{SE}}\coloneqq(\hat{\ell}_{\mathrm{L},i}^{\mathrm{SE}})^{\top}\hat{\ell}_{\mathrm{S},i}^{\mathrm{SE}}$ characterizes how aligned the effects of $u_i$ on the two active requirements is. For $\delta\in(0,1)$, let $\phi_\delta:[-1,1]\rightarrow[-1,0]$ be a $C^2$ function satisfying $\phi_\delta(\eta)=\eta$ for $\eta\leq-\delta$, $\phi_\delta(\eta)=0$ for $\eta\geq\delta$, and $\phi_\delta(\eta)\leq\min\{\eta,0\}$ for $\eta\in[-1,1]$. Then, we define the effective directions 
$$
d_{\mathrm{L},i}^{\mathrm{SE}}\coloneqq \ell_{\mathrm{L},i}^{\mathrm{SE}} - \|\ell_{\mathrm{L},i}^{\mathrm{SE}}\| \phi_\delta(\eta_i^{\mathrm{SE}}) \hat{\ell}_{\mathrm{S},i}^{\mathrm{SE}}
$$ 
and 
$$
d_{\mathrm{S},i}^{\mathrm{SE}}\coloneqq \ell_{\mathrm{S},i}^{\mathrm{SE}} - \|\ell_{\mathrm{S},i}^{\mathrm{SE}}\| \phi_\delta(\eta_i^{\mathrm{SE}}) \hat{\ell}_{\mathrm{L},i}^{\mathrm{SE}}.
$$
When only one SE factor is active at agent $i$, we set $d_{a,i}^{\mathrm{SE}}=\ell_{a,i}^{\mathrm{SE}}$. Intuitively, the modification is inactive when the two effective directions are sufficiently aligned and increases as they become opposed. As such, for each active SE factor, we choose the adaptive constraint allocation coefficient in \eqref{eqn:SE_finite_horizon_zeta} as
\begin{equation}
w_{a,i}^{\mathrm{SE}}
\coloneqq
\frac{
\|d_{a,i}^{\mathrm{SE}}\|^2
}{
\sum_{j\in\mathcal{S}_a^{\mathrm{SE}}}
\|d_{a,j}^{\mathrm{SE}}\|^2
}.
\label{eqn:C2_weights}
\end{equation}
Then, the following proposition establishes the existence of guidance inputs such that \eqref{eqn:SE_local} is satisfied.
\begin{proposition}[Adaptive SE guidance]
\label{prop:c2_feasibility}
If each agent is incident to at most two simultaneously active SE factors, and for every active SE factor the denominator of \eqref{eqn:C2_weights} is nonzero and $d_{a,i}^{\mathrm{SE}}\neq 0$ when $b_{a,i}^{\mathrm{SE}}>0$, then, there exists a guidance input $u_i$ satisfying \eqref{eqn:SE_local}, $\forall \tau\in[0,1]$. Such a guidance input can be obtained from
\begin{equation}
\boxed{
\begin{aligned}
&\underset{u_i\in\mathbb{R}^{m_i}}{\arg\min}
\quad u_i^{\top}H_i u_i\\
&\quad \mathrm{s.t.}\quad\;\;\,
(\ell_{a,i}^{\mathrm{SE}})^{\top}u_i
+
b_{a,i}^{\mathrm{SE}}\leq0,
\quad \forall a\in\mathcal{A}_i^{\mathrm{SE}}
\end{aligned}
}
\label{eqn:SE_local_QP}
\end{equation}
which admits a unique minimizer when $H_i\succ 0$, where $\mathcal{A}_i^{\mathrm{SE}} \coloneqq \{a\in\mathcal{A}^{\mathrm{SE}}: i\in\mathcal{S}_a^{\mathrm{SE}}\}$.
\end{proposition}
\begin{proof}
See Appendix~\ref{appendix:c2_proof}.
\end{proof}

\subsection{PE Guidance with Finite-Horizon Convergence}
Next, we propose a method to ensure finite-horizon convergence for PE guidance. Recall that, unlike the real-valued $V_a^{\mathrm{SE}}\in\mathbb{R}$ in the SE guidance scenario, here each agent has a nonnegative $V_i^{\mathrm{PE}}\in\mathbb{R}_{\geq 0}$ representing the violation of its private requirement that can be dependent on its neighbors. Let $\partial_{\tau}V_i^{\mathrm{PE}}(\tau,\mathbf{z}\,|\,\chi)=0$ and the total violation be 
$$
V^{\mathrm{PE}}(\mathbf{z}\,|\,\chi)=\sum_{i\in\mathcal{N}}V_i^{\mathrm{PE}}(\mathbf{z}\,|\,\chi),
$$
so $V^{\mathrm{PE}}=0$ means that all private requirements are satisfied. We then choose
\begin{equation}
\alpha^{\mathrm{PE}}(s)
\coloneqq
c^{\mathrm{PE}}s^{\rho^{\mathrm{PE}}},
\label{eqn:PE_finite_time_alpha}
\end{equation}
where $c^{\mathrm{PE}}\in\mathbb{R}_{>0}$ and $\rho^{\mathrm{PE}}\in(0,1)$. The choice of $0<\rho^{\mathrm{PE}}<1$ enables finite-horizon convergence, and $\omega^{\mathrm{PE}}(\tau)$ in \eqref{eqn:PE_rhs} modulates the required decrease of $V^{\mathrm{PE}}$ over $\tau$, as shown in the following theorem.
\begin{theorem}[Finite-horizon convergence for PE guidance]
\label{thm:private_finite_time}
Let $\alpha^{\mathrm{PE}}$ be given by \eqref{eqn:PE_finite_time_alpha}. If every agent, along a trajectory of \eqref{eqn:control_affine}, satisfies \eqref{eqn:PE_local}, $\forall \tau\in[0,1]$, then
\begin{equation*}
\begin{aligned}
V^{\mathrm{PE}}(\mathbf{z}(\tau)\,|\,\chi)
&\leq
\max
\left\{
0, \Lambda^{\mathrm{PE}}(\tau)
\right\}^{\frac{1}{1-\rho^{\mathrm{PE}}}},
\end{aligned}
\end{equation*}
in which $\Lambda^{\mathrm{PE}}(\tau) \coloneqq \left(V^{\mathrm{PE}}(\mathbf{z}(0)\,|\,\chi)\right)^{1-\rho^{\mathrm{PE}}}
\!\!-c^{\mathrm{PE}}(1-\rho^{\mathrm{PE}})\int_{0}^{\tau}\!\omega^{\mathrm{PE}}(s)\,\mathrm{d}s$. In particular, if $\Lambda^{\mathrm{PE}}(1)\leq 0$, then $V^{\mathrm{PE}}(\mathbf{z}(1)\,|\,\chi)=0$, i.e., the final generated state satisfies all private requirements.
\end{theorem}
\begin{proof}
See Appendix~\ref{appendix:PE_proof}.
\end{proof}

Unlike SE guidance, PE guidance assigns each agent a single inequality constraint. Define
\begin{equation*}
\ell_i^{\mathrm{PE}} \coloneqq g_i(\tau,\mathbf{z}\,|\,\chi)^{\top}
\nabla_{z_i}V^{\mathrm{PE}}(\mathbf{z}\,|\,\chi),
\end{equation*}
\begin{align*}
b_i^{\mathrm{PE}}
&\coloneqq
\nabla_{z_i}V^{\mathrm{PE}}(\mathbf{z}\,|\,\chi)^{\top}
f_i^{\theta}(\tau,\mathbf{z}\,|\,\chi)\\
&\quad 
+ c^{\mathrm{PE}}\omega^{\mathrm{PE}}(\tau)
\left(
V_i^{\mathrm{PE}}(\mathbf{z}\,|\,\chi)
\right)^{\rho^{\mathrm{PE}}}.
\end{align*}
Then the constraint \eqref{eqn:PE_local} becomes $(\ell_i^{\mathrm{PE}})^{\top}u_i+b_i^{\mathrm{PE}}\leq0$. If $b_i^{\mathrm{PE}}\leq0$, the nominal input $u_i=0$ already satisfies the constraint. If $b_i^{\mathrm{PE}}>0$, a feasible guidance input exists when $\ell_i^{\mathrm{PE}}\neq0$. This gives the following feasibility result.
\begin{corollary}
\label{cor:PE_local_feasibility}
Per Theorem~\ref{thm:private_finite_time}, if $b_i^{\mathrm{PE}}\leq 0$, or if $b_i^{\mathrm{PE}}>0$ and $\ell_i^{\mathrm{PE}}\neq 0$, then there exists a guidance input $u_i$ satisfying \eqref{eqn:PE_local}, $\forall \tau \in [0,1]$. Such a guidance input can be obtained from
\begin{equation}
\boxed{
\begin{aligned}
&\underset{u_i\in\mathbb{R}^{m_i}}{\arg\min}
\quad u_i^{\top}H_i u_i\\
&\quad \mathrm{s.t.}\quad\;\;\,
(\ell_i^{\mathrm{PE}})^{\top}u_i
+b_i^{\mathrm{PE}}
\leq 0
\end{aligned}
}
\label{eqn:PE_local_qp_general}
\end{equation}
which admits a unique minimizer when $H_i\succ0$.
\end{corollary}
\begin{proof}
See Appendix~\ref{appendix:PE_feasibility_proof}.
\end{proof}

\subsection{Distributional Deviation}
For generated objects to satisfy certain hard requirements or reach a desired set in the generative space, the nominal vector field is modified by SE or PE guidance. We then establish a Wasserstein bound to characterize the resulting deviation from the nominal joint generative distribution. The bound relates this deviation to the magnitude of the guidance correction, providing a theoretical characterization of such intervention. Specifically, consider nominal and guided joint generative state trajectories initialized from the same random sample, i.e., $\frac{\mathrm{d}}{\mathrm{d}\tau}\mathbf{X}^{\mathrm{nom}}(\tau)=f^{\theta}(\tau,\mathbf{X}^{\mathrm{nom}}(\tau)\,|\,\chi)$ and $\frac{\mathrm{d}}{\mathrm{d}\tau}\mathbf{X}^{\mathrm{gui}}(\tau) =f^{\theta}(\tau,\mathbf{X}^{\mathrm{gui}}(\tau)\,|\,\chi)+\Gamma(\tau\,|\,\xi)$ with $\mathbf{X}^{\mathrm{nom}}(0)=\mathbf{X}^{\mathrm{gui}}(0)=\xi\sim p_0$, where the joint guidance correction is denoted by 
\begin{equation}
\label{eqn:joint_guidance_correction}
\begin{aligned}
\Gamma(\tau\,|\,\xi)
\coloneqq &
\Bigl[
  (g_1(\tau,\mathbf{X}^{\mathrm{gui}}(\tau)\,|\,\chi)
  u_1(\tau))^{\top},
  \ldots, \\
&\quad
  (g_N(\tau,\mathbf{X}^{\mathrm{gui}}(\tau)\,|\,\chi)
  u_N(\tau))^{\top}
\Bigr]^{\top}.
\end{aligned}
\end{equation}
Let $\mu_1$ and $\nu_1$ denote the probability distributions of $\mathbf{X}^{\mathrm{nom}}(1)$ and $\mathbf{X}^{\mathrm{gui}}(1)$, respectively.
\begin{theorem}[Wasserstein bound on the guidance-induced distributional deviation]
\label{thm:wasserstein}
Assume $\mu_1$ and $\nu_1$ have finite second moments and, for all $\tau\in[0,1]$ and $\mathbf{y},\mathbf{z} \in \mathcal{Z}^N$, $\left\|f^{\theta}(\tau,\mathbf{y}\,|\,\chi)-f^{\theta}(\tau,\mathbf{z}\,|\,\chi)\right\|\leq L(\tau)\left\|\mathbf{y}-\mathbf{z}\right\|$. Denote
\begin{equation*}
\gamma(s)
\coloneqq
\sqrt{
\int_{\mathcal{Z}^{N}}\!
\left\|\Gamma(s\,|\,\xi)\right\|^2
p_0(\xi)
\,\mathrm{d}\xi
},
\end{equation*}
where $\Gamma(s\,|\,\xi)$ per \eqref{eqn:joint_guidance_correction} is the guidance correction corresponding to the initial sample $\xi\sim p_0$. If $L:[0,1]\rightarrow\mathbb{R}_{\geq0}$ and $\gamma:[0,1]\rightarrow\mathbb{R}_{\geq0}$ are integrable, then
\begin{equation}
W_2(\mu_1,\nu_1)
\leq
\int_0^1 \!
\exp\!\left(
\int_s^1 \! 
L(r)
\,\mathrm{d}r
\right) \gamma(s)
\,\mathrm{d}s.
\label{eqn:wasserstein_bound}
\end{equation}
In addition, given $\chi$, let $\bar{\mu}_1$ and $\bar{\nu}_1$ denote the probability distributions of $\mathcal{D}_N(\mathbf{X}^{\mathrm{nom}}(1)\,|\,\chi)$ and $\mathcal{D}_N(\mathbf{X}^{\mathrm{gui}}(1)\,|\,\chi)$, respectively. If the decoder satisfies $\left\|\mathcal{D}_N(\mathbf{y}\,|\,\chi)-\mathcal{D}_N(\mathbf{z}\,|\,\chi)\right\|\leq L_D\left\|\mathbf{y}-\mathbf{z}\right\|$, $\forall \mathbf{y}, \mathbf{z}\in \mathcal{Z}^N$, for some $L_D\in\mathbb{R}_{\geq 0}$, then 
\begin{equation}
W_2(\bar{\mu}_1,\bar{\nu}_1)
\leq
L_D W_2(\mu_1,\nu_1).
\label{eqn:decoder_wasserstein_bound}
\end{equation}
\end{theorem}
\begin{proof}
See Appendix~\ref{appendix:wasserstein_proof}.
\end{proof}

Empirically evaluation of the distributional deviation between the nominal vector field and the guided one is in Appendix~\ref{app:task2_results}. 

\subsection{Algorithm}
Algorithm~\ref{alg:guided_sampling} summarizes the procedure of DeGG-Flow.

\begin{algorithm}[H]
\caption{DeGG-Flow}
\label{alg:guided_sampling}
\begin{algorithmic}[1]

\Require Condition $\chi$, decoder $\mathcal D_N$, pre-trained flow matching model $f^\theta$, and initial distribution $p_0$.

\State $\mathbf z(0) \sim p_0$.

\If{SE guidance}
\State 
Specify $q_a^{\mathrm{SE}}$, $\beta_a^{\mathrm{SE}}$, $\alpha_a^{\mathrm{SE}}$, and $w_{a,i}^{\mathrm{SE}}$, and construct $V_a^{\mathrm{SE}}$ and $\zeta_{a,i}^{\mathrm{SE}}$ using \eqref{eqn:SE_finite_horizon_V}--\eqref{eqn:SE_finite_horizon_zeta}.
\ForAll{$i\in\mathcal N$}
\State
$u_i^*\leftarrow$ solution of \eqref{eqn:SE_local_QP}.
\EndFor
\ElsIf{PE guidance}
\State
Choose $\alpha^{\mathrm{PE}}$ and $\omega^{\mathrm{PE}}$ according to Theorem~\ref{thm:private_finite_time}.
\ForAll{$i\in\mathcal N$}
\State
$u_i^*\leftarrow$ solution of \eqref{eqn:PE_local_qp_general}.
\EndFor
\EndIf

\State
Evolve $\mathbf z(\tau)$ over $\tau\in[0,1]$ per \eqref{eqn:control_affine} with $u_i=u_i^*$.

\State \Return $\mathcal D_N(\mathbf z(1)\,|\,\chi)$.

\end{algorithmic}
\end{algorithm}

Note that DeGG-Flow requires neither retraining nor a particular neural network architecture for $f^\theta$.

\section{Applications}
In this section, we apply DeGG-Flow to three multi-agent tasks with coupled requirements, where Task~1 and Task~2 are visualized in MuJoCo~\citep{todorov2012mujoco}.

\subsection{Task 1: Multi-Robot Policy Generation} \label{sec:task1}
Motivated by disaster response and search-and-rescue tasks, we study how multiple robots can collaborate and exploit tools in the environment to accomplish tasks beyond the capability of an individual robot. In Task~1, a team of robots faces a spatial gap that cannot be crossed directly, and collaboratively uses a plank to construct a traversable bridge. Since all robots act on the same rigid body, their generated motions must remain mutually compatible and satisfy shared task and safety requirements. Therefore, SE guidance is suitable for Task~1.

The generated object is a motion policy specifying the robot contact point velocities on the plank. The condition $\chi$ contains the current robot and plank states and environment information. Only the first few actions of each generated policy are executed before a new policy is generated. Task~1 contains two SE factors, i.e., $q_{\rm task}(\mathbf z\,|\,\chi)$, which encodes task progress at the plank pose reached after the executed actions, and $q_{\rm safe}(\mathbf z\,|\,\chi)$, which evaluates safety along these actions. For each factor $a\in\{\mathrm{task},\mathrm{safe}\}$, $V_a^{\rm SE}(\tau,\mathbf z\,|\,\chi)=q_a(\mathbf z\,|\,\chi)-\beta_a(\tau \,|\, \chi,\mathbf{z}(0))$, where $\beta_a$ is a time-varying upper bound on $q_a$ during each policy generation. Detailed settings are provided in Appendix~\ref{app:task1_model}.

\begin{table}[t]
\caption{Success rates for Task~1}
\label{tab:task1_results}
\centering
\small
\setlength{\tabcolsep}{10pt}
\begin{tabular}{lcc}
\toprule
$N$ & Nominal & Guided\\
\midrule
$4$ & $46/50$ (92\%) & $50/50$ (100\%)\\
$5$ & $45/50$ (90\%) & $50/50$ (100\%)\\
$6$ & $43/50$ (86\%) & $50/50$ (100\%)\\
$7$ (unseen) & $42/50$ (84\%) & $50/50$ (100\%)\\
$8$ (unseen) & $42/50$ (84\%) & $50/50$ (100\%)\\
\midrule
Overall & $218/250$ (87.2\%) & $250/250$ (100\%)\\
\bottomrule
\end{tabular}
\end{table}

\begin{figure}[t]
\centering
\includegraphics[width=1\linewidth]{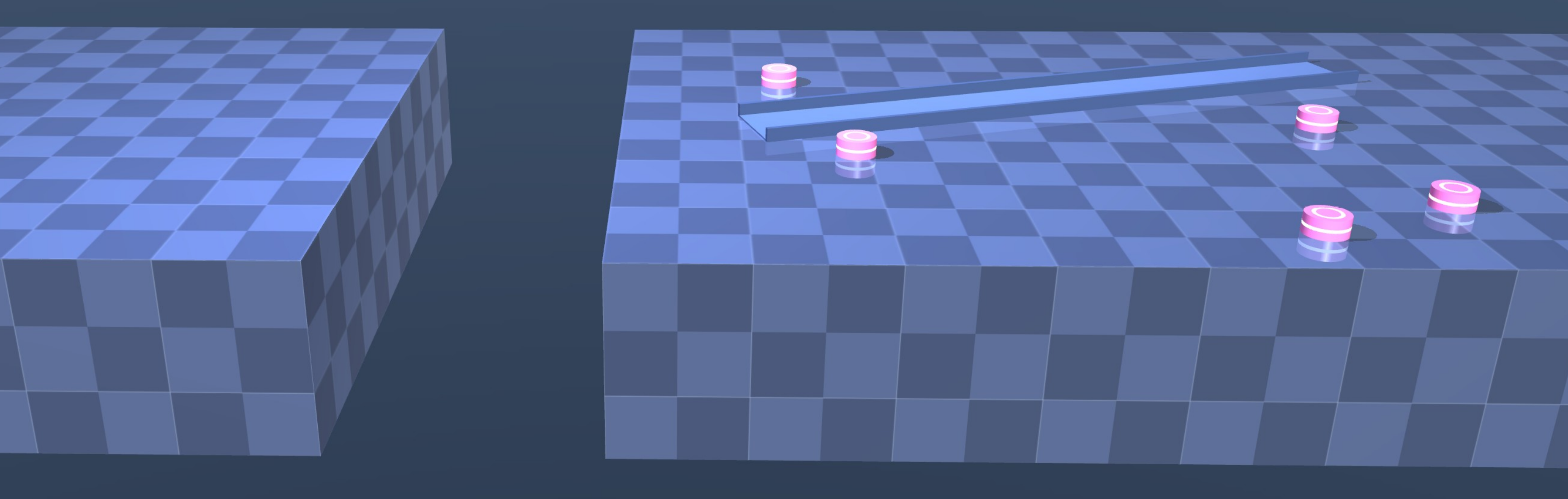}

\vspace{1mm}

\includegraphics[width=1\linewidth]{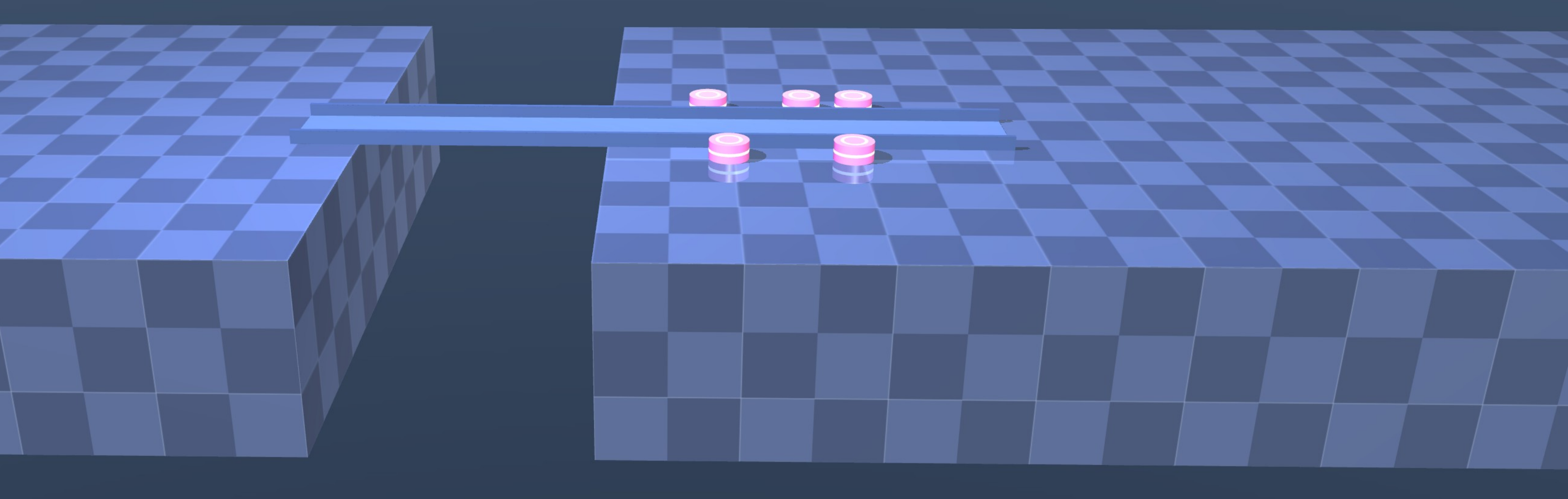}

\vspace{1mm}

\includegraphics[width=1\linewidth]{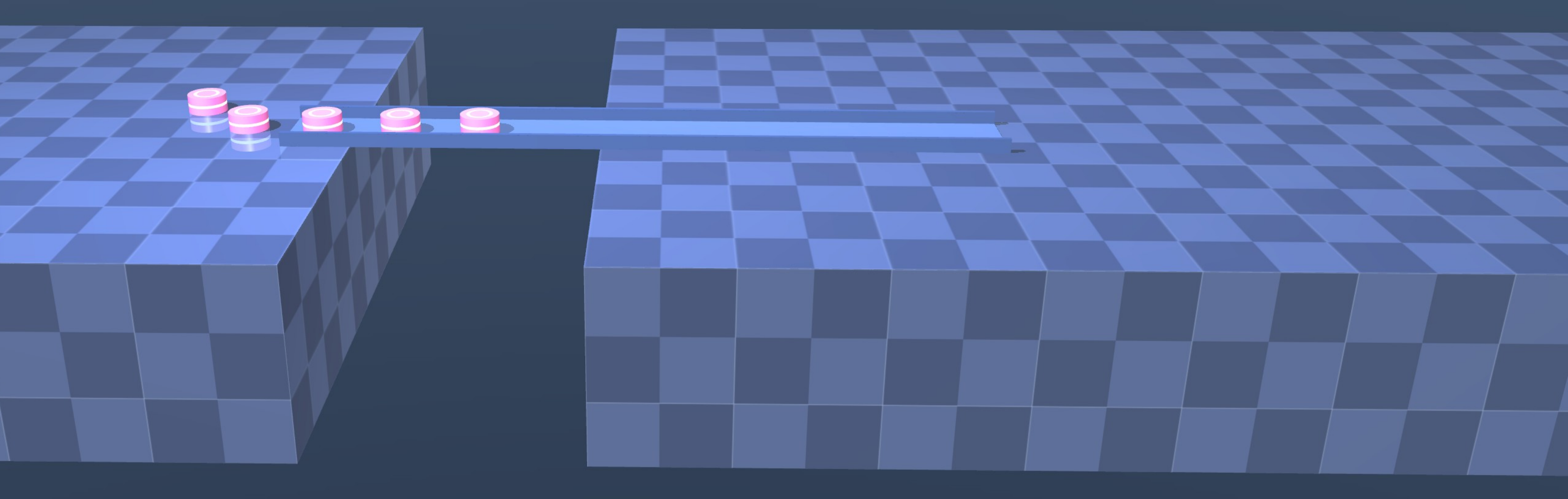}

\caption{Snapshots of an $N=5$ complete process of Task~1 using DeGG-Flow. Top: the initial configuration. Middle: successful bridge construction. Bottom: crossing the bridge.}
\label{fig:task1_rollout}
\end{figure}

\begin{table*}[t]
\caption{Success rates for Task~2}
\label{tab:task2_success}
\centering
\small
\setlength{\tabcolsep}{10pt}
\begin{tabular}{lccc}
\toprule
$N$ & Natural Nominal & Private Nominal & Guided\\
\midrule
$4$ & $44/50$ (88\%) & $36/50$ (72\%) & $50/50$ (100\%)\\
$5$ & $35/50$ (70\%) & $31/50$ (62\%) & $50/50$ (100\%)\\
$6$ & $32/50$ (64\%) & $26/50$ (52\%) & $50/50$ (100\%)\\
$7$ (unseen) & $16/50$ (32\%) & $13/50$ (26\%) & $50/50$ (100\%)\\
$8$ (unseen) & $8/50$ (16\%) & $6/50$ (12\%) & $50/50$ (100\%)\\
\midrule
Overall & $135/250$ (54.0\%) & $112/250$ (44.8\%) & $250/250$ (100\%)\\
\bottomrule
\end{tabular}
\end{table*}

\begin{figure*}[t]
\centering
\includegraphics[width=0.325\linewidth]{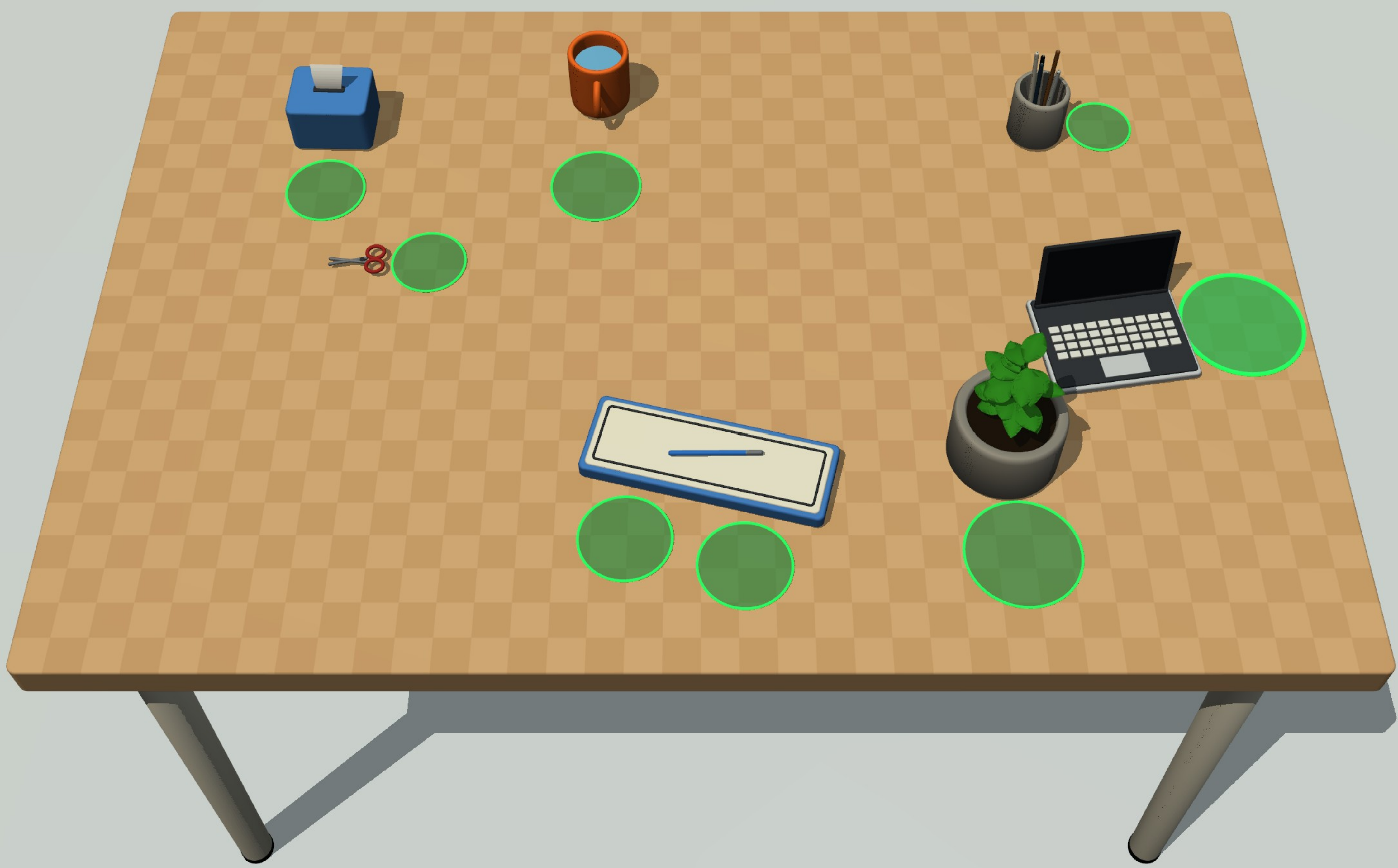}\hfill
\includegraphics[width=0.325\linewidth]{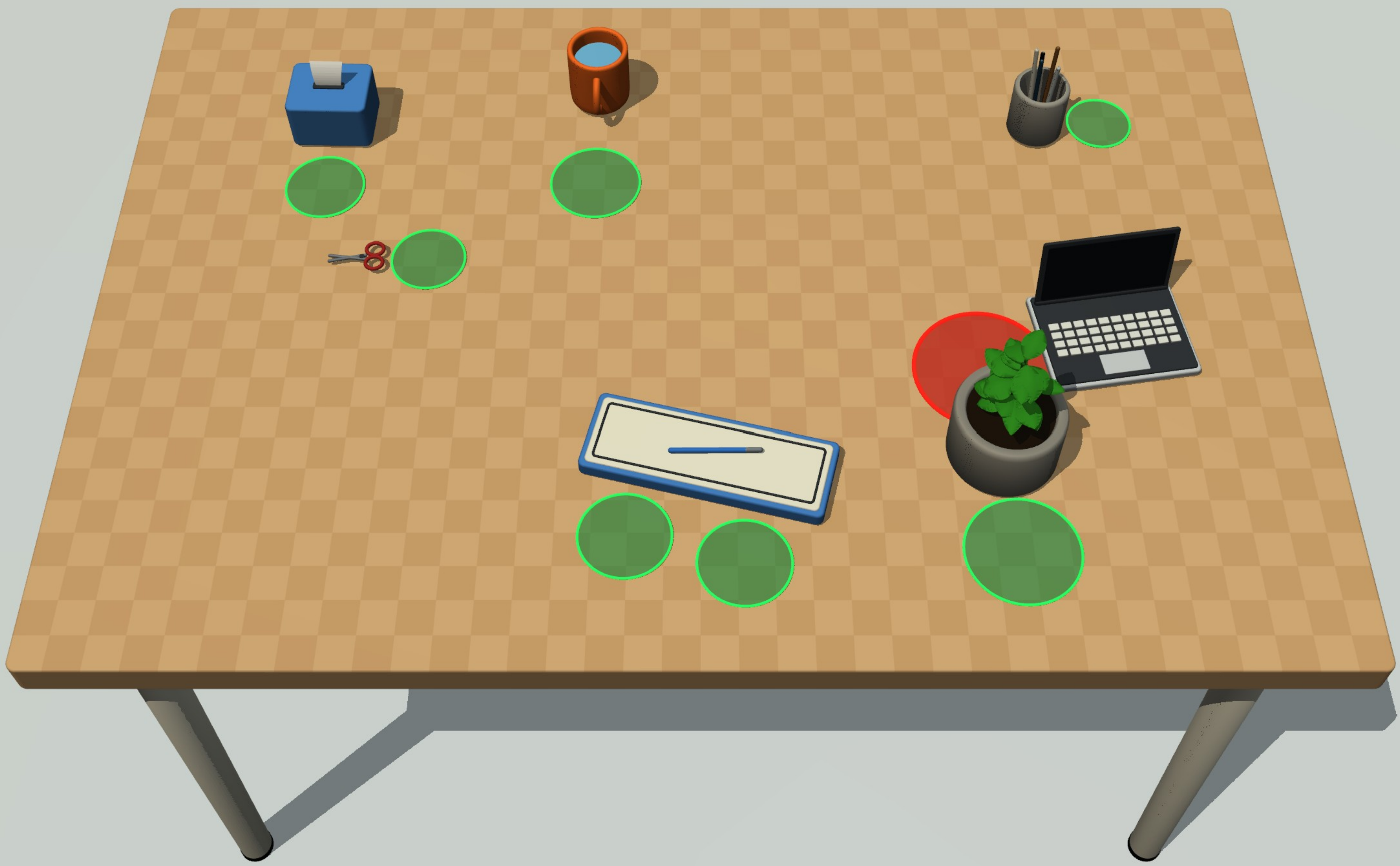}\hfill
\includegraphics[width=0.325\linewidth]{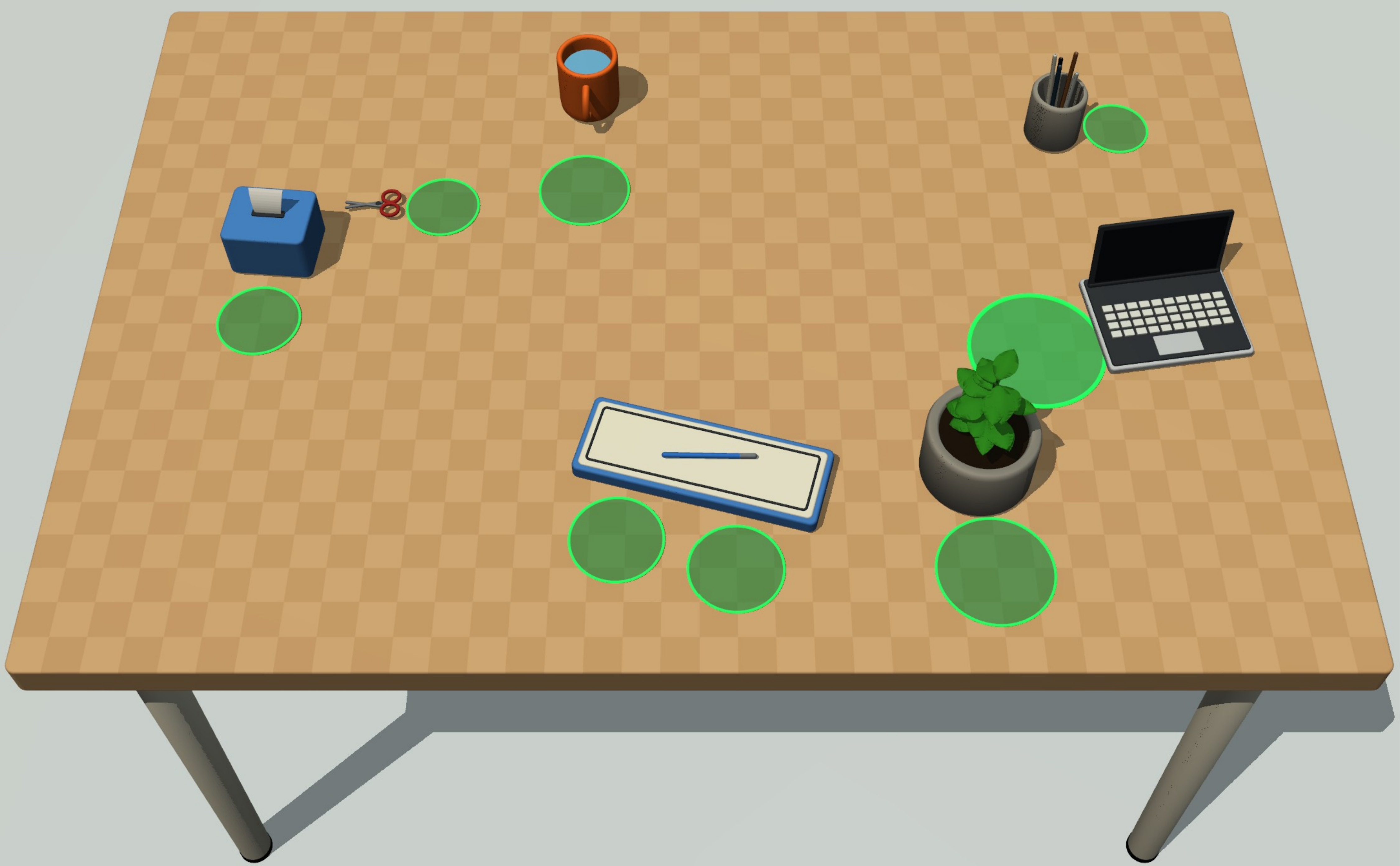}
\caption{An $N=7$ example for Task~2. Green circles represent the objects' affordance regions that are unobstructed, and the red circle represents one that is obstructed. Left: the multi-object scene generated using the pre-trained model (i.e., the nominal vector field) satisfies the default right-handed laptop use requirement. Middle: the same scene evaluated for a left-handed user, where the required mouse operating region on the left side of the laptop is occupied by a potted plant. Right: DeGG-Flow directly generates a scene satisfying the left-handed use requirement without model retraining.}
\label{fig:task2_personalization}
\end{figure*}

The plank pose, robot initial and goal positions, and robot-plank contact points are randomly sampled subject to workspace constraints. Each complete process consists of three stages. First, the robots move to the contact points. Then, the learned multi-agent motion policy generates the collaborative motions for the robots to manipulate the plank to form a bridge over the spatial gap, with or without SE guidance. The success rates in Table~\ref{tab:task1_results} evaluate only this manipulation stage. After the bridge is successfully constructed, the robots go across the spatial gap via the bridge. 

For evaluation, every robot uses the same trained GNN representing the nominal vector field $f_i^\theta$ in \eqref{eqn:control_affine}. We refer to the generative process with $u_i=0$ as \emph{Nominal}, and to the same nominal vector field with SE guidance $g_i u_i^*$, where $u_i^*$ is obtained from \eqref{eqn:SE_local_QP}, as \emph{Guided}. Nominal and Guided use the same initial physical condition, initial sample $\mathbf z(0)$, conditioning information $\chi$, and decoder $\mathcal D_N$. $f_i^\theta$ is trained with $N\in\{4,5,6\}$, and we evaluate 50 manipulation trials for each $N\in\{4,5,6,7,8\}$, with $N=7,8$ evaluating generalization to team sizes unseen in the training data. Figure~\ref{fig:task1_rollout} visualizes an $N=5$ Guided example.

As shown in Table~\ref{tab:task1_results}, Nominal fails in 32 trials, whereas Guided successfully constructs a traversable bridge in all 250 trials, including the 100 trials with unseen team sizes of $N=7,8$. These results show that SE guidance can recover failures of the nominal generative policy and preserve successful collaboration at team sizes beyond those used during training. Additional results are provided in Appendix~\ref{app:task1_results}.

\begin{figure*}[t]
\centering
\includegraphics[width=1\linewidth]{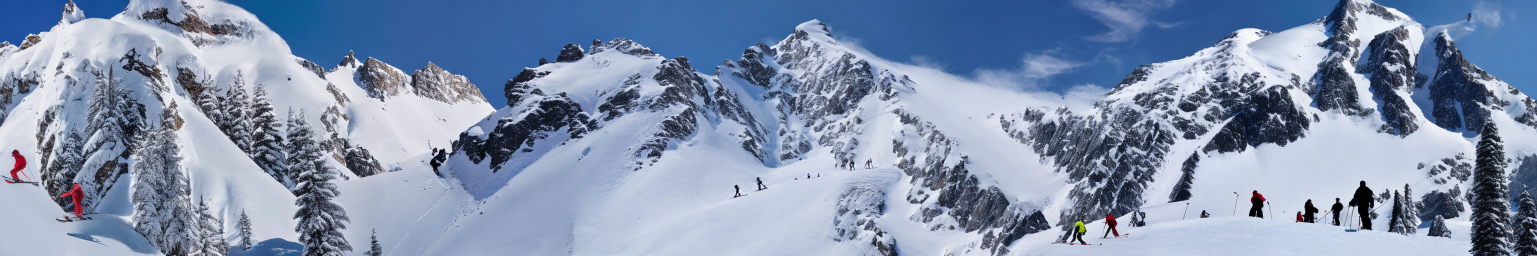}

\vspace{1mm}

\includegraphics[width=1\linewidth]{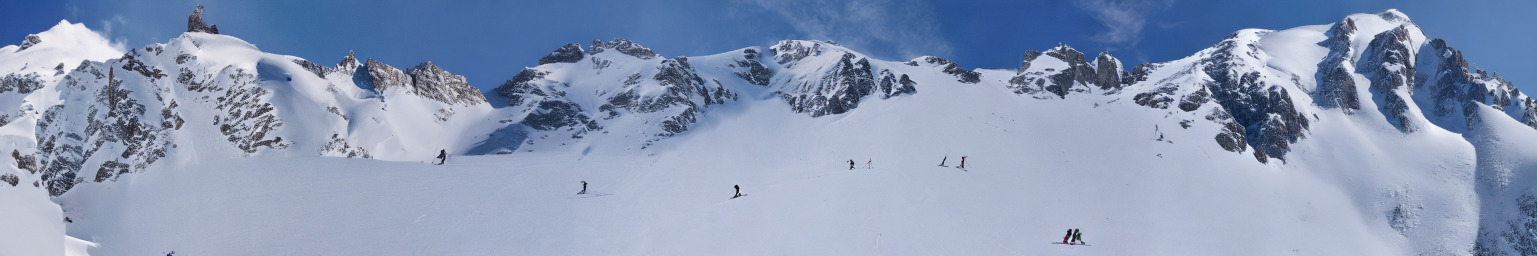}

\vspace{1mm}

\includegraphics[width=1\linewidth]{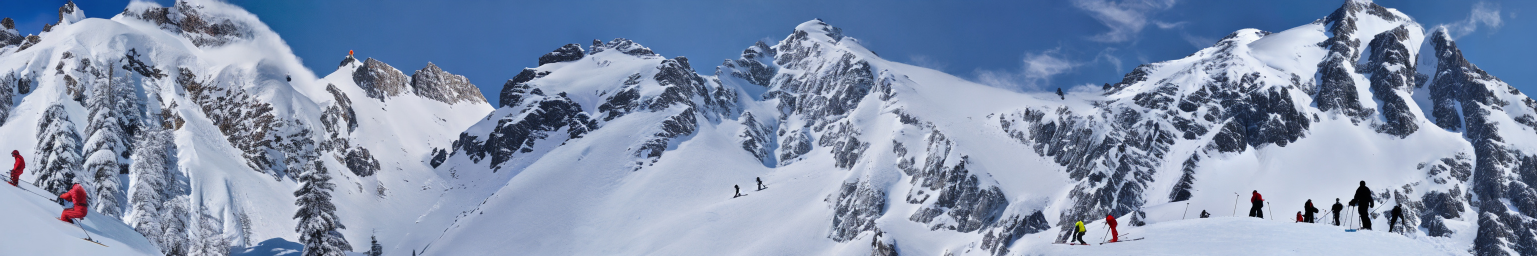}

\caption{Panoramic images for the prompt ``a photo of a snowy mountain peak with skiers.'' Top: MultiDiffusion. Middle: SyncDiffusion. Bottom: DeGG-Flow.}
\label{fig:task3_comparison}
\end{figure*}

\subsection{Task 2: Multi-Object Scene Generation} \label{sec:task2}
Multi-object scene generation can facilitate robot learning by providing diverse, task-specific environments for training policies in simulation, e.g., \cite{pfaff2025steerable,yang2024holodeck}. A useful scene must satisfy not only geometric validity, such as collision avoidance and workspace containment, but also how individual objects are intended to be accessed or used. Each object therefore has its own affordance requirement, and whether that requirement is satisfied can depend on neighboring objects. Thus, PE guidance is suitable for Task~2.

Task~2 generates a tabletop arrangement of $N$ objects. For object $i\in\mathcal N$, denote its pose by
$g_{O_i}=(R_i,p_i)\in\mathbb{SE}(2)$, and let $\mathcal D_N(\mathbf z\,|\,\chi)=\{g_{O_i}\}_{i\in\mathcal N}$. The condition $\chi$ specifies the selected objects from a library, their geometry, the tabletop, and their affordance requirements. For each object, one or more nearby regions are specified and should be unobstructed when the object is used. We refer to these as \emph{affordance regions}. For object $i$, let $h_{ik}(\mathbf z\,|\,\chi)$ denote its $k$th geometric requirement, where $h_{ik}\geq0$ means that the requirement is satisfied. These requirements include object separation, unobstructed affordance regions, and containment of both objects and affordance regions within the tabletop. Let $r_{ik}=-h_{ik}$ and define $V^{\mathrm{PE}} = \sum_{i=1}^{N}V_i^{\mathrm{PE}}$ with $V_i^{\mathrm{PE}}(\mathbf z\,|\,\chi)=\sum_k\frac{1}{2}\max\!\left\{r_{ik}(\mathbf z\,|\,\chi)+\delta^{\mathrm{PE}},0\right\}^{2}$, in which $\delta^{\mathrm{PE}}>0$ provides a margin for satisfying the hard requirements. Specifically, $V^{\mathrm{PE}}(\mathbf{z}\,|\,\chi)\leq(\delta^{\mathrm{PE}})^2/2$ implies $r_{ik}(\mathbf{z}\,|\,\chi)\leq 0$ for all $i,k$, even when $V^{\mathrm{PE}}(\mathbf{z}\,|\,\chi)>0$, whose proof, together with detailed parameter settings, is provided in Appendix~\ref{appendix:DetailedsettingsforTask2}.

For evaluation, the nominal vector field is trained on scenes with $N\in\{4,5,6\}$, and we evaluate 50 cases for each $N\in\{4,5,6,7,8\}$. \emph{Natural Nominal} evaluates scenes generated by the nominal vector field under the default affordance requirements. \emph{Private Nominal} evaluates the same nominal scenes after the affordance requirements are changed (without modifying the pose of any object in the scene). \emph{Guided} starts from the same initial samples and incorporates PE guidance based on the changed requirements. As shown in Table~\ref{tab:task2_success}, the nominal success rate decreases as the number of objects increases, and Private Nominal further shows that a scene valid under the default requirements can become invalid when the desired access or use of an object changes. In contrast, Guided satisfies all requirements in all 250 cases, including the 100 cases with team sizes $N=7,8$ unseen in the training data. Figure~\ref{fig:task2_personalization} shows an $N=7$ example in which PE guidance adapts a nominal scene from a right-handed to a left-handed laptop use requirement without retraining the generative model. These results demonstrate that PE guidance can result in generated objects satisfying private requirements, including at team sizes beyond those used during training. Additional results, including an example with multiple simultaneous violations in a scene generated by the nominal vector field, as well as results across different affordance requirements, are provided in Appendix~\ref{app:task2_results}.

\subsection{Task 3: Panoramic Image Generation} \label{sec:task3}
To demonstrate the applicability of DeGG-Flow to widely used generative models, we consider panoramic image generation using a pre-trained diffusion model~\citep{rombach2022high}. Each agent generates an overlapping image region in the model's latent space, which provides a compressed representation of the image. Similar to Task~1, we choose SE guidance and define $V_{\mathrm{pano}}^{\mathrm{SE}}(\tau,\mathbf z\,|\,\chi)=q_{\mathrm{pano}}(\mathbf{z}\,|\,\chi)-\beta_{\mathrm{pano}}(\tau\,|\,\chi,\mathbf{z}(0))$ to bound the differences between overlapping regions and differences in the channel means and covariances of their latent representations, where the condition $\chi$ is the corresponding text prompt. In addition, for Task~3, we use another form of \eqref{eqn:SE_local_QP} to also allocate the drift (i.e., the nominal vector field),
\begin{equation}
\begin{aligned}
&\underset{u_i\in\mathbb{R}^{m_i}}{\arg\min}
\quad u_i^{\top}H_i u_i\\
&\quad \mathrm{s.t.}\quad\;\;\,
(\ell_{a,i}^{\mathrm{SE}})^{\top}u_i
+
\widetilde{b}_{a,i}^{\mathrm{SE}}
\leq 0,
\quad \forall a\in\mathcal{A}_i^{\mathrm{SE}},
\end{aligned}
\label{eqn:task3_drift_QP}
\end{equation}
in which $\widetilde{b}_{a,i}^{\mathrm{SE}}=w_{a,i}^{\mathrm{SE}}B_a^{\mathrm{SE}}$, 
\begin{equation*}
B_a^{\mathrm{SE}}
=
\sum_{j\in S_a^{\mathrm{SE}}}
(\nabla_{z_j} q_a^{\mathrm SE})^\top f_j^\theta
+
\partial_\tau V_a^{\mathrm SE}
+
\alpha_a^{\mathrm SE}(V_a^{\mathrm SE}),
\end{equation*}
and the theoretical guarantees of \eqref{eqn:task3_drift_QP} follow directly from Section~\ref{sec:SEguidance_properties}. Detailed settings are provided in Appendix~\ref{app:task3_model}.

\begin{table}[t]
\caption{Task~3 results}
\label{tab:task3_results}
\centering
\small
\setlength{\tabcolsep}{4pt}
\begin{tabular}{lcc}
\toprule
Method & Requirements & Time (s)\\
\midrule
MultiDiffusion & $1/50$ (2\%) & $60.4\pm4.3$\\
SyncDiffusion & $15/50$ (30\%) & $474.6\pm191.6$\\
\textbf{DeGG-Flow} & $50/50$ (100\%) & $205.4\pm8.0$\\
\bottomrule
\end{tabular}
\end{table}

We compare DeGG-Flow with MultiDiffusion~\citep{bar2023multidiffusion} and SyncDiffusion~\citep{lee2023syncdiffusion} on five prompts, each evaluated with ten random seeds. For each prompt and seed, all methods use the same pre-trained model, text prompt, initial generative sample, and number of denoising steps. Table~\ref{tab:task3_results} shows the rate of hard requirements met, and the time consumed for generation with mean $\pm$ standard deviation on the same GPU. Figure~\ref{fig:task3_comparison} visualizes a trial of the comparisons. These results demonstrate that DeGG-Flow satisfies hard requirements in all $50$ trials, with generation time in between those of MultiDiffusion and SyncDiffusion. Additional results are provided in Appendix~\ref{app:task3_results}.

\section{Conclusion}
This paper proposes DeGG-Flow, a general framework enabling multi-agent generative models to directly produce objects satisfying hard requirements. Even when these requirements couple multiple agents, each agent can determine its own generative guidance without relying on that of any other agent, allowing the agents to collaborate throughout the generative process without a central decision maker. DeGG-Flow also enables pre-trained models to accommodate new hard requirements unseen during training, without collecting additional training data or retraining the models, which is especially valuable for large generative models that are expensive to train. Future work will explore applications of DeGG-Flow in video generation, music generation, as well as protein and molecular generation.

\bibliographystyle{IEEEtranN}
\bibliography{references}

\appendix
\section{Proofs} \label{appendix:Proofs}

\subsection{Proof of Theorem~\ref{thm:SE_finite_horizon}} \label{appendix:SE_finite_horizon_proof}
For any $i\in\mathcal{S}_a^{\mathrm{SE}}$, \eqref{eqn:SE_finite_horizon_V} gives $\nabla_{z_i} V_a^{\mathrm{SE}} (\tau,\mathbf{z}_{\mathcal{S}_a^{\mathrm{SE}}}\,|\,\chi) = \nabla_{z_i} q_a^{\mathrm{SE}} (\tau,\mathbf{z}_{\mathcal{S}_a^{\mathrm{SE}}}\,|\,\chi)$. Along a trajectory of \eqref{eqn:control_affine},
\begin{align}
\frac{\mathrm{d}}{\mathrm{d}\tau}
V_a^{\mathrm{SE}}
(\tau,\mathbf{z}_{\mathcal{S}_a^{\mathrm{SE}}}(\tau)\,|\,\chi)
&=
\partial_{\tau}
q_a^{\mathrm{SE}}
(\tau,\mathbf{z}_{\mathcal{S}_a^{\mathrm{SE}}}(\tau)\,|\,\chi)
-
\dot{\beta}_a^{\mathrm{SE}}
(\tau\,|\,\chi,\mathbf{z}(0))
\nonumber\\
&\quad+
\sum_{i\in\mathcal{S}_a^{\mathrm{SE}}}
\nabla_{z_i}
q_a^{\mathrm{SE}}
(\tau,\mathbf{z}_{\mathcal{S}_a^{\mathrm{SE}}}(\tau)\,|\,\chi)^{\top}
\left(
f_i^{\theta}(\tau,\mathbf{z}(\tau)\,|\,\chi)
+
g_i(\tau,\mathbf{z}(\tau)\,|\,\chi)u_i
\right).
\label{eqn:app_SE_chain_rule}
\end{align}
According to \eqref{eqn:SE_finite_horizon_zeta}, \eqref{eqn:SE_local} becomes
\begin{align}
&\nabla_{z_i}
q_a^{\mathrm{SE}}
(\tau,\mathbf{z}_{\mathcal{S}_a^{\mathrm{SE}}}(\tau)\,|\,\chi)^{\top}
\left(
f_i^{\theta}(\tau,\mathbf{z}(\tau)\,|\,\chi)
+
g_i(\tau,\mathbf{z}(\tau)\,|\,\chi)u_i
\right) 
\nonumber\\
&+
w_{a,i}^{\mathrm{SE}}
\left(
\partial_{\tau}
q_a^{\mathrm{SE}}
(\tau,\mathbf{z}_{\mathcal{S}_a^{\mathrm{SE}}}(\tau)\,|\,\chi)
-
\dot{\beta}_a^{\mathrm{SE}}
(\tau\,|\,\chi,\mathbf{z}(0))
+
\alpha_a^{\mathrm{SE}}\!\left(
V_a^{\mathrm{SE}}
(\tau,\mathbf{z}_{\mathcal{S}_a^{\mathrm{SE}}}(\tau)\,|\,\chi)
\right)
\right)
\leq0.
\label{eqn:app_SE_local_expanded}
\end{align}
Summing \eqref{eqn:app_SE_local_expanded} over $i\in\mathcal{S}_a^{\mathrm{SE}}$ and using $w_{a,i}^{\mathrm{SE}}\geq0$ such that $\sum_{i\in\mathcal{S}_a^{\mathrm{SE}}} w_{a,i}^{\mathrm{SE}} =1$ yields
\begin{align}
&\sum_{i\in\mathcal{S}_a^{\mathrm{SE}}}
\nabla_{z_i}
q_a^{\mathrm{SE}}
(\tau,\mathbf{z}_{\mathcal{S}_a^{\mathrm{SE}}}(\tau)\,|\,\chi)^{\top}
\left(
f_i^{\theta}(\tau,\mathbf{z}(\tau)\,|\,\chi)
+
g_i(\tau,\mathbf{z}(\tau)\,|\,\chi)u_i
\right)
\nonumber\\
&\;+\,
\partial_{\tau}
q_a^{\mathrm{SE}}
(\tau,\mathbf{z}_{\mathcal{S}_a^{\mathrm{SE}}}(\tau)\,|\,\chi)
-
\dot{\beta}_a^{\mathrm{SE}}
(\tau\,|\,\chi,\mathbf{z}(0))
+
\alpha_a^{\mathrm{SE}}\!
\left(
V_a^{\mathrm{SE}}
(\tau,\mathbf{z}_{\mathcal{S}_a^{\mathrm{SE}}}(\tau)\,|\,\chi)
\right)
\leq 0.
\label{eqn:app_SE_sum}
\end{align}
According to \eqref{eqn:app_SE_sum} and \eqref{eqn:app_SE_chain_rule}, we have
\begin{equation}
\frac{\mathrm{d}}{\mathrm{d}\tau}
V_a^{\mathrm{SE}}
(\tau,\mathbf{z}_{\mathcal{S}_a^{\mathrm{SE}}}(\tau)\,|\,\chi)
+
\alpha_a^{\mathrm{SE}}\!\left(
V_a^{\mathrm{SE}}
(\tau,\mathbf{z}_{\mathcal{S}_a^{\mathrm{SE}}}(\tau)\,|\,\chi)
\right)
\leq0.
\label{eqn:app_SE_global_decay}
\end{equation}
Since $\alpha_a^{\mathrm{SE}}$ is an extended class-$\mathcal{K}_{\infty}$ function, we have $\alpha_a^{\mathrm{SE}}(0)=0$ and $\alpha_a^{\mathrm{SE}}(s)\in \mathbb{R}_{\geq 0}$, $\forall s\in \mathbb{R}_{\geq 0}$. Suppose $V_a^{\mathrm{SE}}(0,\mathbf{z}_{\mathcal{S}_a^{\mathrm{SE}}}(0)\,|\,\chi)\leq 0$ but there exists some $\tau_1\in(0,1]$ such that $V_a^{\mathrm{SE}} (\tau_1,\mathbf{z}_{\mathcal{S}_a^{\mathrm{SE}}}(\tau_1)\,|\,\chi)\in \mathbb{R}_{> 0}$. By continuity, let $\tau_0<\tau_1$ be the last time before $\tau_1$ at which $V_a^{\mathrm{SE}}(\tau_0,\mathbf{z}_{\mathcal{S}_a^{\mathrm{SE}}}(\tau_0)\,|\,\chi)=0$. Then, for every $\tau\in(\tau_0,\tau_1]$, $V_a^{\mathrm{SE}} (\tau,\mathbf{z}_{\mathcal{S}_a^{\mathrm{SE}}}(\tau)\,|\,\chi)\in \mathbb{R}_{> 0}$, $\alpha_a^{\mathrm{SE}}\!\left(V_a^{\mathrm{SE}}(\tau,\mathbf{z}_{\mathcal{S}_a^{\mathrm{SE}}}(\tau)\,|\,\chi)\right)\in \mathbb{R}_{> 0}$, and $\frac{\mathrm{d}}{\mathrm{d}\tau}V_a^{\mathrm{SE}}(\tau,\mathbf{z}_{\mathcal{S}_a^{\mathrm{SE}}}(\tau)\,|\,\chi)<0$ by \eqref{eqn:app_SE_global_decay}, which contradicts $V_a^{\mathrm{SE}}(\tau_0,\mathbf{z}_{\mathcal{S}_a^{\mathrm{SE}}}(\tau_0)\,|\,\chi)=0$ and $V_a^{\mathrm{SE}}(\tau_1,\mathbf{z}_{\mathcal{S}_a^{\mathrm{SE}}}(\tau_1)\,|\,\chi)\in \mathbb{R}_{\geq 0}$. Therefore, 
\begin{equation*}
V_a^{\mathrm{SE}}(\tau,\mathbf{z}_{\mathcal{S}_a^{\mathrm{SE}}}(\tau)\,|\,\chi) \leq 0, \quad \forall\tau\in[0,1],
\end{equation*}
and \eqref{eqn:SE_finite_horizon_V} at $\tau=1$ gives $q_a^{\mathrm{SE}}(1,\mathbf{z}_{\mathcal{S}_a^{\mathrm{SE}}}(1)\,|\,\chi)\leq\beta_a^{\mathrm{SE}}(1\,|\,\chi,\mathbf{z}(0))$.

\subsection{Proof of Proposition~\ref{prop:c2_feasibility}} \label{appendix:c2_proof}
In this paper, we choose
\begin{equation*}
\phi_\delta(\eta)
=
\begin{cases}
\eta, 
&\eta\leq-\delta,\\
{\delta} \left({\eta}/{\delta}-1\right)^3 \left({\eta}/{\delta}+3\right)/{16},
&|\eta|<\delta,\\
0, 
&\eta\geq\delta,
\end{cases}
\end{equation*}
which satisfies that $\phi_\delta\in C^2$, $\phi_\delta(\eta)\in[-1,0]$, and $\phi_\delta(\eta)\leq\min\{\eta,0\}$, $\forall \eta\in[-1,1]$.

Given an agnet $i$, assume it is incident to two active SE factors, indexed by $\mathrm{L}$ and $\mathrm{S}$, at any $\tau \in [0,1]$. Since $\phi_\delta(\eta)\leq\eta$, then we have
\begin{align*}
(\ell_{\mathrm{S},i}^{\mathrm{SE}})^{\top}d_{\mathrm{L},i}^{\mathrm{SE}}
&=
\left\|\ell_{\mathrm{L},i}^{\mathrm{SE}}\right\|
\left\|\ell_{\mathrm{S},i}^{\mathrm{SE}}\right\|
\left(
\eta_i^{\mathrm{SE}}
-
\phi_\delta(\eta_i^{\mathrm{SE}})
\right)
\geq 0,\\
(\ell_{\mathrm{L},i}^{\mathrm{SE}})^{\top}d_{\mathrm{S},i}^{\mathrm{SE}}
&=
\left\|\ell_{\mathrm{L},i}^{\mathrm{SE}}\right\|
\left\|\ell_{\mathrm{S},i}^{\mathrm{SE}}\right\|
\left(
\eta_i^{\mathrm{SE}}
-
\phi_\delta(\eta_i^{\mathrm{SE}})
\right)
\geq 0,\\
(\ell_{\mathrm{L},i}^{\mathrm{SE}})^{\top}d_{\mathrm{L},i}^{\mathrm{SE}}
&=
\left\|\ell_{\mathrm{L},i}^{\mathrm{SE}}\right\|^2
\left(
1-
\eta_i^{\mathrm{SE}}
\phi_\delta(\eta_i^{\mathrm{SE}})
\right)
\geq 0,\\
(\ell_{\mathrm{S},i}^{\mathrm{SE}})^{\top}d_{\mathrm{S},i}^{\mathrm{SE}}
&=
\left\|\ell_{\mathrm{S},i}^{\mathrm{SE}}\right\|^2
\left(
1-
\eta_i^{\mathrm{SE}}
\phi_\delta(\eta_i^{\mathrm{SE}})
\right)
\geq 0,
\end{align*}
where the last two quantities are strictly positive whenever the corresponding $d_{a,i}^{\mathrm{SE}}\neq 0$ since $\eta_i^{\mathrm{SE}}\in[-1,1]$ and
$\phi_\delta(\eta_i^{\mathrm{SE}})\in[-1,0]$. Additionally, for every active SE factor $a$, \eqref{eqn:C2_weights} satisfies $w_{a,i}^{\mathrm{SE}}\geq0$ and $\sum_{i\in\mathcal{S}_a^{\mathrm{SE}}} w_{a,i}^{\mathrm{SE}}=1$. Also, consider the SE inequalities
\begin{equation}
(\ell_{a,i}^{\mathrm{SE}})^{\top}u_i + b_{a,i}^{\mathrm{SE}} \leq 0, \quad \forall a\in\{\mathrm{L},\mathrm{S}\}.
\label{eqn:proof_c2_SE_inequalities}
\end{equation}
Define $\lambda_{a,i} \coloneqq b_{a,i}^{\mathrm{SE}}/((\ell_{a,i}^{\mathrm{SE}})^{\top}d_{a,i}^{\mathrm{SE}})$ when $b_{a,i}^{\mathrm{SE}}>0$, and $\lambda_{a,i}\coloneqq0$ otherwise. One can observe that $(\ell_{a,i}^{\mathrm{SE}})^{\top}d_{a,i}^{\mathrm{SE}}>0$ whenever $b_{a,i}^{\mathrm{SE}}>0$ since $d_{a,i}^{\mathrm{SE}}\neq0$ in this case. Now choose
\begin{equation*}
u_i = -\lambda_{\mathrm{L},i}d_{\mathrm{L},i}^{\mathrm{SE}} -\lambda_{\mathrm{S},i}d_{\mathrm{S},i}^{\mathrm{SE}}.
\end{equation*}
Then, for $a=\mathrm{L}$,
\begin{equation}
(\ell_{\mathrm{L},i}^{\mathrm{SE}})^{\top}u_i
+b_{\mathrm{L},i}^{\mathrm{SE}}
=
-\lambda_{\mathrm{L},i}
(\ell_{\mathrm{L},i}^{\mathrm{SE}})^{\top}d_{\mathrm{L},i}^{\mathrm{SE}}
-\lambda_{\mathrm{S},i}
(\ell_{\mathrm{L},i}^{\mathrm{SE}})^{\top}d_{\mathrm{S},i}^{\mathrm{SE}}
+b_{\mathrm{L},i}^{\mathrm{SE}}.
\label{eqn:proof2_inequality_u_i}
\end{equation}
If $b_{\mathrm{L},i}^{\mathrm{SE}} >0$, the first and last terms on the right-hand side of \eqref{eqn:proof2_inequality_u_i} cancel, and the remaining term is nonpositive. If $b_{\mathrm{L},i}^{\mathrm{SE}}\leq 0$, then $\lambda_{\mathrm{L},i}=0$ and the left-hand side of \eqref{eqn:proof2_inequality_u_i} is nonpositive. The same analysis holds for $a=\mathrm{S}$. Hence, \eqref{eqn:proof_c2_SE_inequalities} are satisfied. If only one SE factor is active, then $d_{a,i}^{\mathrm{SE}}=\ell_{a,i}^{\mathrm{SE}}$, and the same construction with a single coefficient $\lambda_{a,i}$ also gives a feasible $u_i$. Therefore, the feasible set of \eqref{eqn:SE_local_QP} is nonempty.

\subsection{Proof of Theorem~\ref{thm:private_finite_time}} \label{appendix:PE_proof}
Since $\partial_{\tau}V_i^{\mathrm{PE}}(\tau,\mathbf{z}\,|\,\chi)=0$, $\forall i\in\mathcal{N}$, we have
\begin{equation}
\begin{aligned}
\frac{\mathrm{d}}{\mathrm{d}\tau}
V^{\mathrm{PE}}(\mathbf{z}(\tau)\,|\,\chi)
=
\sum_{i\in\mathcal{N}}
\nabla_{z_i}V^{\mathrm{PE}}
(\mathbf{z}(\tau)\,|\,\chi)^{\top}
\left(
f_i^{\theta}(\tau,\mathbf{z}(\tau)\,|\,\chi)+
g_i(\tau,\mathbf{z}(\tau)\,|\,\chi)u_i
\right).
\end{aligned}
\label{eqn:app_PE_chain_rule}
\end{equation}
Using \eqref{eqn:PE_finite_time_alpha} in \eqref{eqn:PE_rhs}, summing \eqref{eqn:PE_local} over $i\in\mathcal{N}$, and applying \eqref{eqn:app_PE_chain_rule} yield
\begin{equation*}
\frac{\mathrm{d}}{\mathrm{d}\tau}
V^{\mathrm{PE}}(\mathbf{z}(\tau)\,|\,\chi)
+
c^{\mathrm{PE}}\omega^{\mathrm{PE}}(\tau)
\sum_{i\in\mathcal{N}}
\left(
V_i^{\mathrm{PE}}(\mathbf{z}(\tau)\,|\,\chi)
\right)^{\rho^{\mathrm{PE}}}
\leq 0.
\end{equation*}
Since $\rho^{\mathrm{PE}}\in(0,1)$, $\alpha^{\mathrm{PE}}(s)=c^{\mathrm{PE}}s^{\rho^{\mathrm{PE}}}$ is subadditive on $\mathbb{R}_{\geq0}$. Hence,
\begin{equation*}
\alpha^{\mathrm{PE}}\!\left(
V^{\mathrm{PE}}(\mathbf{z}(\tau)\,|\,\chi)
\right)
\leq
\sum_{i\in\mathcal{N}}
\alpha^{\mathrm{PE}}\!\left(
V_i^{\mathrm{PE}}(\mathbf{z}(\tau)\,|\,\chi)
\right),
\end{equation*}
and thus
\begin{equation}
\frac{\mathrm{d}}{\mathrm{d}\tau}
V^{\mathrm{PE}}(\mathbf{z}(\tau)\,|\,\chi)
\leq
-
c^{\mathrm{PE}}\omega^{\mathrm{PE}}(\tau)
\left(
V^{\mathrm{PE}}(\mathbf{z}(\tau)\,|\,\chi)
\right)^{\rho^{\mathrm{PE}}},
\label{eqn:app_PE_scalar_decay}
\end{equation}
which adapts the finite-time Lyapunov inequality~\citep{bhat2000finite} to the generative process using the time-varying weight $\omega^{\mathrm{PE}}(\tau)$. On any interval over which $V^{\mathrm{PE}}(\mathbf{z}(\tau)\,|\,\chi)>0$, multiplying \eqref{eqn:app_PE_scalar_decay} by $(1-\rho^{\mathrm{PE}})
\left(V^{\mathrm{PE}}(\mathbf{z}(\tau)\,|\,\chi)\right)^{-\rho^{\mathrm{PE}}}$ leads to
\begin{equation*}
\frac{\mathrm{d}}{\mathrm{d}\tau}
\left(
V^{\mathrm{PE}}(\mathbf{z}(\tau)\,|\,\chi)
\right)^{1-\rho^{\mathrm{PE}}}
\leq
-
c^{\mathrm{PE}}
(1-\rho^{\mathrm{PE}})
\,\omega^{\mathrm{PE}}(\tau).
\end{equation*}
Integrating from $0$ to $\tau$ results in
\begin{equation*}
\left(
V^{\mathrm{PE}}(\mathbf{z}(\tau)\,|\,\chi)
\right)^{1-\rho^{\mathrm{PE}}}
\leq
\left(
V^{\mathrm{PE}}(\mathbf{z}(0)\,|\,\chi)
\right)^{1-\rho^{\mathrm{PE}}}
\!-
c^{\mathrm{PE}}
(1-\rho^{\mathrm{PE}})
\int_{0}^{\tau}
\omega^{\mathrm{PE}}(s)\,\mathrm{d}s.
\end{equation*}
Additionally, \eqref{eqn:app_PE_scalar_decay} and $V^{\mathrm{PE}}\geq0$ imply that, for any $\bar{\tau}\in[0,1]$,
\begin{equation*}
V^{\mathrm{PE}}(\mathbf{z}(\bar{\tau})\,|\,\chi)=0
\;\;\Longrightarrow\;\;
V^{\mathrm{PE}}(\mathbf{z}(\tau)\,|\,\chi)=0,
\;
\forall\tau\in[\bar{\tau},1].
\end{equation*}
Therefore,
\begin{equation*}
V^{\mathrm{PE}}(\mathbf{z}(\tau)\,|\,\chi)
\leq
\max\left\{
0,\;
\left(
V^{\mathrm{PE}}(\mathbf{z}(0)\,|\,\chi)
\right)^{1-\rho^{\mathrm{PE}}}
\!-
c^{\mathrm{PE}}
(1-\rho^{\mathrm{PE}})
\int_{0}^{\tau}
\omega^{\mathrm{PE}}(s)\,\mathrm{d}s
\right\}^{\frac{1}{1-\rho^{\mathrm{PE}}}},
\end{equation*}
and evaluating the upper bound at $\tau=1$ implies $V^{\mathrm{PE}}(\mathbf{z}(1)\,|\,\chi)=0$.

\subsection{Proof of Corollary~\ref{cor:PE_local_feasibility}} \label{appendix:PE_feasibility_proof}
If $b_i^{\mathrm{PE}}\leq0$, then $u_i=0$ satisfies the inequality constraint in \eqref{eqn:PE_local_qp_general}. If $b_i^{\mathrm{PE}}>0$ and $\ell_i^{\mathrm{PE}}\neq0$, choose
\begin{equation*}
u_i = -\frac{b_i^{\mathrm{PE}}}{(\ell_i^{\mathrm{PE}})^{\top}H_i^{-1}\ell_i^{\mathrm{PE}}} H_i^{-1}\ell_i^{\mathrm{PE}}.
\end{equation*}
Since $H_i\succ0$ and $\ell_i^{\mathrm{PE}}\neq0$, we have $(\ell_i^{\mathrm{PE}})^{\top}H_i^{-1}\ell_i^{\mathrm{PE}}>0$, and thus $(\ell_i^{\mathrm{PE}})^{\top}u_i+b_i^{\mathrm{PE}}=0$. Therefore, the feasible set of \eqref{eqn:PE_local_qp_general} is nonempty.

\subsection{Proof of Theorem~\ref{thm:wasserstein}} \label{appendix:wasserstein_proof}
For an initial sample $\xi\sim p_0$, define 
$$
\Delta(\tau\,|\,\xi) \coloneqq \mathbf{X}^{\mathrm{gui}}(\tau\,|\,\xi)-\mathbf{X}^{\mathrm{nom}}(\tau\,|\,\xi).
$$
Since the nominal and guided trajectories start from the same initial sample, $\Delta(0\,|\,\xi)=0$. Then, we can obtain
\begin{equation*}
\Delta(\tau\,|\,\xi)
=
\int_0^\tau \!
\left(
f^{\theta}
(s,\mathbf{X}^{\mathrm{gui}}(s\,|\,\xi)\,|\,\chi)
-
f^{\theta}
(s,\mathbf{X}^{\mathrm{nom}}(s\,|\,\xi)\,|\,\chi)
+
\Gamma(s\,|\,\xi)
\right)
\,\mathrm{d}s.
\end{equation*}
Since $\left\|f^{\theta}(\tau,\mathbf{y}\,|\,\chi)-f^{\theta}(\tau,\mathbf{z}\,|\,\chi)\right\|\leq L(\tau)\left\|\mathbf{y}-\mathbf{z}\right\|$, $\forall \tau\in[0,1]$, $\forall \mathbf y,\mathbf z\in\mathcal Z^N$, we have
\begin{equation*}
\left\|\Delta(\tau\,|\,\xi)\right\|
\leq
\int_0^\tau \!
L(s)\left\|\Delta(s\,|\,\xi)\right\|
\,\mathrm{d}s
+
\int_0^\tau \!
\left\|\Gamma(s\,|\,\xi)\right\|
\,\mathrm{d}s.
\end{equation*}
Then, applying Gr\"onwall's inequality leads to
\begin{equation*}
\left\|\Delta(1\,|\,\xi)\right\|
\leq
\int_0^1
\exp\!\left(
\int_s^1 L(r)\,\mathrm{d}r
\right)
\left\|\Gamma(s\,|\,\xi)\right\|
\,\mathrm{d}s.
\end{equation*}
Thus, using Minkowski's inequality, we can obtain
\begin{equation}
\label{eqn:L2_Minkowski_bound}
\sqrt{
\int_{\mathcal{Z}^N}\!
\left\|
\Delta(1\,|\,\xi)
\right\|^2
p_0(\xi)
\,\mathrm{d}\xi
}
\leq
\int_0^1\!
\exp\!
\left(
\int_s^1 \!
L(r)
\,\mathrm{d}r
\right)
\sqrt{
\int_{\mathcal{Z}^N}\!
\left\|\Gamma(s\,|\,\xi)\right\|^2
p_0(\xi)
\,\mathrm{d}\xi
}
\;\mathrm{d}s.
\end{equation}
Since the Wasserstein distance between $\mu$ and $\nu$ with finite second moments is
\begin{equation}
\label{eqn:Wasserstein_def}
W_2^2(\mu,\nu)
\coloneqq
\inf_{\pi\in\Pi(\mu,\nu)}
\int \!
\left\|\mathbf{x}-\mathbf{y}\right\|^2
\,\mathrm{d}\pi(\mathbf{x},\mathbf{y}),
\end{equation}
where $\Pi(\mu,\nu)$ denotes the set of probability measures $\pi$ on $\mathcal{Z}^N\times\mathcal{Z}^N$ whose first and second marginal distributions are $\mu$ and $\nu$, respectively, we have
\begin{equation*}
W_2^2(\mu_1,\nu_1)
\leq
\int_{\mathcal{Z}^N}\!
\left\|
\mathbf{X}^{\mathrm{gui}}(1\,|\,\xi)
-
\mathbf{X}^{\mathrm{nom}}(1\,|\,\xi)
\right\|^2
p_0(\xi)
\,\mathrm{d}\xi
=
\int_{\mathcal{Z}^N}\!
\left\|\Delta(1\,|\,\xi)\right\|^2
p_0(\xi)
\,\mathrm{d}\xi.
\end{equation*}
Using \eqref{eqn:L2_Minkowski_bound}, \eqref{eqn:wasserstein_bound} can be obtained.

In addition, for a $\pi\in\Pi(\mu_1,\nu_1)$, let $(\mathbf{Z}^{\mathrm{nom}},\mathbf{Z}^{\mathrm{gui}})$ have joint probability distribution $\pi$. Then $\mathcal{D}_N(\mathbf{Z}^{\mathrm{nom}}\,|\,\chi)$ and $\mathcal{D}_N(\mathbf{Z}^{\mathrm{gui}}\,|\,\chi)$ have probability distributions $\bar{\mu}_1$ and $\bar{\nu}_1$, respectively. Thus, according to \eqref{eqn:Wasserstein_def} and $\left\|\mathcal{D}_N(\mathbf{y}\,|\,\chi)-\mathcal{D}_N(\mathbf{z}\,|\,\chi)\right\|\leq L_D\left\|\mathbf{y}-\mathbf{z}\right\|$, $\forall \mathbf y,\mathbf z\in\mathcal Z^N$, we can obtain
\begin{equation*}
W_2^2(\bar{\mu}_1,\bar{\nu}_1)
\leq
\int \!
\left\|
\mathcal{D}_N(\mathbf{x}\,|\,\chi)
-
\mathcal{D}_N(\mathbf{y}\,|\,\chi)
\right\|^2
\,\mathrm{d}\pi(\mathbf{x},\mathbf{y})
\leq
L_D^2
\int \!
\left\|\mathbf{x}-\mathbf{y}\right\|^2
\,\mathrm{d}\pi(\mathbf{x},\mathbf{y}).
\end{equation*}
Since $\pi\in\Pi(\mu_1,\nu_1)$ is arbitrary, we have
\begin{equation*}
W_2^2(\bar{\mu}_1,\bar{\nu}_1)
\leq
L_D^2 W_2^2(\mu_1,\nu_1),
\end{equation*}
which gives
\eqref{eqn:decoder_wasserstein_bound}.

\section{Task 1: Detailed Settings and Additional Results} \label{appendix:task1_details}
\subsection{Detailed settings for Task 1}
\label{app:task1_model}

Denote the plank pose as $g_B\in\mathbb{SE}(2)$, with orientation $R_B\in\mathbb{SO}(2)$ and position $p_B\in\mathbb{R}^2$. Its body twist is
\[
\xi_B^b
=
(g_B^{-1}\dot g_B)^\vee
=
[(v_B^b)^\top,\omega_B^b]^\top,
\]
where $(\cdot)^\vee:\mathfrak{se}(2)\rightarrow\mathbb{R}^3$. For robot $i$, let $r_i=[r_{i,x},r_{i,y}]^\top\in\mathbb{R}^2$ denote its contact point in the body frame of the plank and define $r_i^\perp \coloneqq [-r_{i,y},\,r_{i,x}]^\top$. The velocity of the contact point is $c_i^b = J_B(r_i)\xi_B^b$, where $J_B(r_i) \coloneqq [I_2,r_i^\perp] \in\mathbb{R}^{2\times3}$. The object generated in Task~1 is a motion policy over a finite physical horizon. Conditioned on $\chi$, the decoder maps the joint generative state to a sequence of robot contact-point velocities, $\mathcal D_N(\mathbf z\,|\,\chi) = \{c_{i,t}^b\}_{i\in\mathcal N,\,t\in\mathcal T}$, where $\mathcal T\coloneqq\{1,\ldots,T\}$ and $c_{i,t}^b = J_B(r_i)\xi_{B,t}^b(\mathbf z\,|\,\chi)$. Only the first $T_{\rm exec}$ actions of the generated policy are executed before a new policy is generated. Let $\mathcal K\coloneqq\{1,\ldots,K\}$. For each robot, the generative state $z_i\in\mathbb R^{24}$ is decoded into $K=8$ three-dimensional vectors $\eta_{i,k}\in\mathbb R^3$, $\forall k \in \mathcal{K}$. Define $M_i\coloneqq J_B(r_i)^\top J_B(r_i)$ and $G\coloneqq\sum_{i=1}^{N}M_i$, with $G\succ0$. Let $\mathcal T\coloneqq\{1,\ldots,T\}$ with $T=20$, and let $B_{\rm spl}\in\mathbb R^{T\times K}$ be the B-spline matrix. Before a new policy is generated, a $T$-step body twist sequence $\{\bar\xi_{B,t}^b\}_{t\in\mathcal T}$ is constructed with the previously generated policy and the current motion. The decoder $\mathcal D_N(\mathbf z\,|\,\chi)=\{c_{i,t}^b\}_{i\in\mathcal N,t\in\mathcal T}$ gives 
$$
\xi_{B,t}^b=\bar\xi_{B,t}^b+\sum_{k=1}^{K}(B_{\rm spl})_{t,k} (G^{-1}\sum_{i=1}^{N}M_i\eta_{i,k}),
$$
and $c_{i,t}^b=J_B(r_i)\xi_{B,t}^b$, $\forall i\in\mathcal N, t\in\mathcal T$. Only the first $T_{\rm exec}=4$ actions of the generated $T$-step policy are executed before a new policy is generated. Before the next initial sample is drawn from a Gaussian distribution, the previous policy is shifted forward by these four actions, with its final value repeated to keep length $T$, and its first five values interpolated between the current motion and the shifted sequence. The resulting sequence $\{\bar\xi_{B,t}^b\}_{t\in\mathcal T}$ is included in the next condition $\chi$.

For Task~1, the condition $\chi$ contains only quantities known before a new policy is generated and is characterized by three vectors for each robot $i$. Specifically, the vector $\lambda_i\in\mathbb R^{23}$ represents robot $i$, the plank, and their locations relative to the workspace geometry. For each other robot $j\neq i$, the vector $\pi_{ij}\in\mathbb R^8$ represents the position and velocity of robot $j$ relative to robot $i$. The vector $c_i\in\mathbb R^{45}$ represents robot $i$'s contact location on the plank, the desired plank pose, and the continuation of its contact point velocities from the previously generated policy. In particular, let $p_i^w,v_i^w\in\mathbb R^2$ be the current position and velocity of robot $i$ in the world frame. Denote the current plank pose as $g_B=(R_B,p_B)$ with $p_B=[x_B,y_B]^\top$, and let $\theta_B$ denote the angle associated with $R_B$. Let $v_B^w\in\mathbb R^2$ and $\omega_B^b\in\mathbb R$ be the current translational and angular velocities of the plank. The corresponding quantities in the body frame of the plank are $\Delta p_i^b\coloneqq R_B^\top(p_i^w-p_B)$, $v_i^b\coloneqq R_B^\top v_i^w$, and $v_B^b\coloneqq R_B^\top v_B^w$. Let $[X_{\min},X_{\max}]\times[Y_{\min},Y_{\max}]$ quantify the workspace, define $X\coloneqq(X_{\max}-X_{\min})/2$ and $Y\coloneqq(Y_{\max}-Y_{\min})/2$, and let $x_{\rm gap}^L$ and $x_{\rm gap}^R$ be the two $x$-coordinates of the gap edges. We denote by $x_{\rm goal}^{\min}$ the minimum $x$-coordinate of the region from which robot goals are sampled. The vector $e_B \in \mathbb R^{11}$ consists of the plank orientation and its position relative to the workspace, i.e., $e_B := [\cos(\theta_B),\,-\sin(\theta_B),\,\sin(\theta_B),\,\cos(\theta_B),\,(X_{\min}-x_B)/X,\,(x_{\rm gap}^{L}-x_B)/X,\,(x_{\rm gap}^{R}-x_B)/X,\,(x_{\rm goal}^{\min}-x_B)/X,\,(X_{\max}-x_B)/X,\,(Y_{\min}-y_B)/Y,\,(Y_{\max}-y_B)/Y]^{\top}$. Denote $o_i\in\{0,1\}^3$, which indicates if robot $i$ is on the initial side ground, on the plank, or on the opposite side ground. Let $a_i\in\{0,1\}$ indicate whether robot $i$ contacts with the plank, and let $b_B\in\{0,1\}$ indicate whether bridge construction has succeeded. Additionally, $\lambda_i = [(\operatorname{diag}(1/X,1/Y)\Delta p_i^b)^\top,\,(v_i^b/0.86)^\top,\,o_i^\top,\,(v_B^b/0.44)^\top,\,\omega_B^b/0.32,\,b_B,\,a_i,\,e_B^\top]^\top \in \mathbb{R}^{23}$. For every $j\neq i$, denote the relative position and velocity $\Delta p_{ij}^b\coloneqq R_B^\top(p_j^w-p_i^w)$ and $\Delta v_{ij}^b\coloneqq R_B^\top(v_j^w-v_i^w)$ in the body frame of the plank, then $\pi_{ij} = [(\operatorname{diag}(1/X,1/Y)\Delta p_{ij}^b)^\top,\,(\Delta v_{ij}^b/0.86)^\top,\,o_j^\top,\,a_j]^\top \in \mathbb{R}^8$. Let $L_B$ be the plank length, and the desired plank pose be $g_B^{\rm goal}=(R_g(\theta_g),p_g)$. Denote $R_B^\top(p_g-p_B)=[\Delta x_g^b,\Delta y_g^b]^\top$, and let $\Delta\theta_g\in(-\pi,\pi]$ be the shortest signed angular difference from $\theta_B$ to $\theta_g$. Let $C_i^{\rm cont} = \left[J_B(r_i)\bar\xi_{B,1}^b, \ldots, J_B(r_i)\bar\xi_{B,T}^b\right]^\top \in\mathbb R^{T\times2}$. For each of the $40$ entries of $C_i^{\rm cont}$, the corresponding training set mean is subtracted and the result is divided by the corresponding training set standard deviation. Let $\widetilde C_i^{\rm cont}\in\mathbb R^{T\times2}$ denote the resulting matrix, then $c_i = [r_{i,x}/(L_B/2),\,\operatorname{sgn}(r_{i,y}),\,\Delta x_g^b/L_B,\,\Delta y_g^b/Y,\,\Delta\theta_g/(\pi/3),\,\operatorname{vec}(\widetilde C_i^{\rm cont})^\top]^\top \in \mathbb R^{45}$, where $\operatorname{sgn}(s)=1$ for $s>0$ and $\operatorname{sgn}(s)=-1$ for $s<0$, and $\operatorname{vec}(\cdot)$ stacks the entries of a matrix into a vector.

For each $j\neq i$, the multilayer perceptron (MLP) $\psi_{\rm msg}$, with layer dimensions $77\rightarrow160\rightarrow160$, is applied to $z_j$, $c_j$, and $\pi_{ij}$ to form the message vector $m_{ij}=\psi_{\rm msg}\!\left([z_j^\top,c_j^\top,\pi_{ij}^\top]^\top\right)\in\mathbb R^{160}$. The message vectors from all other robots are averaged as $\bar m_i\coloneqq (1/(N-1)) \sum_{j\neq i}m_{ij}$. The MLP $\psi_{\rm ag}$ shared by each agent, with layer dimensions $258\rightarrow256\rightarrow256\rightarrow256\rightarrow24$, is used to represent $f_i^\theta(\tau,\mathbf z\,|\,\chi)=\psi_{\rm ag}\!\left([z_i^\top,c_i^\top,\lambda_i^\top,\bar m_i^\top,(N-1)/7,\varphi(\tau)^\top]^\top\right)\in\mathbb R^{24}$, where $\varphi(\tau)=[\tau,\sin(\pi\tau),\cos(\pi\tau),\sin(2\pi\tau),\cos(2\pi\tau)]^\top\in\mathbb R^5$.

The nominal vector field is trained and validated on $N\in\{4,5,6\}$. For Task~1, we instantiate the conditional probability path by choosing $p_0=\mathcal N(0,I)$. For an encoded data sample $\mathbf Z\sim p_{\rm data}(\mathbf z\,|\,\chi)$, we draw $\mathbf X_0\sim p_0$ and choose 
$$
\mathbf X(\tau) = (1-\tau)\mathbf X_0+\tau\mathbf Z, \quad \forall \tau\sim\mathrm{Unif}[0,1].
$$
Then, the corresponding conditional vector field is $v(\tau,\mathbf X(\tau)\,|\,\mathbf Z,\chi)=\mathbf Z-\mathbf X_0$. At $\tau$ in the generative process, the estimated endpoint is $\widehat{\mathbf Z} = \mathbf X(\tau) + (1-\tau) f^\theta(\tau,\mathbf X(\tau)\,|\,\chi)$. Let $\widehat c_{i,t}^b$ and $\widehat\xi_{B,t}^b$ denote the velocity of the contact point and the body twist of the plank obtained by decoding $\widehat{\mathbf Z}$, and let $c_{i,t}^{b,*}$ and $\xi_{B,t}^{b,*}$ denote the corresponding quantities from the trajectory used to construct $\mathbf Z$. Let $\sigma_c\in\mathbb R^2$ contain the training set scales for the two components of the contact point velocity, and let $\sigma_\xi\in\mathbb R^3$ contain the training set scales for the three components of the body twist of the plank. Define $\rho_t=2.5$ for $t=1,\ldots,4$ and $\rho_t=1$ otherwise. The loss function $\mathcal L_{\rm MAC}$ for Task~1 is augmented as $\mathcal L_{\rm train} = \mathcal L_{\rm MAC} + 0.16\mathcal L_{\rm motion} + 0.45\mathcal L_{\rm twist}$, in which $\mathcal L_{\rm motion} = (1/(2N\sum_{t=1}^{T}\rho_t))\sum_{i=1}^{N}\sum_{t=1}^{T}\rho_t\|\operatorname{diag}(\sigma_c)^{-1}(\widehat c_{i,t}^b-c_{i,t}^{b,*})\|^2$ and $\mathcal L_{\rm twist} = (1/(3\sum_{t=1}^{T}\rho_t))\sum_{t=1}^{T}\rho_t\|\operatorname{diag}(\sigma_\xi)^{-1}(\widehat\xi_{B,t}^b-\xi_{B,t}^{b,*})\|^2$.

The task and safety factors are constructed from geometric margins. Each margin is positive when its corresponding requirement is satisfied. Specifically, the workspace is $[X_{\min},X_{\max}]\times[Y_{\min},Y_{\max}] =[-7,6.8]\times[-3,3]$, and the spatial gap takes $x\in[x_{\rm gap}^{\min},x_{\rm gap}^{\max}]=[0,2]$. The plank length and width are $L_B=5.8$ and $W_B=0.72$, each robot is a disk with radius $r_R=0.16$, the required plank support on each side of the gap is $d_{\rm sup}=0.55$, and the traversable region for the robot centers must extend at least $d_{\rm cor}=0.05$ beyond each gap edge. Let $g_{B,t}=(R_{B,t},p_{B,t})$, $t=1,\ldots,T_{\rm exec}$, denote the
plank poses predicted after the $T_{\rm exec}=4$ actions that will be executed, with $p_{B,t}=[x_t,y_t]^\top$ and orientation angle $\theta_t$. Define $|\cos(\theta)|_\varepsilon\coloneqq \sqrt{\cos^2(\theta) + \varepsilon^2}$ and $|\sin(\theta)|_\varepsilon\coloneqq \sqrt{\sin^2(\theta) + \varepsilon^2}$ with $\varepsilon=10^{-6}$. The distance from the plank center to its outermost point along the $x$-axis of the world frame is $h_x(\theta)\coloneqq
\tfrac12\!\left(L_B|\cos(\theta)|_\varepsilon+W_B|\sin(\theta)|_\varepsilon\right)$. The projection onto the same axis of the usable half-length of the plank for a robot center is $h_c(\theta)\coloneqq \left(L_B/2-r_R-0.075\right)|\cos(\theta)|_\varepsilon$. At predicted step $t$, let the four margins be $m_{1,t}^{\rm task}=x_{\rm gap}^{\min}-x_t+h_x(\theta_t)-d_{\rm sup}$, $m_{2,t}^{\rm task}=x_t+h_x(\theta_t)-x_{\rm gap}^{\max}-d_{\rm sup}$, $m_{3,t}^{\rm task}=x_{\rm gap}^{\min}-x_t+h_c(\theta_t)-d_{\rm cor}$, and $m_{4,t}^{\rm task}=x_t+h_c(\theta_t)-x_{\rm gap}^{\max}-d_{\rm cor}$, and let $m_{5,t}^{\rm task},\ldots,m_{M_{\rm task},t}^{\rm task}$ denote the remaining geometric margins included in the task factor, where $M_{\rm task}$ is the total number of task margins. The task factor evaluates all of these margins at the final predicted plank pose as
\begin{equation*}
q_{\rm task}(\mathbf z\,|\,\chi)
=
\frac{1}{80}
\log\!\left(
\sum_{k=1}^{2}
\exp\!\left[
-80\,\frac{m_{k,T_{\rm exec}}^{\rm task}}{0.55}
\right]
+
\sum_{k=3}^{M_{\rm task}}
\exp\!\left[
-80\,\frac{m_{k,T_{\rm exec}}^{\rm task}}{0.2}
\right]
\right).
\end{equation*}
Let $m_{t,1}^{\rm safe},\ldots,m_{t,M_{\rm safe}}^{\rm safe}$ denote the geometric margins at predicted step $t\in\{1,\ldots,T_{\rm exec}\}$, where $M_{\rm safe}$ is the number of safety margins evaluated at each step. These margins measure the distances of the plank from the upper and lower workspace boundaries using a $0.2$ buffer, the distances of the robot disks from the gap and from the upper and lower workspace boundaries using a $0.1$ buffer, and whether the plank has sufficient support on the side of the workspace where robots are initiated or has reached sufficient support on the opposite side. The safety factor is then
\begin{equation*}
q_{\rm safe}(\mathbf z\,|\,\chi)
=
\frac{1}{60}
\log\!\left(
\sum_{t=1}^{T_{\rm exec}}
\sum_{k=1}^{M_{\rm safe}}
\exp\!\left[
-60\,\frac{m_{t,k}^{\rm safe}}{0.1}
\right]
\right).
\end{equation*}
The success and failure are evaluated using the exact geometry (without buffers) instead of the factors above. A trial is successful if the robots can move the plank such that it provides at least $d_{\rm sup}=0.55$ of support on both sides of the gap and the region traversable by the robot centers extends beyond both gap edges as required. A failure is counted if the plank loses support on the initial side before reaching the opposite side, leaves the workspace, or causes any robot to cross the workspace boundary or enter the spatial gap.

We choose $g_i=H_i=I_{24}$, $\alpha_a^{\rm SE}(s)=s$, and $\beta_a(\tau \,|\, \chi,\mathbf{z}(0))
= \beta_{a,\rm end} + (\beta_{a,\rm start}-\beta_{a,\rm end})(1-6\tau^5+15\tau^4-10\tau^3)$. At the beginning of each policy generation, let $g_B^{\rm cur}$ denote the plank pose and let $q_{\rm task}^{\rm cur}$ denote the task factor evaluated at $g_B^{\rm cur}$. Let $\mathbf z_0\coloneqq\mathbf z(0)$ denote the Gaussian initial state during the generative process. Specifically, $\beta_{\rm task,end}=q_{\rm task}^{\rm cur}-0.05$, $\beta_{\rm task,start}=\max\!\left\{q_{\rm task}(\mathbf z_0\,|\,\chi)+0.02,\,\beta_{\rm task,end}+0.02\right\}$, $\beta_{\rm safe,end}=0$, $\beta_{\rm safe,start}=q_{\rm safe}(\mathbf z_0\,|\,\chi)+0.02$. Therefore, at the end of each generative process, the task factor is required to satisfy $q_{\rm task}\le q_{\rm task}^{\rm cur}-0.05$, and the safety factor is required to satisfy $q_{\rm safe}\leq 0$. Note that $\beta_{task,\rm start}$ is calculated once before the generative process and is fixed throughout the generative process so $\beta_{\rm task}(\tau)$ is $C^1$ in $\tau$. In addition, the adaptive constraint allocation coefficients $w_{a,i}^{\rm SE}$ are calculated using \eqref{eqn:C2_weights} with $\delta=0.02$. Each robot then solves the decoupled guidance using \eqref{eqn:SE_local_QP} with one constraint corresponding to the task factor and one corresponding to the safety factor.

\subsection{Additional results of Task 1} \label{app:task1_results}
\begin{figure*}[t]
\centering
\includegraphics[width=0.25\linewidth]{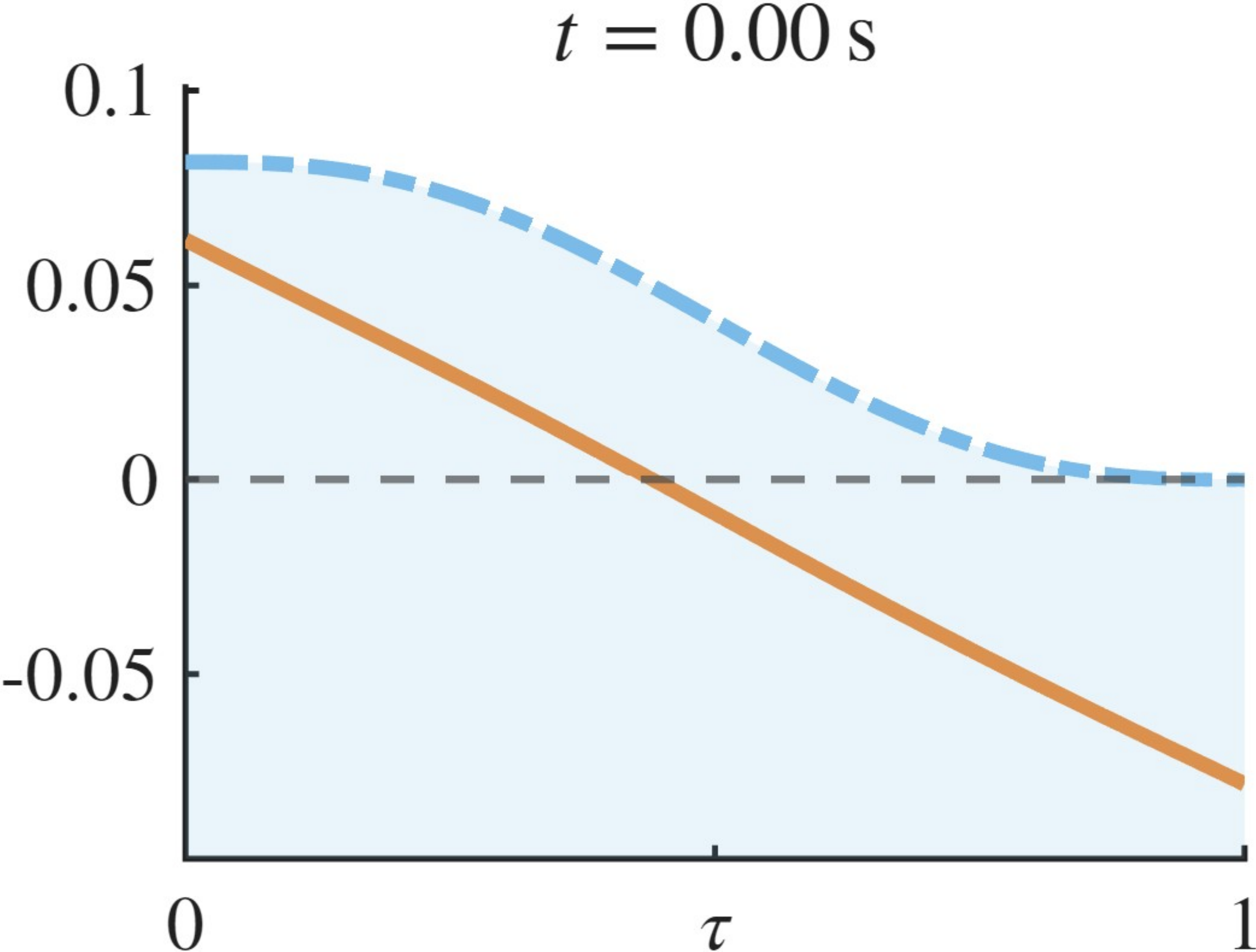}
\hspace{0.02\linewidth}
\includegraphics[width=0.25\linewidth]{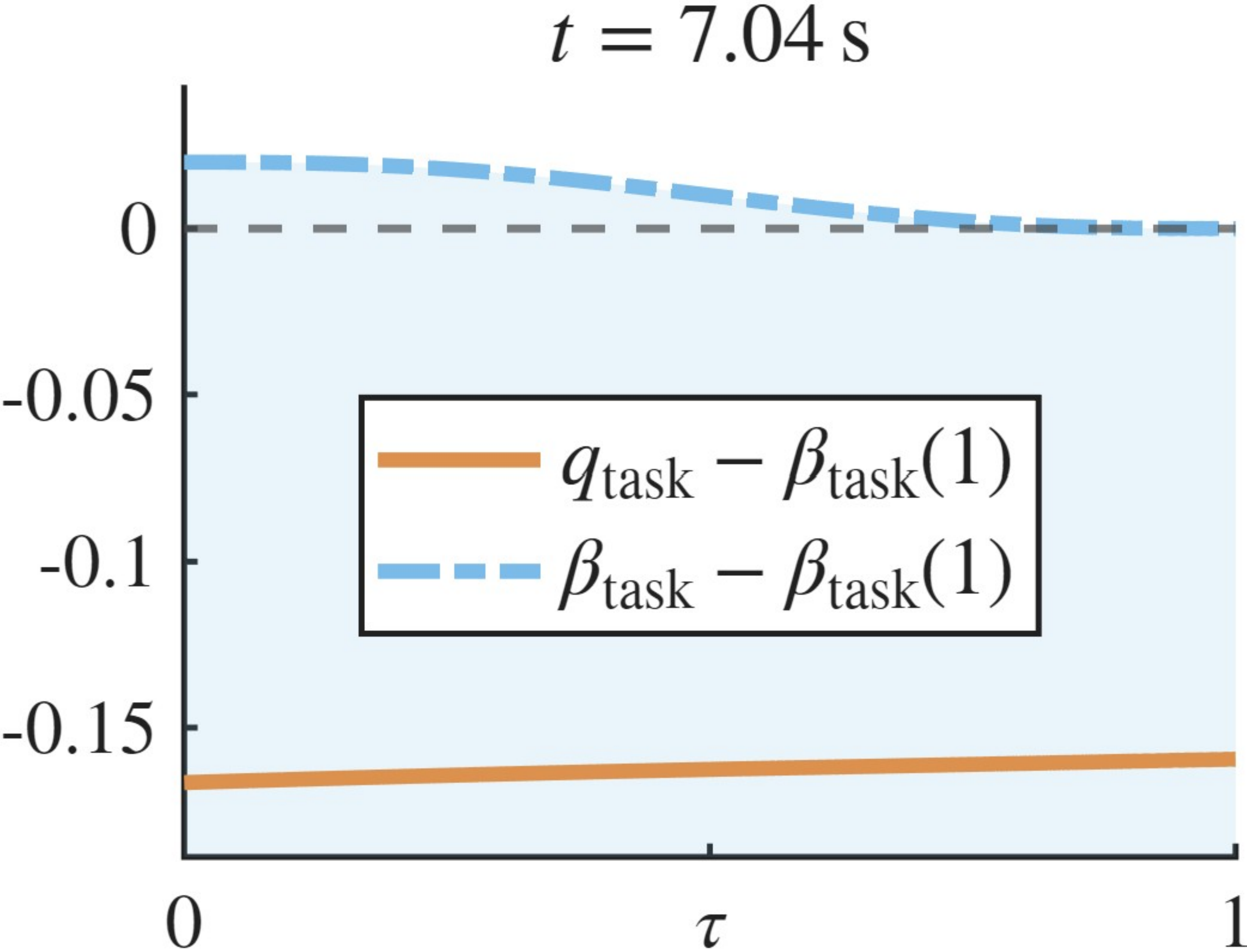}
\hspace{0.02\linewidth}
\includegraphics[width=0.25\linewidth]{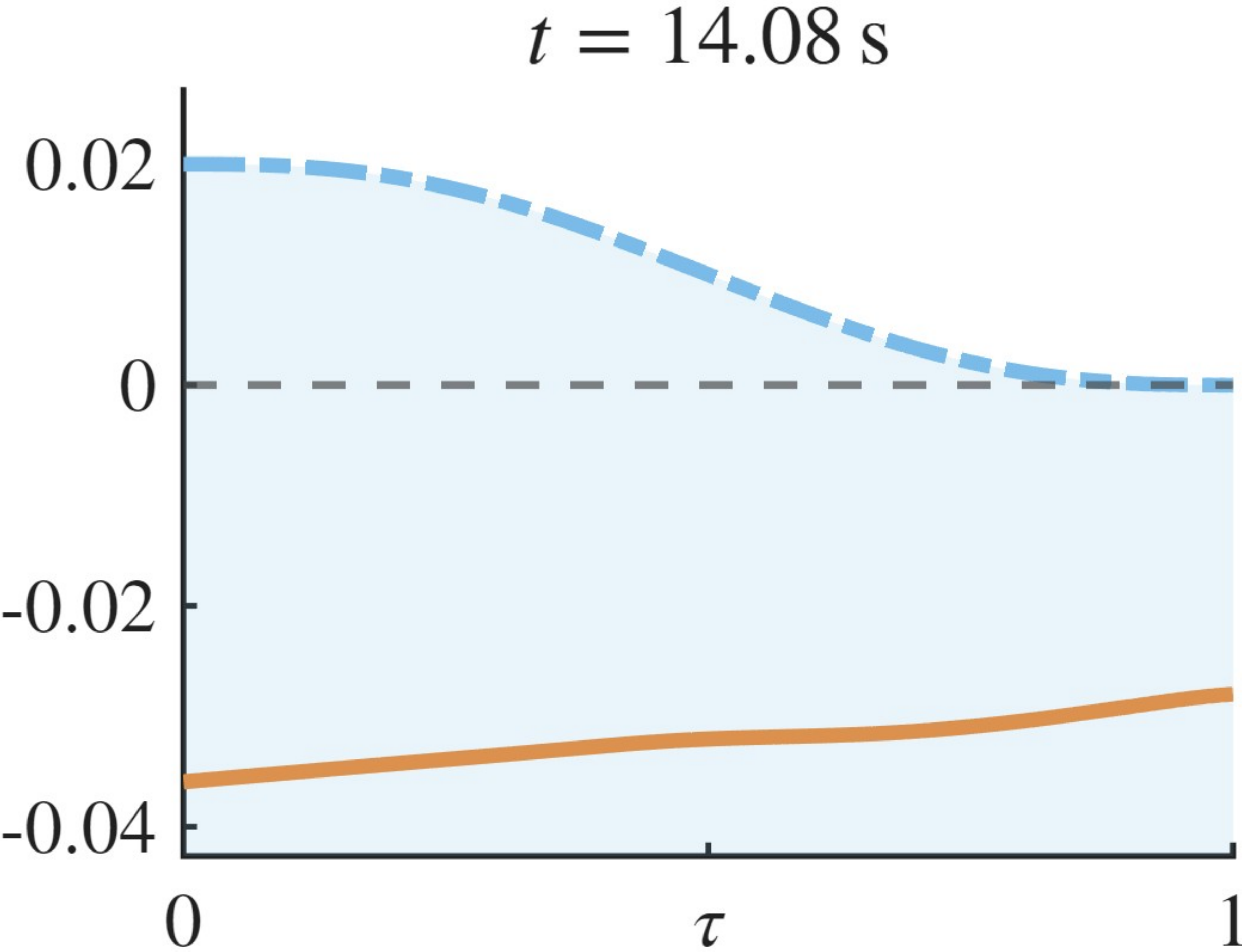}
\caption{Evolution of $q_{\rm task}$ and its time-varying upper bound $\beta_{\rm task}$ during generative processes for policies generated at early (left), middle (center), and late (right) physical times in the manipulation stage of the example shown in Figure~\ref{fig:task1_rollout}.}
\label{fig:task1_task_sampling}
\end{figure*}

Figure~\ref{fig:task1_task_sampling} shows $q_{\rm task}(\mathbf z(\tau)\,|\,\chi)$ and $\beta_{\rm task}(\tau)$ along $\tau$ in the generative process for policies generated at early, middle, and late physical times in the manipulation stage of the $N=5$ Guided example shown in Figure~\ref{fig:task1_rollout}. For visualization only, both curves in each plot are shifted by the same constant $\beta_{\rm task}(1)$, and $q_{\rm task}-\beta_{\rm task}$ is unchanged. The shaded region corresponds to $q_{\rm task}(\mathbf z(\tau)\,|\,\chi)\leq \beta_{\rm task}(\tau)$. As shown in Figure~\ref{fig:task1_task_sampling}, early in physical time, $q_{\rm task}$ starts close to its upper bound $\beta_{\rm task}$ and then moves farther below it during the generative process. This is because $\beta_{\rm task,start}$ is initialized just above the task factor value of the initial sample, and $\beta_{\rm task,end}=q_{\rm task}^{\rm cur}-0.05$ requires the generated policy to reduce the task factor by at least $0.05$ relative to its value at the current plank pose. The SE guidance therefore progressively separates $q_{\rm task}$ from the contracting upper bound and improves the task accomplishment. At the middle physical time, $q_{\rm task}$ stays below $\beta_{\rm task}$ even though it increases during the generative process. This is because the SE guidance does not require $q_{\rm task}$ to decrease monotonically. Instead, it only requires $q_{\rm task}\leq \beta_{\rm task}$, allowing the nominal vector field to dominate the generated policy. Across the three physical times, the relation $q_{\rm task}(\mathbf z(\tau)\,|\,\chi)\le\beta_{\rm task}(\tau)$ is consistent with Theorem~\ref{thm:SE_finite_horizon}.

Figure~\ref{fig:task1_rescue_N5} shows another trial in which the two policies, Nominal and Guided, start from the same physical configuration and initial generative sample. The policy generated by the nominal vector field drives the plank out of the workspace, whereas the policy generated with SE guidance avoids this failure and successfully constructs the bridge. 

\begin{figure*}[t]
\centering
\includegraphics[width=1\linewidth]{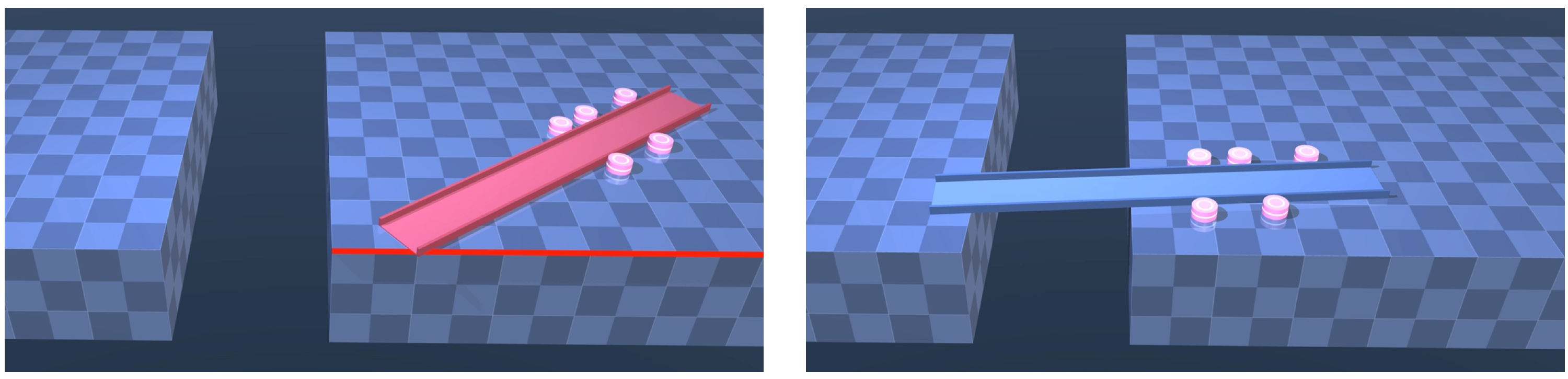}
\caption{Left: the pre-trained policy is driving the plank out of the workspace. Right: DeGG-Flow leads to successful bridge construction.}
\label{fig:task1_rescue_N5}
\end{figure*}

\section{Task 2: Detailed Settings and Additional Results} \label{appendix:task2_details}

\subsection{Detailed settings for Task~2}
\label{appendix:DetailedsettingsforTask2}
Task~2 considers tabletop scenes with a library of 12 object types. For each scene, $N$ different objects are randomly selected from the library. Each object has a fixed geometric shape and one or more nearby regions that should remain unobstructed when the object is used, which are referred to as affordance regions. A user-specified affordance requirement can change the location, size, or direction of these regions relative to the object. An object can therefore be required to remain accessible from a different side, to have more free space around it, or to satisfy combinations of these changes. The complete object library, geometric representations, affordance regions, and ranges used to generate modified requirements are given below.

For Task~2, choose $g_i=H_i=I_3$, $\delta^{\mathrm{PE}}=0.02$, $\omega^{\mathrm{PE}}(\tau)=0.1+1.8\tau$, $c^{\mathrm{PE}} = (1.15/(1-\rho^{\mathrm{PE}}))\max\{V^{\mathrm{PE}}(0),10^{-30}\}^{1-\rho^{\mathrm{PE}}}$, and $\rho^{\mathrm{PE}}=0.75$. For the construction of $V^{\mathrm{PE}}$ in Section~\ref{sec:task2}, the buffer $\delta^{\mathrm{PE}}>0$ ensures the satisfaction of the hard requirements, as discussed in the following corollary.
\begin{corollary}
Let $\delta^{\mathrm{PE}}>0$ and
\begin{equation*}
V^{\mathrm{PE}}(\mathbf{z}\,|\,\chi)
=
\frac{1}{2}\sum_{i\in\mathcal{N}}\sum_k \max
\left\{
r_{ik}(\mathbf{z}\,|\,\chi)
+
\delta^{\mathrm{PE}},0
\right\}^2.
\end{equation*}
If $V^{\mathrm{PE}}(\mathbf z\,|\,\chi) \leq {(\delta^{\mathrm{PE}})^2}/{2}$, then $h_{ik}(\mathbf z\,|\,\chi)\geq 0$ for all $i,k$.
\end{corollary}
\begin{proof}
Since every term in $V^{\mathrm{PE}}$ is nonnegative, then we have
\begin{equation*}
\frac{1}{2} \max
\left\{
r_{ik}(\mathbf{z}\,|\,\chi) + \delta^{\mathrm{PE}},0
\right\}^2
\leq 
V^{\mathrm{PE}}(\mathbf{z}\,|\,\chi)
\leq 
\frac{(\delta^{\mathrm{PE}})^2}{2}
\end{equation*}
for every $i,k$. Thus, $r_{ik}(\mathbf{z}\,|\,\chi) + \delta^{\mathrm{PE}} \leq\delta^{\mathrm{PE}}$, meaning that $h_{ik}(\mathbf{z}\,|\,\chi)=-r_{ik}(\mathbf{z}\,|\,\chi)\geq 0$.
\end{proof}

\noindent The tabletop is centered at the origin and has size $1.40\times0.94$. For each object $i$, $z_i\in\mathbb R^3$ is decoded into a planar position $p_i\in\mathbb R^2$ and an angular coordinate $\theta_i\in\mathbb R$ by $[p_i^\top,\theta_i]^\top =\operatorname{diag}(0.70,0.47,1.35)\,z_i$ with the corresponding orientation $R_i(\theta_i)\in\mathbb{SO}(2)$. Each type of object also has a fixed geometric shape defined in its body frame, represented as a disk, a rounded rectangle, or a combination of these shapes. Table~\ref{tab:task2_library} summarizes the 12 object categories, with each category showing at most once in a scene. Denote $C(r)$ for a disk of radius $r$, $R(a,b,r_c)$ for a rounded rectangle with half width $a$, half height $b$, and rounded corner radius $r_c$, and $A(x,y,r)$ for a circular affordance region of radius $r$ centered at $(x,y)$ in the object's body frame. For multiple affordance regions with the same $y$-coordinate and radius, we list their $x$-coordinates as a set. For example, $A(\{x_1,\ldots,x_m\},y,r)$ denotes circles centered at $(x_1,y),\ldots,(x_m,y)$, each with radius $r$. For the $q$th affordance region of object $i$, let $a_{iq}\in\mathbb R^2$ denote the center of the region in the body frame of object $i$, and let $r_{iq}>0$ denote its radius. A user-specified requirement is parameterized by $(s_i^o,s_i^r,\varphi_i)$, where $s_i^o$ changes the distance of the affordance region from the object, $s_i^r$ changes its radius, and $\varphi_i$ changes its direction relative to the object. Then, the center and radius expressed in the world frame are $ a_{iq}^{W}= p_i+ R(\theta_i+\varphi_i)s_i^o a_{iq}$ and $r_{iq}^{W}=s_i^r r_{iq}$, respectively. The default requirement corresponds to $(s_i^o,s_i^r,\varphi_i)=(1,1,0)$. A user may instead require an object to be accessible from another side, require more free space around the object, or combine these changes. We sample $s_i^o\in[0.98,1.12]$ and $s_i^r\in[1.04,1.24]$. For a change in access direction, $\varphi_i$ is selected from the allowed directions for that object type, with an additional perturbation in $[-\pi/22.5,\pi/22.5]$. The alternative directions are $-\pi/6$ and $\pi/6$ for drawer organizer and drawing board; $-\pi/2$ and $\pi/2$ for desk lamp, handled mug, and tissue box; $\pi$, $-\pi/2$, and $\pi/2$ for teacup, laptop, scissors, pencil case, and pen cup; and $-\pi/2$, $\pi/2$, and $\pi$ for potted plant and compact printer.
\begin{table*}[t]
\caption{Object geometry and default affordance regions for Task~2}
\label{tab:task2_library}
\centering
\small
\setlength{\tabcolsep}{8pt}
\renewcommand{\arraystretch}{1.08}
\begin{tabular}{lll}
\toprule
Object & Geometry & Affordance region(s)\\
\midrule
Drawer organizer
& $R(0.156,0.090,0.012)$
& $A(\{-0.1,0,0.1\},-0.205,0.108)$\\
Teacup
& $C(0.046)\cup(C(0.019)+(0.061,0))$
& $A(0.125,0,0.055)$\\
Desk lamp
& $C(0.06)$
& $A(0,-0.147,0.079)$\\
Laptop
& $R(0.085,0.065,0.006)$
& $A(0.155,0,0.07)$\\
Potted plant
& $C(0.078)$
& $A(0,-0.123,0.062)$\\
Handled mug
& $C(0.067)$
& $A(0,-0.141,0.054)$\\
Drawing board
& $R(0.135,0.048,0.008)$
& $A(\{-0.065,0.065\},-0.108,0.05)$\\
Scissors
& $R(0.039,0.026,0.016)$
& $A(0.083,0,0.043)$\\
Compact printer
& $R(0.099,0.07,0.009)$
& $A(0,-0.133,0.057)$\\
Pencil case
& $R(0.037,0.072,0.007)$
& $A(0.087,0,0.046)$\\
Tissue box
& $R(0.049,0.04,0.009)$
& $A(0,-0.103,0.047)$\\
Pen cup
& $C(0.034)$
& $A(0.083,0,0.039)$\\
\bottomrule
\end{tabular}
\end{table*}

In addition, training contains 4,000 scenes for each $N\in\{4,5,6\}$ and 400 validation scenes for each $N$. For every scene, $N$ different types of objects are sampled uniformly without replacement from the library. Thus, all object categories appear during training. $N=7,8$ evaluate generalization to scenes containing more objects than those used for training. The nominal vector field is represented by a GNN shared by every agent, and trained using \eqref{eqn:macfm_loss}.

\subsection{Additional results for Task 2} 
\label{app:task2_results}
Figure~\ref{fig:task2_sampling} shows the evolution of $V^{\mathrm{PE}}$ during the generative process for the $N=7$ example corresponding to Figure~\ref{fig:task2_personalization}. Both trajectories start from the same initial Gaussian sample and are evaluated under the same left-handed laptop use requirement, which requires the mouse operating region on the left side of the laptop to be unobstructed. Without guidance, this requirement is violated at the end of the generative process, whereas PE guidance guarantees that $V^{\mathrm{PE}}$ converges to zero and generates a scene satisfying the new requirement without model retraining, which is consistent with Theorem~\ref{thm:private_finite_time}.

\begin{figure}[t]
\centering
\includegraphics[width=0.35\linewidth]{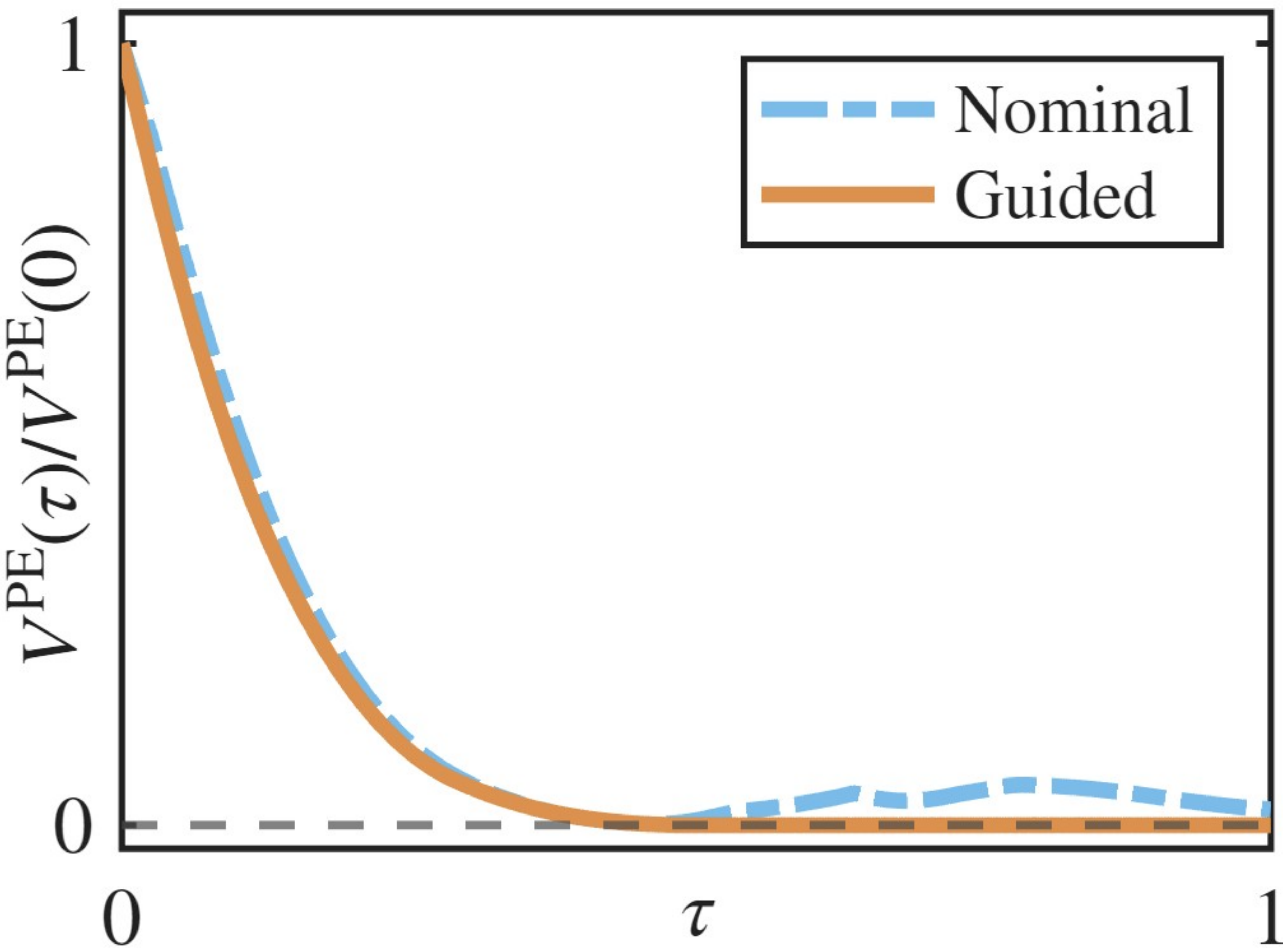}
\caption{Evolution of $V^{\mathrm{PE}}$ during the generative process for the $N=7$ example in Figure~\ref{fig:task2_personalization}.}
\label{fig:task2_sampling}
\end{figure}

We additionally evaluate five types of requirement: the default requirements, opposite-side access, side-direction access, a larger affordance region, and combinations of these changes. Opposite-side access selects the direction farthest from the default one, and side-direction access selects one of the other non-default directions. Each type contains 50 cases, with 10 cases for each $N\in\{4,5,6,7,8\}$. As shown in Table~\ref{tab:task2_families}, Guided directly generates the scenes satisfying all 50 cases in every type of requirement without model retraining.

\begin{table*}[t]
\centering
\small
\setlength{\tabcolsep}{10pt}
\caption{Task~2 results with different affordance requirements.}
\label{tab:task2_families}
\begin{tabular}{lccc}
\toprule
Requirement
& Natural Nominal
& Private Nominal
& Guided\\
\midrule
Default
& $27/50$ (54\%)
& $27/50$ (54\%)
& $50/50$ (100\%)\\
Opposite-side access
& $27/50$ (54\%)
& $23/50$ (46\%)
& $50/50$ (100\%)\\
Side-direction access
& $31/50$ (62\%)
& $23/50$ (46\%)
& $50/50$ (100\%)\\
Larger affordance region
& $23/50$ (46\%)
& $22/50$ (44\%)
& $50/50$ (100\%)\\
Combined
& $27/50$ (54\%)
& $17/50$ (34\%)
& $50/50$ (100\%)\\
\bottomrule
\end{tabular}
\end{table*}

\begin{figure*}[t]
\centering
\includegraphics[width=0.4\linewidth]{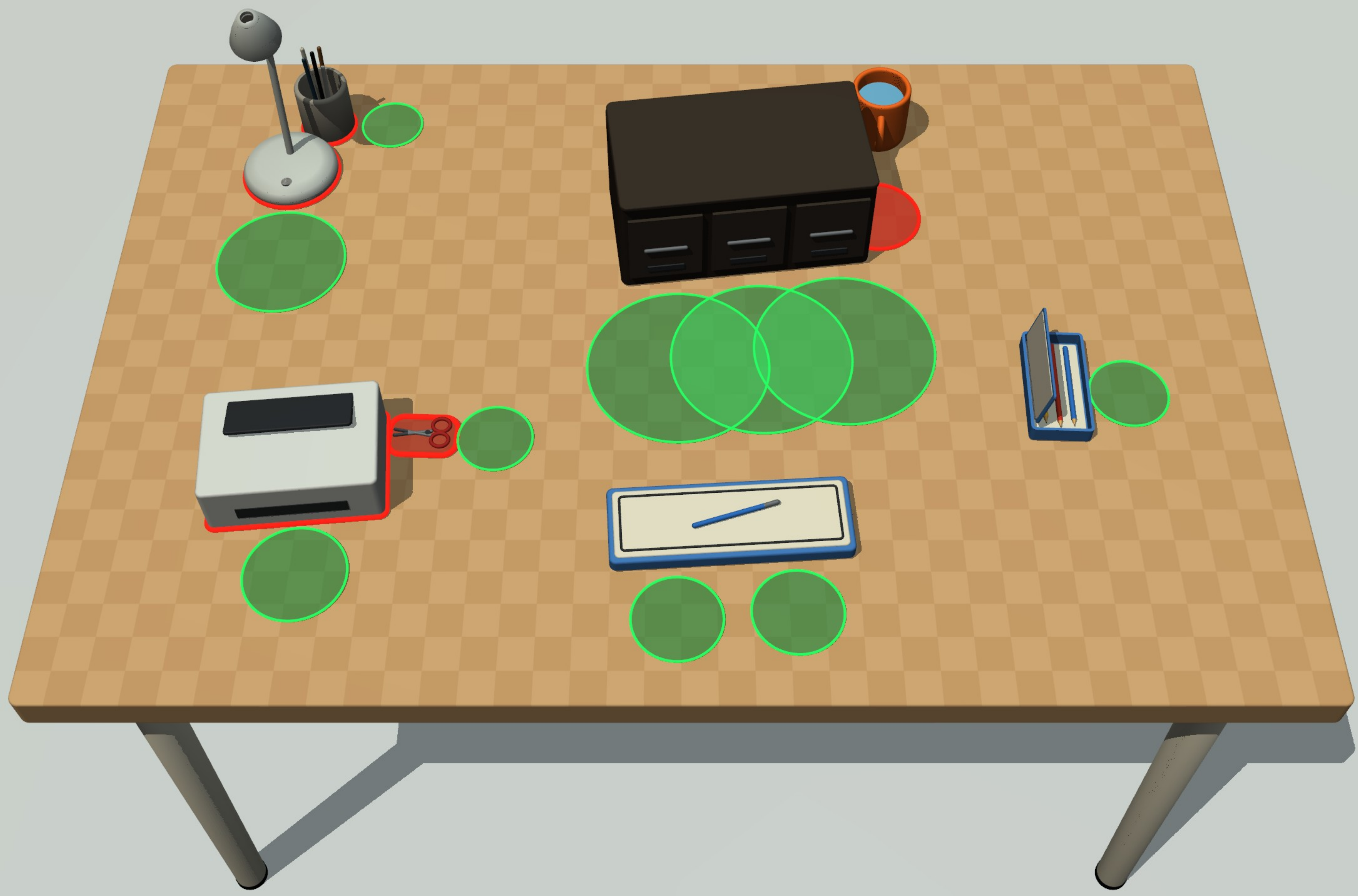}
\hspace{0.02\linewidth}
\includegraphics[width=0.4\linewidth]{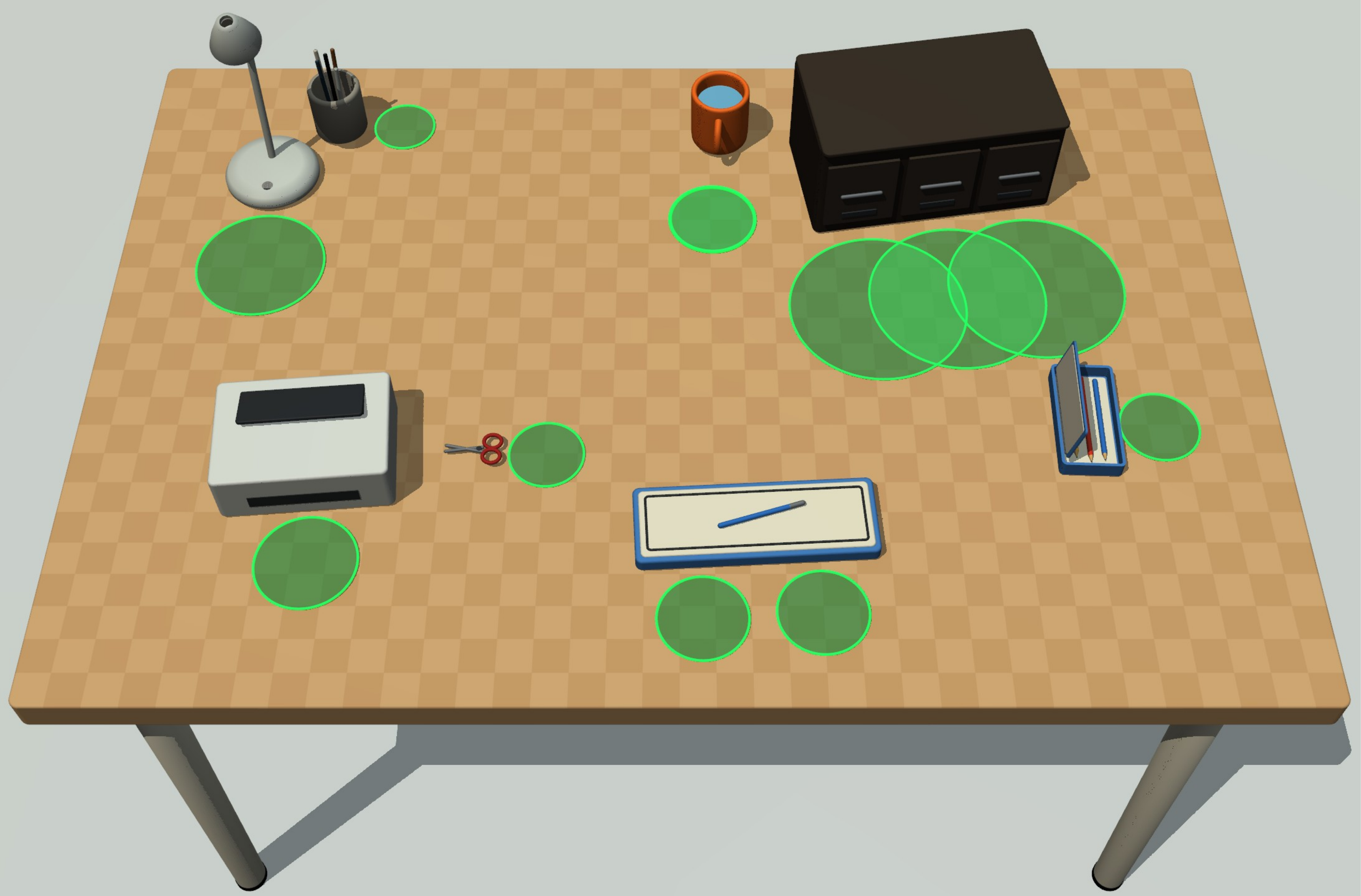}
\caption{An $N=8$ example of Task~2. Green circles denote unobstructed affordance regions, a red circle denotes an obstructed affordance region, and red object outlines indicate pairs of objects that are too close to each other. Left: the compact printer and scissors, and the desk lamp and pen cup, are too close to each other, and the drawer organizer obstructs the affordance region of the handled mug. Right: DeGG-Flow resolves all three violations from the same initial sample without model retraining.}
\label{fig:task2_canonical_rescue}
\end{figure*}

Figure~\ref{fig:task2_canonical_rescue} shows a complementary $N=8$ example without any user-specified change to the default requirements. The scene generated by the nominal vector field contains several violations: the compact printer and scissors are too close to each other, the desk lamp and pen cup are also too close to each other, and the drawer organizer obstructs the affordance region of the handled mug. Starting from the same initial Gaussian sample, PE guidance directly generates a scene satisfying all of these requirements, without any model retraining.

We further evaluate how much PE guidance changes the distribution characterized by the nominal flow. Task~2 provides a natural setting for this evaluation as each generative process directly produces one complete scene consisting of the poses of all $N$ objects. For each fixed $N\in\{4,5,6,7,8\}$, let $M=50$ denote the number of evaluated scenes. For scene $m\in\{1,\ldots,M\}$ and object $i\in\{1,\ldots,N\}$, denote the final generated states by Nominal and Guided by $z_{m,i}^{\rm Nom}$ and $z_{m,i}^{\rm Gui}$, respectively. The corresponding joint generative states of the complete scenes are then denoted by $\mathbf z_m^{\rm Nom}=[(z_{m,1}^{\rm Nom})^\top, \ldots, (z_{m,N}^{\rm Nom})^\top]^{\top}$ and $\mathbf z_m^{\rm Gui}=[(z_{m,1}^{\rm Gui})^\top, \ldots, (z_{m,N}^{\rm Gui})^\top]^{\top}$, respectively. Thus, the two sets $\{\mathbf z_m^{\rm Nom}\}_{m=1}^{M}$ and $\{\mathbf z_m^{\rm Gui}\}_{m=1}^{M}$ approximate the corresponding Nominal and Guided scene distributions. First, we evaluate the distance directly in the generative state space. For two scene states $\mathbf z=[z_1^\top,\ldots,z_N^\top]^\top$ and $\mathbf z'=[(z_1')^\top,\ldots,(z_N')^\top]^\top$, denote the normalized Euclidean distance as $d_z^2(\mathbf z,\mathbf z') = \frac{1}{N} \sum_{i=1}^{N} \|z_i-z_i'\|^2$. Using the $M$ generated scenes, we approximate the Wasserstein distance between the Nominal and Guided joint distributions as
\begin{equation*}
\widehat W_{2,z}^{\,2}(N)
=
\min_{\sigma\in\mathfrak S_M}
\frac{1}{M}
\sum_{m=1}^{M}
d_z^2
\left(
\mathbf z_m^{\rm Nom},
\mathbf z_{\sigma(m)}^{\rm Gui}
\right),
\end{equation*}
where $\mathfrak S_M$ denotes the set of all permutations of $\{1,\ldots,M\}$. Thus, $\sigma$ specifies a optimal transport assignment between the $M$ Nominal scenes and the $M$ Guided scenes.

We also evaluate the same distributional deviation after decoding into object poses using $[(p_{m,i}^{b})^\top,\theta_{m,i}^{b}]^\top =
\operatorname{diag}(0.70,0.47,1.35)\,z_{m,i}^{b}$, with $b\in\{{\rm Nom},{\rm Gui}\}$ representing Nominal and Guided, and $g_{m,i}^{b} \in\mathbb{SE}(2)$. Denote the complete decoded scene as $\mathbf g_m^{b} = \left(g_{m,1}^{b}, \ldots, g_{m,N}^{b} \right)$. For two object poses $g_i$ and $g_i'$, we use the geodesic \citep{park1995distance}
\begin{equation*}
d_{\mathbb{SE}(2)}^2(g_i,g_i')
=
\|p_i-p_i'\|^2
+
\lambda_\theta^2
\Delta\theta_i^2,
\end{equation*}
where $\Delta\theta_i = \operatorname{atan2} \left(\sin(\theta_i-\theta_i'), \cos(\theta_i-\theta_i') \right) \in[-\pi,\pi]$, and we choose $\lambda_\theta=0.1$, which places translational and rotational differences on a common length scale comparable to the sizes of the objects in Table~\ref{tab:task2_library}. For two decoded scenes $\mathbf g=(g_1,\ldots,g_N)$ and $\mathbf g'=(g_1',\ldots,g_N')$, define the scene distance by
\begin{equation*}
d_{\rm scene}^2(\mathbf g,\mathbf g')
=
\frac{1}{N}
\sum_{i=1}^{N}
d_{\mathbb{SE}(2)}^2(g_i,g_i').
\end{equation*}
Then the corresponding approximation of the Wasserstein distance is
\begin{equation*}
\widehat W_{2,g}^{\,2}(N)
=
\min_{\sigma\in\mathfrak S_M}
\frac{1}{M}
\sum_{m=1}^{M}
d_{\rm scene}^2
\left(
\mathbf g_m^{\rm Nom},
\mathbf g_{\sigma(m)}^{\rm Gui}
\right).
\end{equation*}
The magnitude of a Wasserstein distance is difficult to interpret without a reference scale, so we compare it with the root mean square (RMS) pairwise distance among the $M$ complete scenes generated by the nominal vector field. In the generative state space, define
\begin{equation*}
S_{{\rm Nom},z}(N)
=
\left(\!
\frac{2}{M(M-1)}
\sum_{1\le m<n\le M}\!\!
d_z^2
\left(
\mathbf z_m^{\rm Nom},
\mathbf z_n^{\rm Nom}
\right)
\!\right)^{\!\!\frac{1}{2}}\!.
\end{equation*}
Similarly, in the decoded physical space, define
\begin{equation*}
S_{{\rm Nom},g}(N)
=
\left(\!\!
\frac{2}{M(M-1)}\!
\sum_{1\le m<n\le M}\!\!\!
d_{\rm scene}^2
\left(
\mathbf g_m^{\rm Nom},
\mathbf g_n^{\rm Nom}
\right)
\!\!\right)^{\!\!\!\frac{1}{2}}\!\!.
\end{equation*}

\begin{figure}[t]
\centering
\includegraphics[width=0.4\linewidth]{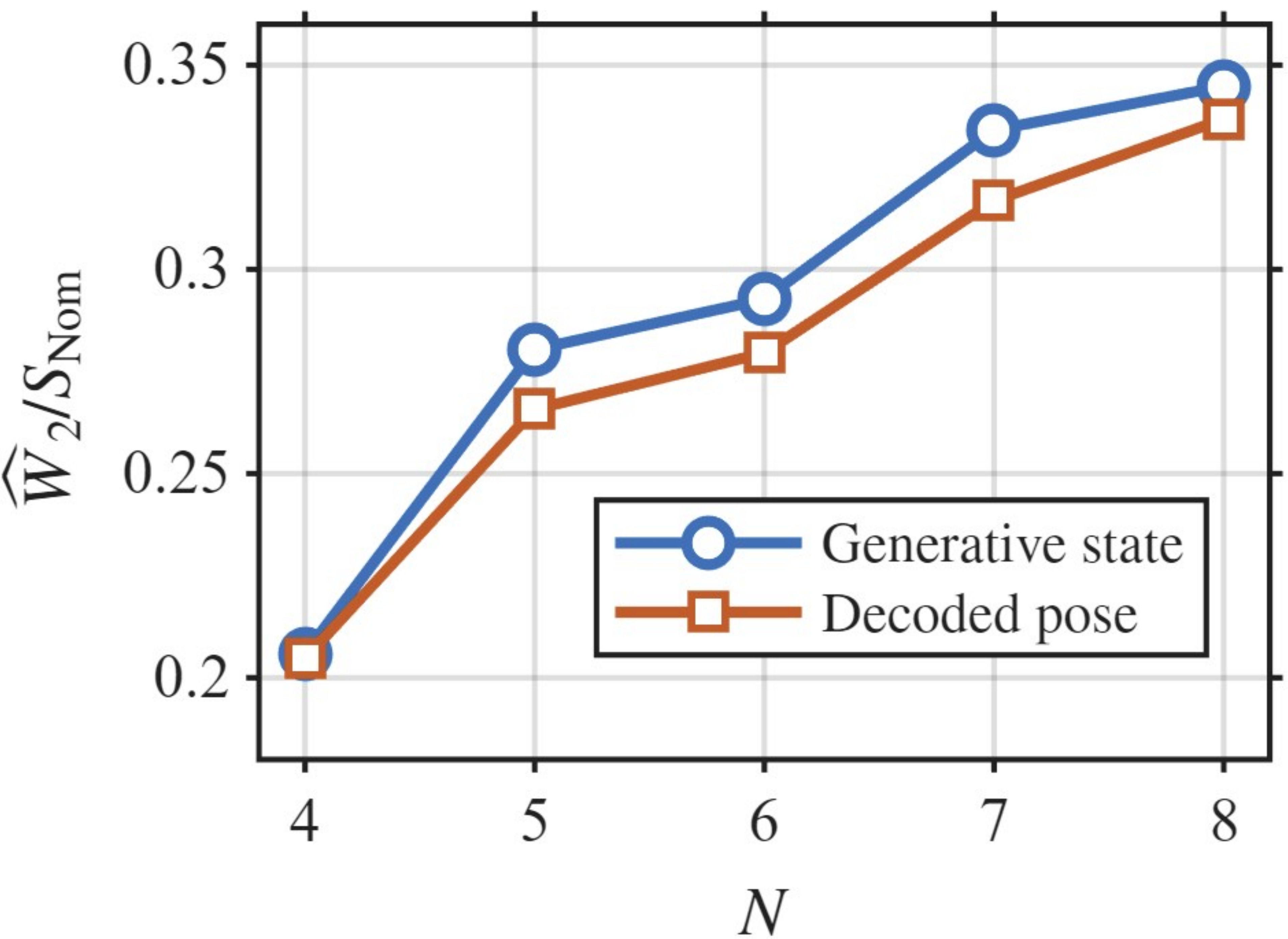}
\caption{Approximation of the Wasserstein distance between the Nominal and Guided distributions in Task~2 with respect the total number of generated objects, normalized by the RMS pairwise distance among the $M=50$ scenes generated by the nominal vector field. The blue and orange curves use the distance metrics in the generative state space and the decoded pose space, respectively.}
\label{fig:task2_w2}
\end{figure}
We then plot the two ratios $\frac{\widehat W_{2,z}(N)}{S_{{\rm Nom},z}(N)}$ and $\frac{\widehat W_{2,g}(N)}{S_{{\rm Nom},g}(N)}$ in Figure~\ref{fig:task2_w2}, where the first ratio measures the Wasserstein distance between the Nominal and Guided distributions in the generative state space relative to the RMS pairwise distance among the $M$ complete scenes generated by the nominal vector field, and the second ratio measures the one in the decoded pose space. For example, a value of $0.3$ means that the Wasserstein distance between the Nominal and Guided distributions is $0.3$ times the RMS pairwise distance among the $M$ complete scenes generated by the nominal vector field. As the number of objects increases, both curves increase gradually, meaning that guidance makes larger modifications when there are more objects and thus more coupled requirements. As shown in Figure~\ref{fig:task2_w2}, in the generative state space, $\widehat W_{2,z}(N)/S_{{\rm Nom},z}(N)$ increases from $20.6\%$ for $N=4$ to $34.5\%$ for $N=8$. In the decoded pose space, $\widehat W_{2,g}(N)/S_{{\rm Nom},g}(N)$ increases from $20.5\%$ to $33.7\%$. Thus, even for $N=8$, the Wasserstein distance between the Nominal and Guided distributions is only about one third of the RMS pairwise distance among the complete scenes generated by the nominal vector field that measures how spread out the generated scenes are. The results show that guidance introduce a distributional deviation that is much smaller than the variation among scenes generated by the nominal vector field.

We additionally examine whether guidance reduces how spread the generated objects are. Define
\begin{equation*}
S_{{\rm Gui},z}(N)
=
\left(
\frac{2}{M(M-1)}
\sum_{1\le m<n\le M}
d_z^2
\left(
\mathbf z_m^{\rm Gui},
\mathbf z_n^{\rm Gui}
\right)
\right)^{\frac{1}{2}},
\end{equation*}
and
\begin{equation*}
S_{{\rm Gui},g}(N)
=
\left(\!
\frac{2}{M(M-1)}
\sum_{1\le m<n\le M}\!\!
d_{\rm scene}^2
\left(
\mathbf g_m^{\rm Gui},
\mathbf g_n^{\rm Gui}
\right)
\!\right)^{\!\!\frac{1}{2}}\!\!.
\end{equation*}
These quantities are defined in exactly the same way as $S_{{\rm Nom},z}(N)$ and $S_{{\rm Nom},g}(N)$, but using the $M$ complete scenes generated with guidance. Across $N=4,\ldots,8$, $S_{{\rm Gui},z}(N)/S_{{\rm Nom},z}(N)$ ranges from $98.3\%$ to $99.9\%$, and $S_{{\rm Gui},g}(N)/S_{{\rm Nom},g}(N)$ ranges from $99.1\%$ to $100.5\%$, meaning that the RMS pairwise distance among Guided scenes is almost the same as that among scenes generated by the nominal vector field. 

These results show that the generated objects with PE guidance satisfy the desired hard requirements leading to only a moderate deviation of the distribution resulting from the nominal vector field, and approximately preserving the spread of the generated scenes.

\section{Task 3: Detailed Settings and Additional Results}
\label{appendix:task3_details}

\subsection{Detailed settings for Task 3}
\label{app:task3_model}

All three methods use the same pre-trained Stable Diffusion v2-base model \citep{rombach2022high}, from which the nominal vector field is obtained following \citet{song2020denoising,song2020score}. We use 50 DDIM steps with a classifier-free guidance scale \citep{ho2022classifier} of 7.5 for all three methods. Using the panoramic image setting of MultiDiffusion~\citep{bar2023multidiffusion} and SyncDiffusion~\citep{lee2023syncdiffusion}, we generate a panoramic image of size $512 \times 3072$ pixels using $N=21$ overlapping image regions of size $512\times512$ pixels with a horizontal stride of 128 pixels. Each region is taken care of by an agent, represented by $z_i\in \mathbb R^{4\times64\times64}$ in the diffusion model's latent space. Let $\mathcal{M}(\mathbf{z})$ denote the latent representation of the panoramic image, obtained by averaging the overlapping regions of the agents' latent states. The condition $\chi$ is the text prompt, and $\mathcal{D}_N(\mathbf{z}(1)\,|\,\chi)$ is the final generated image. For Task 3, we choose the SE requirements to bound the differences between overlapping latent regions and differences in their channel means and covariances. Specifically, let
\begin{equation*}
Q(\mathbf{z}\,|\,\chi)
=
0.05
\log
\left(
\sum_{i=1}^{20}
\exp\!\left(
\frac{e_i(\mathbf{z})}
{0.05 \times 0.014^2}
\right)
+
\sum_{i\ne i^*}
\exp\!
\left(
\frac{D_i(\mathbf{z})}
{0.05 \times 0.94 \gamma_{\mathrm{p}}}
\right)
\right),
\end{equation*}
where $e_i(\mathbf{z})={\|z_i^{\mathrm{R}}-z_{i+1}^{\mathrm{L}}\|^2}/{(4\times 64\times 48)}$, $\forall i=1,\ldots,20$, with $z_i^{\mathrm{R}}$ and $z_{i+1}^{\mathrm{L}}$ as the overlapping parts of adjacent regions, and $D_i(\mathbf{z})=\|\mathfrak{s}(\mathcal{C}_i(\mathcal{M}(\mathbf{z})))-\mathfrak{s}(\mathcal{C}_{i^*}(\mathcal{M}(\mathbf{z})))\|^2/26$ measuring the differences in channel means and covariances relative to the center region $i^*=11$, with $\mathcal{C}_i$ extracting the $i$-th region from $\mathcal{M}(\mathbf{z})$ and $\mathfrak{s}:\mathbb{R}^{4\times64\times64}\rightarrow\mathbb{R}^{26}$ mapping each latent tensor to its channel means and covariances. Specifically, the 26 entries consist of 16 channel means computed over four groups of 16 rows, 4 channel variances, and 6 covariances between channels, with the variances and covariances computed over all spatial locations. For each prompt, $\gamma_{\mathrm{p}}$ is the median difference in channel means and covariances between the left and right regions of 100 reference images. We choose $q_{\mathrm{pano}}=Q-1$, $\beta_{\mathrm{pano}}=0$, and $\alpha_{\mathrm{pano}}^{\mathrm{SE}}(s)=4s$. All 50 initial states satisfy $V_{\mathrm{pano}}^{\mathrm{SE}}\leq 0$. 

Denote $\hat{\mathbf{z}}$ as the predicted denoised latent states, $\mathcal{P}_{h,w}$ as average pooling to an output size of $h\times w$, $\mathcal{U}$ as bilinear upsampling to $64\times384$, and $\operatorname{MSE}(\cdot,\cdot)$ as the mean squared error. We define an objective function $J_{\mathrm{pano}}=J_{\mathrm{stat}}+0.45J_{\mathrm{ov}}+0.6J_{\mathrm{pool}}$ during the generative process to favor consistency in the generated panoramic image, where $J_{\mathrm{stat}}=(\sum_{i\ne i^*}(0.75\max\{D_i^{\mathrm{raw}}(\hat{\mathbf{z}})/\gamma_{\mathrm{p}}-0.5,0\}^2+0.25\max\{D_i(\hat{\mathbf{z}})/\gamma_{\mathrm{p}}-0.5,0\}^2))/(N-1)$ with $D_i^{\mathrm{raw}}(\hat{\mathbf{z}})=\|\mathfrak{s}(\hat{z}_i)-\mathfrak{s}(\hat{z}_{i^*})\|^2/26$, $J_{\mathrm{ov}}=(\sum_{i=1}^{20}(0.5e_i(\hat{\mathbf{z}})+0.3e_i^{(2)}(\hat{\mathbf{z}})+0.2e_i^{(4)}(\hat{\mathbf{z}})))/(20\times0.20^2)$ with $e_i^{(r)}(\hat{\mathbf{z}})=\operatorname{MSE}(\mathcal{P}_{64/r,48/r}(\hat{z}_i^{\mathrm{R}}),\mathcal{P}_{64/r,48/r}(\hat{z}_{i+1}^{\mathrm{L}}))$ for $r\in\{2,4\}$, and $J_{\mathrm{pool}}=(\sum_{i=1}^{N}\operatorname{MSE}(\mathcal{P}_{4,4}(\hat{z}_i),\mathcal{P}_{4,4}(\hat{z}_i^{\mathrm{ref}})))/(N\times0.12^2)$ with $\hat{z}_i^{\mathrm{ref}}=\mathcal{C}_i(\mathcal{U}(\mathcal{P}_{4,24}(\mathcal{M}(\hat{\mathbf{z}}))))$. Then, we choose $g_i=I$ and $H_i=I$ unless the choice $H_i^{-1}=(I+v_{\mathrm{pano},i}(\hat{\ell}_{\mathrm{pano},i}^{\mathrm{SE}})^\top)(I+\hat{\ell}_{\mathrm{pano},i}^{\mathrm{SE}}v_{\mathrm{pano},i}^\top)$ has a lower $J_{\mathrm{pano}}$ and satisfies the constraints in \eqref{eqn:task3_drift_QP} and $Q\leq0.99$, where $v_{\mathrm{pano},i}$ is obtained by projecting $\nabla_{z_i}J_{\mathrm{pano}}$ onto the orthogonal complement of $\ell_{\mathrm{pano},i}^{\mathrm{SE}}$, scaled by a nonnegative scalar with $\|v_{\mathrm{pano},i}\|\leq2$, and $\nabla_{z_i}J_{\mathrm{pano}}$ is computed through the pre-trained diffusion model. A trial satisfies the requirements if $\max_i\sqrt{e_i}\leq0.05$ and $\max_{i\ne i^*}D_i\leq\gamma_{\mathrm{p}}$ at $\tau=1$. $e_i=0$ for MultiDiffusion and SyncDiffusion since both methods average the overlapping regions at each step. For SyncDiffusion, we use learned perceptual image patch similarity (LPIPS)~\citep{zhang2018unreasonable} with VGG~\citep{simonyan2014very} features,  a decay rate of $0.95$, and an initial gradient descent weight of $20$.

\subsection{Additional Results of Task 3}
\label{app:task3_results}

\begin{figure}[t]
\centering
\includegraphics[width=1\linewidth]{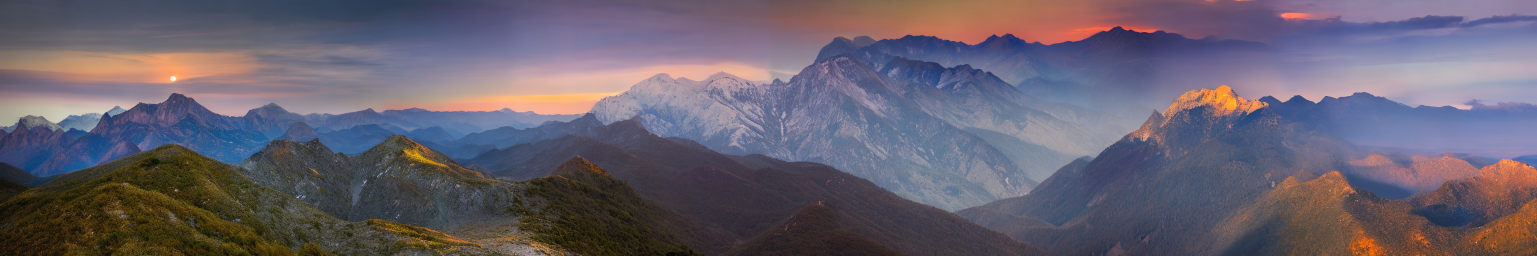}

\vspace{1mm}

\includegraphics[width=1\linewidth]{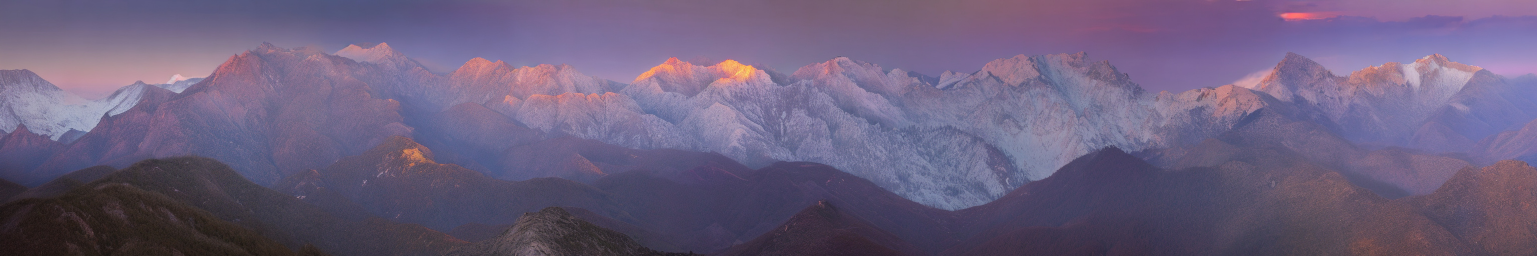}

\vspace{1mm}

\includegraphics[width=1\linewidth]{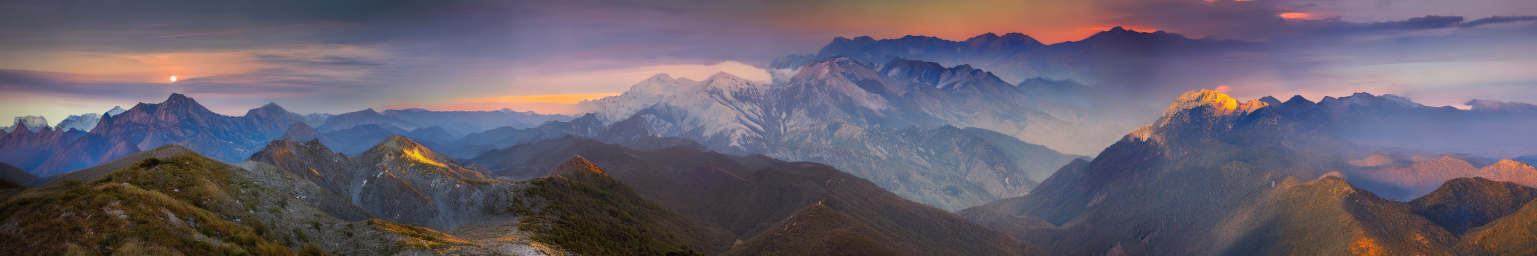}

\caption{Panoramic images for the prompt ``a photo of a mountain range at twilight.'' Top: MultiDiffusion. Middle: SyncDiffusion. Bottom: DeGG-Flow.}
\label{fig:task3_twilight}
\end{figure}

\begin{figure}[t]
\centering
\includegraphics[width=1\linewidth]{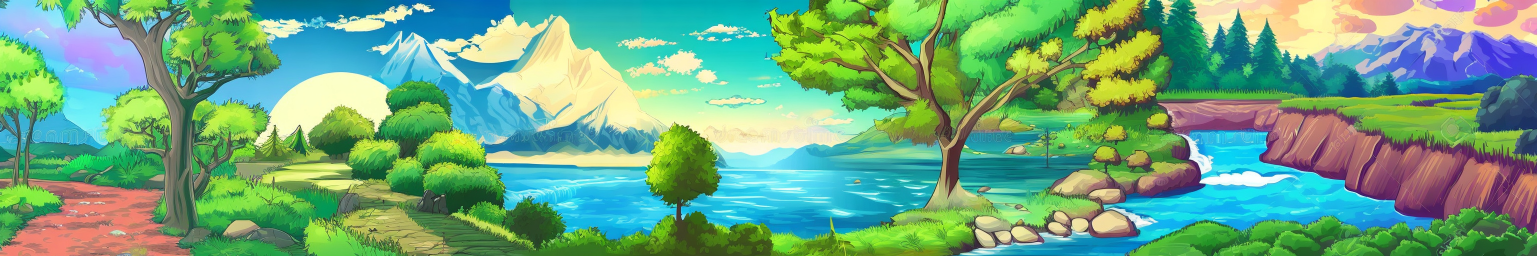}

\vspace{1mm}

\includegraphics[width=1\linewidth]{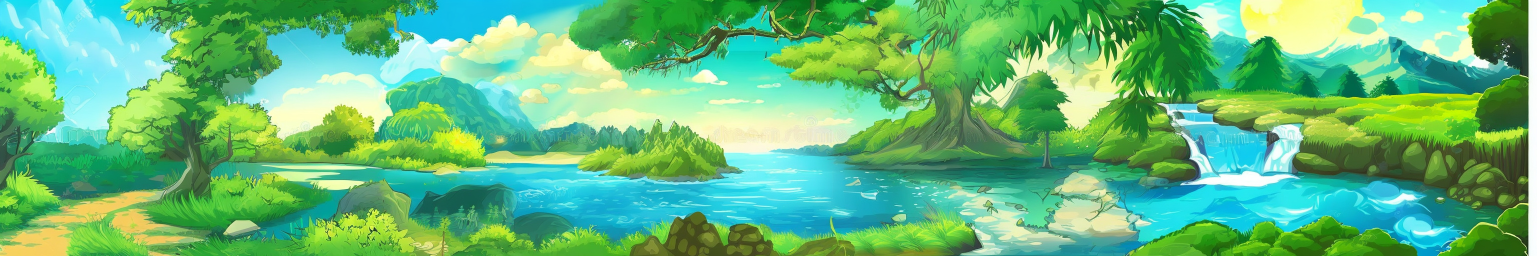}

\vspace{1mm}

\includegraphics[width=1\linewidth]{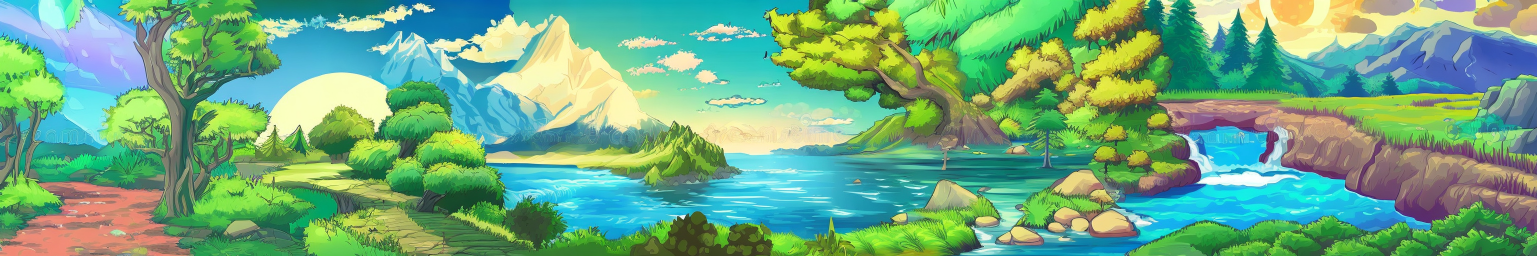}

\caption{Panoramic images for the prompt ``natural landscape in anime style illustration.'' Top: MultiDiffusion. Middle: SyncDiffusion. Bottom: DeGG-Flow.}
\label{fig:task3_anime}
\end{figure}

\begin{figure}[t]
\centering
\includegraphics[width=1\linewidth]{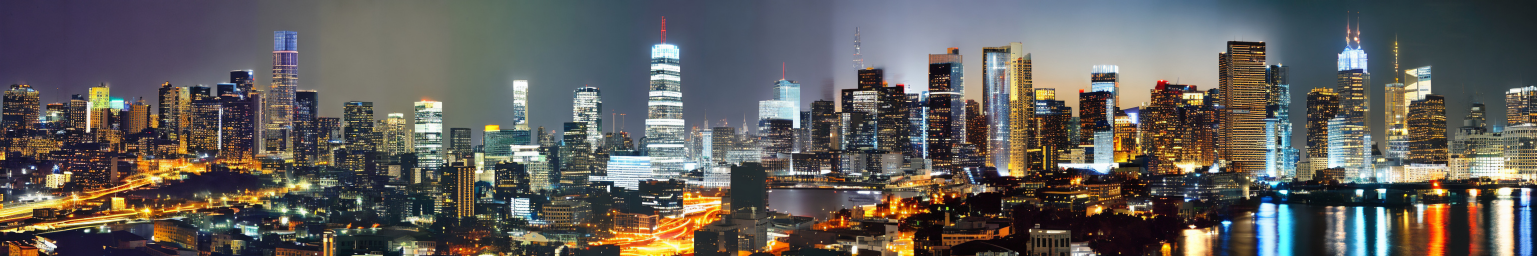}

\vspace{1mm}

\includegraphics[width=1\linewidth]{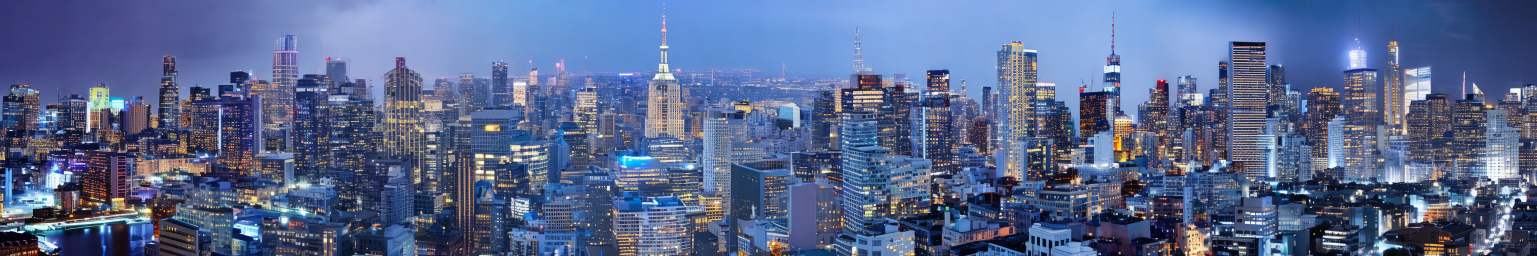}

\vspace{1mm}

\includegraphics[width=1\linewidth]{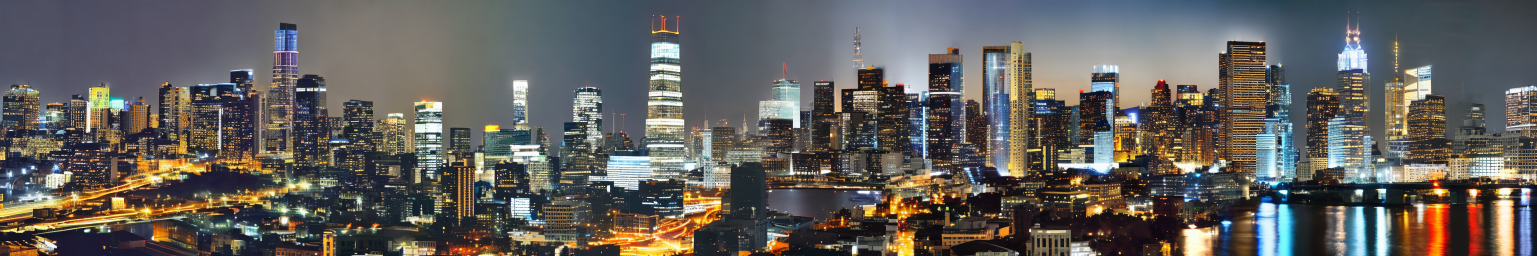}

\caption{Panoramic images for the prompt ``a photo of a city skyline at night.'' Top: MultiDiffusion. Middle: SyncDiffusion. Bottom: DeGG-Flow.}
\label{fig:task3_city}
\end{figure}

\begin{figure}[t]
\centering
\includegraphics[width=1\linewidth]{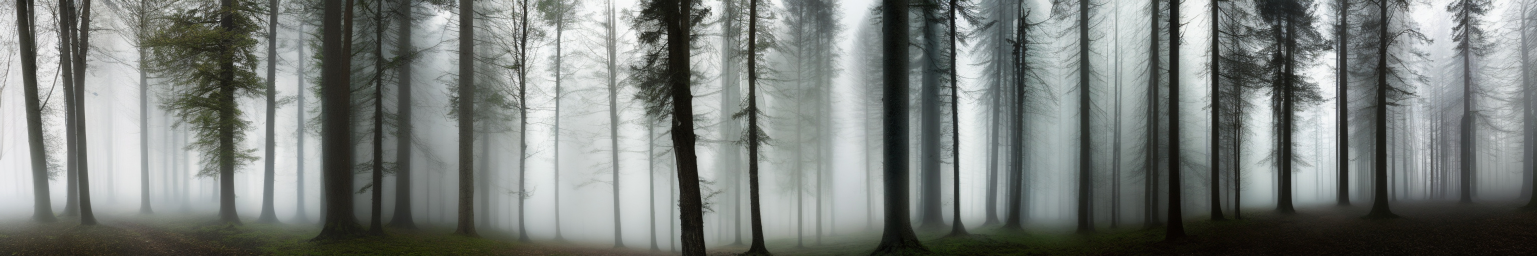}

\vspace{1mm}

\includegraphics[width=1\linewidth]{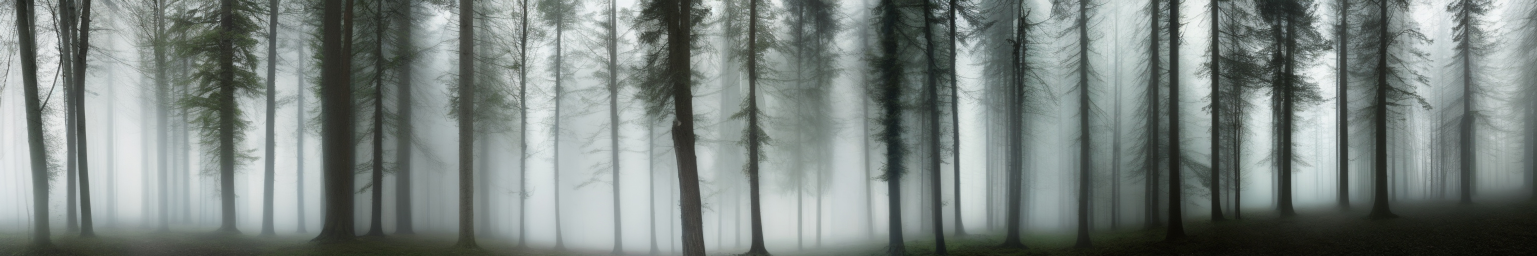}

\vspace{1mm}

\includegraphics[width=1\linewidth]{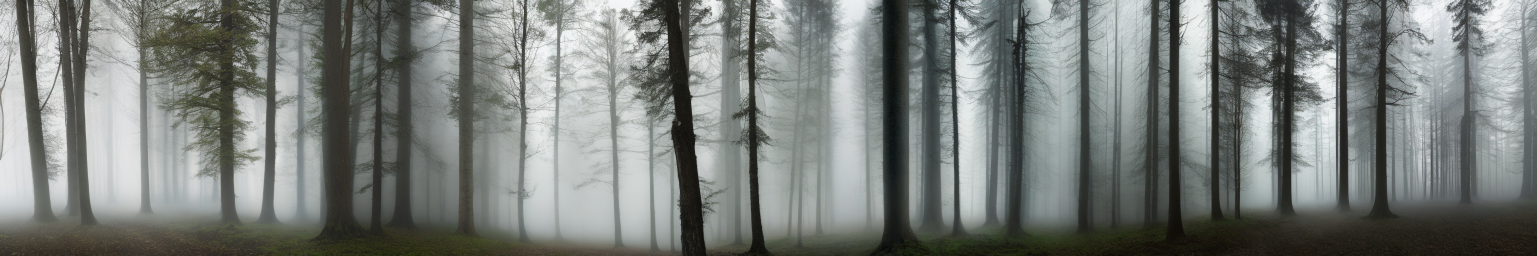}

\caption{Panoramic images for the prompt ``a photo of a forest with a misty fog.'' Top: MultiDiffusion. Middle: SyncDiffusion. Bottom: DeGG-Flow.}
\label{fig:task3_forest}
\end{figure}

\end{document}